\documentclass[10pt]{article}
\usepackage{geometry}                
\usepackage{graphicx}
\graphicspath{{../figs/}}
\usepackage{amssymb}
\usepackage{epstopdf}
\usepackage{subfigure}
\usepackage{tikz}
\usepackage{multirow}
\usepackage{rotating}
\usepackage{array}
\usepackage{pdflscape}
\usepackage{url}
\usepackage{amsmath}
\usepackage{amsthm}
\usepackage{fullpage}
\usepackage[titletoc, title]{appendix}
\newtheorem{remark}{Remark}
\newtheorem{assumption}{Assumption}
\newtheorem{proposition}{Proposition}

\usetikzlibrary{shapes,arrows}
\newcommand{\etal}{{\em et al. }}
\DeclareMathOperator*{\argmax}{arg\,max}

\begin{document}
\title{Approximate False Positive Rate Control in Selection Frequency for Random Forest}
\author{Ender Konukoglu\\
\small{Athinoula A. Martinos Center for Biomedical Imaging / Harvard Medical School, Boston}\\
\small{enderk@nmr.mgh.harvard.edu}\\
Melanie Ganz\\
\small{Athinoula A. Martinos Center for Biomedical Imaging / Harvard Medical School, Boston}\\
\small{Neurobiology Research Unit, Rigshospitalet, Copenhagen}\\
\small{mganz@nru.dk}}
\date{}

\maketitle
\begin{abstract}
Random Forest has become one of the most popular tools for feature selection. 
Its ability to deal with high-dimensional data makes this algorithm especially useful for studies in neuroimaging and bioinformatics. 
Despite its popularity and wide use, feature selection in Random Forest still lacks a crucial ingredient: false positive rate control. 
To date there is no efficient, principled and computationally light-weight solution to this shortcoming. 
As a result, researchers using Random Forest for feature selection have to resort to using heuristically set thresholds on feature rankings. 
This article builds an approximate probabilistic model for the feature selection process in random forest training,
which allows us to compute an estimated false positive rate for a given threshold on selection frequency. 
Hence, it presents a principled way to determine thresholds for the selection of relevant features without any additional computational load.
Experimental analysis with synthetic data demonstrates that the proposed approach can limit false positive rates on the order of the desired values and keep false negative rates low. 
Results show that this holds even in the presence of a complex correlation structure between features. 
Its good statistical properties and light-weight computational needs make this approach widely applicable to feature selection for a wide-range of applications.
\end{abstract}
\section{Introduction}
Feature selection is a well known problem in data analysis and statistics. With the growth of data in almost all scientific disciplines and wide applications of machine learning in the industry, its importance has increased substantially. There are two main usages of feature selection. The first is to determine the ``optimal'' set of features that is optimal in the sense of being the smallest subset while maintaining high prediction accuracy. This usage is particularly important for applications of machine learning where resources, such as memory or allowed computation time, are constrained. Moreover, in certain scenarios feature selection can also eliminate the "noise" variables and yield higher prediction accuracy. The second usage aims to find the entire set of features that are related to an underlying phenomenon. The literature here often refers to features as {\em covariates} or {\em predictors} and the underlying phenomenon is often represented with labels referred to as {\em response variable} or {\em traits}. This usage is particularly important in neuroimaging and biology research, where sifting through large amount of measurements to determine the ones that are informative regarding the labels, {\em relevant ones}, has become part of the everyday routine. The analysis and the proposed model in this article applies to both of these usages and it is particularly important in the context of the second one.  

Random Forests \cite{Amit1997,Breiman2001}, among other techniques, have shown great potential for feature selection in a wide range of problems in neuroimaging \cite{Menze2009,Genuer2010,Langs2011}, biology \cite{Lunetta2004,Svetnik2004,Jiang2004,Diaz-Uriarte2006,Archer2008,Tang2009,Yang2009} and others \cite{Genuer2011,Gregorutti2013}. Compared to univariate analysis and other multivariate methods random forest has certain advantages. First, as a multivariate method, contrary to univariate methods, it has the potential to detect multivariate interactions if they exist as shown by Lunetta \etal in their simulation studies in \cite{Lunetta2004}. Second, it can naturally deal with high-dimensional problems, i.e. problems with high number of features and low number of samples, where other multivariate methods often need additional regularizations, such as sparsity for logistic regression or SVM, which make problems computationally tractable but not necessarily model the underlying phenomena. Third, it can explicitly work with raw features without requiring any dimensionality reduction methods that transform the features to a lower dimensional space and hence get rid of ``irrelevant'' dimensions but also impair interpretability, such as principal component or independent component analyses. Fourth, without much modification it can be applied to a variety of problems, such as regression, multi-category classification and unsupervised learning or clustering \cite{Konukoglu2013,Criminisi2012}. Lastly, without much added computational burden it offers a variety of measures of feature importance that indicate the relevance of the feature for the underlying phenomenon and task 

Despite the popularity and success of Random Forest for feature selection, feature importance measures still lack a very crucial ingredient: {\em false positive rate control}. Random Forest can produce mainly three different types of importance measures \cite{Breiman2001,Strobl2007}: selection frequency, Gini importance and permutation importance. Researchers in the last decade have proposed some variations of these in \cite{Deng2012,Deng2013,Wang2010,Strobl2008b,Zhou2010,Sandri2006}. All these measures can be used to rank features. These rankings however, only provide relative information as to which feature is more important than others. To actually identify individual features as relevant or irrelevant one needs to determine a threshold on the ranking. This applies to all uses of feature selection with Random Forests, whether the relevant feature assignments will be used directly for interpretation or they will be integrated in backward \cite{Svetnik2004,Diaz-Uriarte2006, Jiang2004,Menze2009,Gregorutti2013} or forward \cite{Genuer2010} feature selection wrappers. 
To date there is no principled way to determine such thresholds. Users of Random Forest have to use arbitrary heuristics knowing that they have substantial impact on the application-specific interpretations made on the relevant features as well as on the efficiency of wrapper methods. One natural and principled way to determine a threshold for a given importance measure would be to quantify the expected number of false positive rates, in other words the expected number of irrelevant features that would get assigned as relevant, for a given threshold value. With such a quantification one would simply limit the false positive rate at a desired level and the corresponding threshold can be computed. Such a quantification would also allow informed interpretations of the relevant features as well as allow wrapper algorithms eliminate irrelevant features faster and build adaptive schemes. However, currently there is no principled, efficient and satisfactory way to quantify false positive rates neither. This drawback is the motivation for the analysis and the model we present here. 

There has been three attempts to solve the problem of determining thresholds and quantifying false positive rates. First, Breiman and Cutler in \cite{Breiman2008} proposed a statistical test for permutation importance, however, Strobl and Zeileis in \cite{Strobl2008} demonstrated that this test has very undesirable properties, such as decreasing power with increasing sample size and boundlessly increasing normalized z-scores with increasing number of trees. Second, Rodenburg \etal in \cite{Rodenburg2008} and Altmann \etal in \cite{Altmann2010} proposed to use conventional permutation testing, where they permute the labels while keeping feature vectors intact, to compute significance for importance measures and determine thresholds. Lastly, Hapfelmeier and Ulm in \cite{Hapfelmeier2013} proposed a variation of permutation testing where they permuted the individual variables across samples while keeping the labels and other features intact. Permutation testing is a sound approach, however, for large scale problems its computational burden can be a limiting factor. Especially, permuting individual variables can quickly become infeasible. This computational drawback might also be the reason why permutation testing has not been used within wrappers so far. 

In this article, we propose a novel formulation for false positive rate control for one of the most basic importance measure in random forests: the feature selection frequency. Our formulation approximately models the feature selection process in the training phase of a random forest under the null hypothesis that there is no relationship between the features and the response variables. As such, it allows us to calculate the probability of an irrelevant feature being selected a given number of times under the null hypothesis. This gives rise to a principled way to determine thresholds and associate false positive rates at no additional computational cost. We derive two models for two most popular training strategies: random feature subset selection happening (i) at every node and (ii) at every tree. In our experiments we focus only on high-dimensional problems since interesting problems in neuroimaging and biology are of such nature, and Random Forest is especially useful and more advantageous than alternatives in such problems. We focus our attention on simulation experiments where ground truth data is available and therefore proper analysis can be performed. We demonstrate our approach on two sets of experiments. The first set assumes {\em independence} between features and the second set assumes a complex correlation structure that is computed from a real dataset of cortical thickness measurements, a type of measurement very widely used in neuroimaging. Our experiments demonstrate that the proposed model does not share the same drawbacks as the statistical test proposed in \cite{Breiman2008}, and it yields similar False Positive Rates and lower False Negative Rates than the permutation testing as proposed in \cite{Altmann2010} at no additional computational cost. The proposed method can be used in any study, retrospective or prospective, and for any type of problem Random Forest can be used in. 

The article structure is as follows. We first present a brief overview of random forests, different feature importance measures and motivate our approach. Then, we present our model and theoretical analysis. Following that we present experimental analysis and conclude with discussions and conclusions.  
\section{Random Forest and Feature Importance Measures}\label{sec:RF}
Random Forest (RF) is a generic technique that can be applied to various pattern analysis tasks including but not limited to classification and regression. It became very popular in the recent years in computer vision, medical image analysis and bioinformatics due to its good generalization properties and efficiency in performing predictions. Beyond its capabilities as a predictor, recently, RF has also become an important tool for feature selection. Here we provide a brief overview of RF with the focus on feature selection. 
\subsection{Notation}
We first introduce the basic notation we adopt throughout the article. We represent labels (response variables) with $y\in\mathbb{R}$ and features (covariates) with $\mathbf{x}=[x_1,...,x_F]^T\in\mathbb{R}^F$, where $F$ is the feature dimension and we denote the set of feature indices with $\mathcal{F}=\{1,...,F\}$. A dataset of samples is denoted by $\mathcal{S}=\{(y_i,\mathbf{x}_i)\}_{i=1,...,M}$, where $M$ is the sample size. 
\subsection{Random Forest}  To be able to describe the proposed model in full detail and provide links between different importance measures we choose to present a brief summary of random forests in this section. We suggest readers who are comfortable in Random Forest training to skip this section and go to Section~\ref{sec:model}. 
A random forest is an ensemble of decision trees~\cite{Amit1997,Breiman2001}, where each tree is a set of hierarchically nested binary tests represented with a directed graph. Figure~\ref{fig:tree}(a) shows an example of a decision tree, where $t_i$ denotes the $i^{th}$ internal node and $l_j$ denotes the $j^{th}$ leaf node. The directed edges connecting nodes represent the nesting structure of the hierarchy. 
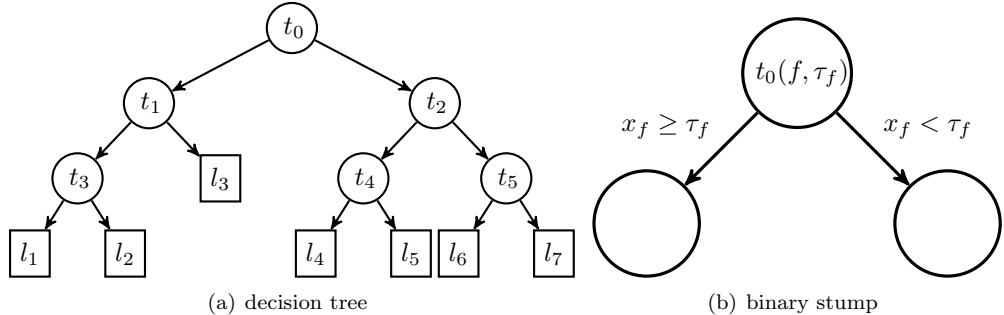
\begin{figure}[!htb]
\begin{center}
\subfigure[decision tree]{\tikzset{
	treenode/.style = {align=center, inner sep=0pt, text centered, font=\sffamily},
	arn_n/.style = {treenode, circle, black, font=\sffamily\bfseries, draw=black, 	text width=1.75em, thick},
	arn_x/.style = {treenode, rectangle, draw=black, minimum width=1.5em, minimum height=1.75em, thick}
}
\begin{tikzpicture}[->,>=stealth',thick,level/.style={sibling distance=3.8cm/#1,level distance=1.0cm}]
\node [arn_n]{$t_0$}
	child{node [arn_n]{$t_1$}
		child{node [arn_n]{$t_3$}
			child{node [arn_x]{$l_1$}}
			child{node [arn_x]{$l_2$}}
		}
		child{node [arn_x]{$l_3$}}
	}
	child{node [arn_n]{$t_2$}
		child{node [arn_n]{$t_4$}
			child{node [arn_x]{$l_4$}}
			child{node [arn_x]{$l_5$}}
		}
		child{node [arn_n]{$t_5$}
			child{node [arn_x]{$l_6$}}
			child{node [arn_x]{$l_7$}}
		}
	};

\end{tikzpicture}}
\subfigure[binary stump]{\tikzset{
	treenode/.style = {align=center, inner sep=0pt, text centered, font=\sffamily},
	arn_n/.style = {treenode, circle, black, font=\sffamily\bfseries, draw=black, 	text width=4.0em, very thick},
	arn_x/.style = {treenode, rectangle, draw=black, minimum width=1.5em, minimum height=1.5em}
}
\begin{tikzpicture}[->,>=stealth',very thick,level/.style={sibling distance=4.0cm/#1,level distance=2.0cm}]
\node [arn_n]{$t_0 (f, \tau_f)$}
	child{node [arn_n]{} edge from parent node[above left]{$x_f \geq \tau_f$}}	
	child{node [arn_n]{} edge from parent node[above right]{$x_f < \tau_f$}}
;
\end{tikzpicture}}
\end{center}
\caption{\label{fig:tree}(a) Schematic for (a) a decision tree example and (b) a binary stump.}
\end{figure} 
Each internal node holds a binary test that is applied on the features. In this article we are interested in the simplest type of binary test, the binary stump: 
\begin{equation}
\nonumber t_0 = t_0(f, \tau_f)(\mathbf{x}) \triangleq \left\{\begin{array}{cl}
0 ,x_f \geq \tau_f\\
1 ,x_f < \tau_f
\end{array}\right.,
\end{equation}
where $f$ is an index and $\tau_f\in\mathbb{R}$ is a threshold. While testing each tree produces an independent prediction for a test sample. The sample is pushed through the root node $t_0$ of a tree where the binary test is applied to its features. Based on the result the sample is sent to one of $t_0$'s two children, see Figure~\ref{fig:tree}(b). Repeating this process at every upcoming node the sample traverses the tree and reaches a leaf node $l_j$. At the leaf node a prediction for this sample is made based on the training samples that have traversed the tree and reached the same leaf node. The same process is repeated for each tree in the forest and the predictions are aggregated. It is easy to see that the binary tests are the components that decide on the path a sample will take in a tree and hence play a great role on the prediction for that sample. 

In the training phase each tree is constructed independently to optimize for prediction accuracy. This is done using a training dataset and in a greedy manner. The basic algorithmic blocks in training a tree are the node {\em optimization} and {\em split}. Let us assume that we have already constructed a tree until the $n^{th}$ node and let us denote the training samples that arrive at this node with $\mathcal{S}_n$. The node optimization is defined as  
\begin{eqnarray}
\label{eqn:node_optimization} f^*, \tau_f^* &=& \argmax_{f\in\mathcal{F}_n,\tau_f\in\mathbb{R}} G\left(\mathcal{S}_n,\mathcal{S}_n^R,\mathcal{S}_n^L\right),\\
\nonumber \mathcal{S}_n^R &=& \left\{ (y,\mathbf{x})\in \mathcal{S}_n | x_f < \tau_f \right\}, \\ 
\nonumber \mathcal{S}_n^L &=& \left\{ (y,\mathbf{x})\in \mathcal{S}_n | x_f \geq \tau_f\right\},
\end{eqnarray}
where $G(\cdot,\cdot,\cdot)$ is the objective function, for which the exact definition depends on the task at hand\footnote{$G(\cdot,\cdot,\cdot)$ can be basic such as the Gini score for binary classification or the variance reduction for continuous regression \cite{Breiman2001}. It can also be more complex, such as the one based on $\chi^2$ statistic proposed in \cite{Wang2010}, the one based on the conditional inference framework proposed in \cite{Hothorn2006} or the one based on clustering as in \cite{Konukoglu2013}.}, and $\mathcal{F}_n\subseteq \mathcal{F}$ is the subset of feature indices node $n$ will be optimized over. The node split is defined as setting the binary test at this node to $t_n (f^*,\tau_f^*)$, creating two new {\em child} nodes budding from node $n$ and splitting $\mathcal{S}_n$ among the child nodes based on the test. Not every node is split and this is formulated using a set of constraints on the maximum depth allowed, number of samples remaining in the node $n$ or the optimal value of $G(\cdot,\cdot,\cdot)$. Nodes that are not split further are set as leaf nodes. To construct a tree from scratch, one starts with a single node - the root node - and places all the training samples in it. Then iteratively applies the node optimization and split on the root node and all the other nodes that are created until there are no more nodes to split. 

In a random forest the construction of each tree has a random component that makes individual trees different from each other. This random component can be through bootstrapping (or subsampling) of samples, i.e. training each tree with a bootstrap sample of the entire training dataset, feature subsampling, training each node or tree using only a random subset of the entire feature set, or a combination of these. The randomization reduces the statistical correlation between trees and this in turn has been shown to improve generalization properties \cite{Breiman2001,Criminisi2012}. 
\subsection{Feature Importance Measures}
Random Forest training allows for several natural ways to assign importance measures to individual features. Here we provide a brief overview of the most basic and popular measures and explain the reasons why we choose to focus on the feature selection frequency. We note that to refer to different measures we use the same terminology Strobl \etal used in \cite{Strobl2007}.  
\subsubsection{Feature Selection Frequency}
Feature selection frequency is the number of times a given feature is selected as the optimal based on Equation~\ref{eqn:node_optimization} across all the nodes in the entire forest:
\begin{equation}
\nonumber I_{SF}(f) = \sum_T\sum_n \mathbf{1}(f=f^*),
\end{equation}
where the first sum is over all the trees, the second sum is over all the nodes and $\mathbf{1}(f=f^*) = 1$ if $f^*=f$ at that node and $0$ otherwise. The underlying hypothesis is that relevant features will be statistically related to the label and therefore, will be useful for prediction. So they should yield better $G(\cdot,\cdot,\cdot)$ values than non-relevant features and get selected more often. Based on this hypothesis ranking features based on their selection frequency is possible. However, in order this ranking to be useful one needs a threshold that will be used to split features into relevant and non-relevant groups. Currently there is no principled way to determine such a threshold and researchers use heuristics. This limitation is the focus of this article. 
\subsubsection{Gini Importance}
Gini importance measure, \cite{Breiman2001}, is an extension of the selection frequency, where not only the number of selection but also the optimal $G(\cdot,\cdot,\cdot)$ values are recorded.\footnote{We note that we use the term ``Gini'' rather loosely here. It was first  proposed for classification problems where Gini coefficient simply refers to the well known Gini impurity. However, the same principle for measuring feature relevance can easily be extended to arbitrary $G(\cdot,\cdot,\cdot)$ functions, for instance see \cite{Wang2010}.} The importance of a feature is defined as the following sum: 
\begin{equation}
\nonumber I_{GI}(f) = \sum_T\sum_n \mathbf{1}(f=f^*) \max G(\mathcal{S}_n,\mathcal{S}_n^R,\mathcal{S}_n^L),
\end{equation}
which not only counts but also sums the optimal $G(\cdot,\cdot,\cdot)$ values. The underlying hypothesis is very similar to the selection frequency. Indeed this score can be viewed as a weighted selection frequency. Gini importance can too be used to rank features and since it includes information on the optimal $G(\cdot,\cdot,\cdot)$ associated with the feature, it can theoretically provide a more informed ranking.  However, it shares the same drawback as the selection frequency: lack of a principled way to determine a threshold. Moreover, the problem is more drastic here because in determining such a threshold one has to consider the numerical properties of the function $G(\cdot,\cdot,\cdot)$ beyond the basic assumption that more relevant features will yield better $G$ values. 
\subsubsection{Permutation Importance}
In training if subsampling or bootstrapping of samples is used then during the construction of each tree a subset of samples are left out. These out-of-bag (OOB) samples can be used to compute a prediction error for the entire forest. Based on this, the permutation importance measure, \cite{Breiman2001,Breiman2008}, is the increase in the OOB error-rate when a particular feature is permuted across the samples. The underlying hypothesis is that permuting the feature will break its statistical links with the label and the other features while maintaining its marginal distribution. If the feature is relevant and used for prediction in the forest then permuting its values will adversely affect the prediction quality. The amount of increase in OOB error-rate quantifies this effect and is defined as the importance measure. 
Breiman and Cutler in \cite{Breiman2008} proposed a statistical test for permutation importance, which can be used to determine a threshold for this measure. However, as Strobl and Zeileis in \cite{Strobl2008} showed this statistical test has very undesirable properties. The statistical power of the test decreases with increasing sample size and the normalized z-scores associated with importance measures increase proportional to $\sqrt{\textrm{number of trees}}$. Due to these drawbacks this test has not become popular in the literature. In summary, permutation importance too lacks a useful principled way to determine a threshold. And as was the case with Gini importance, the complexity of this measure makes statistical modeling of this measure quite difficult. 

There are two important remarks that we should make here. 
\begin{remark}
Although their definitions are quite different, empirical studies presented in \cite{Strobl2007,Archer2008,Zhou2010} have shown that the three measures described above yield very similar rankings to features. 
\end{remark}
\begin{remark}
Strobl \etal in \cite{Strobl2007} shows that all these measures in their most basic forms have bias towards features with larger number categories or continuous variables. However, this bias can be eliminated using the conditional inference framework Hothorn \etal presented in \cite{Hothorn2006}. The model we will present shortly can also be used with the same conditional inference framework.   
\end{remark}

Just like these three basic measures, their variations also suffer from the same drawback. To date the only way to determine thresholds or statistical significance for Random Forest based importance measures is permutation testing \cite{Rodenburg2008,Altmann2010,Hapfelmeier2013}. However, these methods are computationally expensive and therefore, cannot easily be used in every study and cannot be integrated in wrapper methods. 

In the next section we will present a simple probabilistic model for selection frequency and show how it can be used for determining thresholds and associating false positive rates. Before delving into the details we would like to note the reasons why we choose to model selection frequency, which is not that commonly used in the literature for feature selection. First of all, as shown in empirical studies all these three measures provide very similar rankings. This is not very surprising since both Gini and permutation importance measures rely on the fact that a feature is chosen at a node. Indeed, they can be seen as weighted versions of the selection frequency. Second, a statistical model of the selection frequency applies to all the different types of problems a Random Forest can be used to solve, such as regression, multi-label classification and unsupervised learning. On the contrary, possible statistical models of the other measures need to take into account the exact form of the objective function (or the loss function) and therefore, they cannot be generally applicable. Lastly, for small sample sizes, as also noted by Archer and Kimes in \cite{Archer2008}, subsampling or bootstrapping of the samples during training might not be desirable, which limits the use of permutation importance. 
\section{A Model of False Positive Rates for Selection Frequency}\label{sec:model}
The main question we want to answer in this section is {\em ``What is the probability that a feature $f$ is selected $k$ times in a Random Forest under the null hypothesis?''} with the null hypothesis being {\em ``There is no statistical relationship between features (covariates) and labels (response variables).''} Answering this question will allow us to estimate false positive rates for a given threshold on selection frequency. Consequently, it will provide us a principled way to determine thresholds to split relevant and non-relevant features. 

Our strategy is to build an approximate probabilistic model of selection frequency that takes into account the details of the training procedure, such as the feature subsampling scheme, number of trees and the size of the random feature subset each node optimizes over\footnote{The $mtry$ variable in the randomForest function in R}. Once this model is built we can compute the expected rate of false positives for a given threshold on selection frequency and consequently determine thresholds by limiting this rate. We take a bottom-up approach and start by modeling the probability of a feature $f$ being selected at a given node and build-up the model from there. 

The randomness in the feature selection process is due to the random nature of forest training. 
As mentioned in Section~\ref{sec:RF} there are two components that provide this random nature: feature subsampling and sample subsampling or bootstrapping. 
In our approach we explicitly formulate the feature subsampling and this is the source of the probabilistic nature of our model. 
We do not explicitly model the sample subsampling scheme and its effect on the selection frequency. 
There are two main reasons for this. 
First, the effect of sample subsampling on feature selection is much more difficult to explicitly formulate. 
Second, for small datasets bootstrapping or subsampling of the data might not be used at all when training a random forest. 
Feature subsampling on the other hand is used all the time. 
However, although we do not model sample subsampling explicitly the proposed model still applies whether a sampling scheme is used or not. 

There are various feature subsampling schemes that can be used in Random Forest training. For each of these schemes the model will be slightly different. Here, we provide the models for the two most popular approaches. Accordingly, Section~\ref{sec:experiments} will present results on both of the schemes. 

In our approach we make one assumption that applies to all the feature subsampling schemes we focus on. The approximate nature of our model is due to this assumption. Before delving into the details of the selection frequency models we first build the mathematical framework and explicitly state our assumption.
\subsection{Mathematical Setting}
We focus on the construction of an arbitrary node in an arbitrary tree during the training phase. The key idea is to notice that at this arbitrary node the index of the optimal feature, determined by optimizing Equation~\ref{eqn:node_optimization}, can be seen as a random variable whose randomness is due to feature subsampling. In order to build the mathematical setting that supports this view we construct a probability space that is shared among all the nodes across the forest. To do this we first define the sample space that arises naturally from the feature subsampling. Without loss of generality let us assume that every node has access to $F_n$ features. Then, let 
\begin{equation}
\nonumber \Omega=\left\{ \left\{f_{m_1},\dots,f_{m_{F_n}}\right\} |\ m_i \in \mathcal{F}, m_i \neq m_j\ \textrm{for}\ i\neq j\right\}
\end{equation}
the set of all subsets of size $F_n$ within $\mathcal{F}$ be the sample space. Note that any element of $\Omega$ corresponds to a specific subsampling of the feature space. Using this sample space we define the probability space $\left(\Omega,2^\Omega,P\right)$, where we construct the measure $P$ on any $\mathcal{A}\subseteq\Omega$ using the uniform measure on $\Omega$ as 
\begin{equation}
\nonumber P(\mathcal{A}) = \sum_{w\in\mathcal{A}} p(w),\ \sum_{w\in\Omega}p(w) = 1,\ p(w) = \textrm{constant}\ \forall w\in\Omega.
\end{equation}
In concrete terms $(\Omega,2^\Omega,P)$ models the random selection of a feature subset, where all subsamplings are equally likely. It can be used in both situations where feature subsampling occurs at each node independently or it occurs once per tree and all the nodes in the tree use the same feature subset. Using basic combinatorics it is easy to show that 
\begin{equation}
\nonumber p(w) = 1\Big/\dbinom{F}{F_n}.
\end{equation}

The link between $p(w)$ and the optimal feature index can be build by observing that the optimal feature is determined by optimizing an objective function over a subset of features, in other words a specific $w$. Let us focus on a node $n$ where $\mathcal{S}_n$ denotes the subset of the samples at that node. At this node we can define $f^*_n:\Omega\rightarrow\mathcal{F}$ to be the random variable 
\begin{equation}
\nonumber f^*_n(w) \triangleq \argmax_{f\in w}\max_{\tau_f\in\mathbb{R}} G(\mathcal{S}_n, \mathcal{S}_n^R, \mathcal{S}_n^L),
\end{equation}
where $\mathcal{S}_n^R$ and $\mathcal{S}_n^L$ are defined as in Equation~\ref{eqn:node_optimization}. The $f^*_n(w)$ mapping simply represents the solution of the node optimization for a given feature subset $w$. Using this notation we can compute the probability that a feature $f$ gets selected at a given node $n$ simply by 
\begin{equation}
\nonumber P(f\ \textrm{selected at } n) = P(f^*_n = f) \triangleq \sum_{w\in\Omega}p(w)\mathbf{1}\left(f^*_n(w) = f\right)=\left(\sum_{w\in\Omega}\mathbf{1}\left(f^*_n(w) = f\right)\right)\Big/\dbinom{F}{F_n},
\end{equation}
where $\mathbf{1}(\cdot)$ returns $1$ if its argument is true and $0$ otherwise. The fact that each subsampling is equally likely makes $P(f^*_n = f)$ simply proportional to the number of times $f$ becomes the optimal feature over all possibilities in $\Omega$. 

At this stage we are now ready to explicitly state our assumption that will be the basis of the selection frequency models we will present next. 
\begin{assumption}\label{assumption}
Under the null hypothesis that there is no statistical relationship between labels and features we assume that for all nodes, regardless of $\mathcal{S}_n$, and $\forall f\in\mathcal{F}$
\begin{equation}
\label{eqn:assumption} P(f^*_n=f) = \frac{1}{F}.
\end{equation}
\end{assumption}
It is important to understand where this assumption stems from to realize its limitations. In the hypothetical case where there are finite number of features and infinite number of samples, if $G(\cdot,\cdot,\cdot)$ is unbiased, in the sense that it will detect that there are no statistical relationships between labels and features, then under the null hypothesis all the features will provide the same $G$ value. As a result, the choice of the optimal feature will be a purely random choice hence the assumption holds. Now, there are two realistic cases. First is when there are many more samples than feature dimensions, i.e. $S >> F$, and in this case Assumption~\ref{assumption} will be very close to reality, see the empirical analysis in Section~\ref{sec:simple_problem}. In the second high-dimensional case, i.e. $S << F$, there will be spurious statistical relationships between features and labels and there will be features that will have higher probabilities to get selected. So, the validity of the assumption will be jeopardized for some features. The usual way to tackle the spurious statistical relationships is to use bagging or bootstrapping of the samples and these techniques solves the problem up to some extent. However, for small sample sizes even these techniques might not get rid of spurious relationships. In Section~\ref{sec:experiments} we analyze such cases and show that even in these cases Assumption~\ref{assumption} is still useful for discriminating relevant and non-relevant features.  

Another important point we should make here is about correlation structures between features. Assumption~\ref{assumption} does not say anything about correlations between features. However, it is not hard to imagine that strong correlations between features will harm the validity of the assumption especially in the second extreme case $S << F$. Nonetheless, in our experiments we demonstrate that even for complex correlations structures the models derived from Assumption~\ref{assumption} yield useful tools in determining thresholds.

It might seem odd that in the assumption $P(f^*_n=f)$ does not depend explicitly on $F_n$, the size of the feature subset. This seems particularly surprising considering Breiman and Cutler's statement in ``[$F_n$] is the only adjustable parameter to which random forests is somewhat sensitive.'' in \cite{Breiman2008}. However, as Breiman himself pointed out (\cite{Breiman2001}:14-16,\cite{Breiman2008}) the error rates and the ``strength''\footnote{see \cite{Breiman2001} for an exact definition} of the forest is fairly insensitive to $F_n$ over a large range. The insensitivity has also been observed in empirical studies for feature selection~\cite{Lunetta2004,Svetnik2004,Diaz-Uriarte2006}. In addition to these previous results, our experiments demonstrate that the quality of the models based on this assumption are not affected by changing $F_n$ for a large range around $\sqrt{F}$, the most commonly used setting. Nonetheless, it is important to note that in extreme cases, such as $F_n=F$, Assumption~\ref{assumption} might lose its validity. However, we do not concentrate on these cases since they lead to very correlated trees and therefore, harm the generalization properties of Random Forest in the first place. 

The above assumption has an important consequence on the conditional probabilities. Let us define $\mathcal{C}_n^f$ which denotes the event that feature $f$ is in the randomly selected subset for the node $n$. Then we can state the following proposition regarding conditional probabilities.
\begin{proposition}\label{proposition}
If $P(f^*_n = f) = 1 / F$ $\forall f\in\mathcal{F}$ and all nodes, then $P(f^*_n = f | \mathcal{C}_n^f) = 1 / F_n$.
\end{proposition}
\begin{proof}
$P(f^*_n=f) = P(f^*_n=f|\mathcal{C}_n^f)P(\mathcal{C}_n^f) + P(f^*_n=f|\bar{\mathcal{C}_n^f})P(\bar{\mathcal{C}_n^f})$, where $\bar{\mathcal{C}_n^f}$ is the complementary event. $f$ cannot be selected if it is not in the subsampled feature set, so, $P(f^*_n=f|\bar{\mathcal{C}_n^f})=0$. As a result, $P(f^*_n=f) = P(f^*_n=f|\mathcal{C}_n^f)P(\mathcal{C}_n^f)$. The number of sets of size $F_n$ one can draw among a set of size $F$ is a well-known identity in combinatorics: $\dbinom{F}{F_n}$. Based on that $P(\mathcal{C}_n^f)$ can be given as the ratio of number of sets of size $F_n$ that includes $f$ to the total number: $P(\mathcal{C}_n^f)=\dbinom{F-1}{F_n-1}\Big/\dbinom{F}{F_n}$. Therefore, 
\begin{equation}
\nonumber P(f^*_n=f|\mathcal{C}_n^f) = \frac{1}{F}\dbinom{F}{F_n}\Big/\dbinom{F-1}{F_n-1} = \frac{1}{F_n}.
\end{equation}
\end{proof}

In the following two sections we use Assumption~\ref{assumption} and the result given in Proposition~\ref{proposition} to construct two basic models of selection frequency for two different training schemes that are commonly used in applications of Random Forest. In both cases our goal is to build a model that can compute the probability that a feature is selected $k$ times under the null hypothesis. Since the models rely on the validity of the assumption they are approximate at best. 
\subsection{Training Strategy I}\label{sec:st1}
The first training scheme randomly chooses a new subset $\mathcal{F}_n\subseteq\mathcal{F}$ of size $F_n \leq F$ from the entire set at every node independently from the others. So, based on Assumption~\ref{assumption} the probability of a feature to be selected in any one node under the null hypothesis is given by $P(f^*_n = f) = 1/F$. A random forest is composed of many trees that are all independently constructed and $P(f^*_n=f)$ applies to all the nodes in all the trees. As a result the probability of a feature $f$ being chosen $k$ times in the entire forest follows a binomial distribution. If there are $T$ trees and each tree has $K$ internal nodes then this probability is given as
\begin{equation}
\label{eqn:modelI} P(\mathcal{C}_{k,T}^f) = \dbinom{TK}{k}\left(\frac{1}{F}\right)^k\left(1 - \frac{1}{F}\right)^{TK-k}.
\end{equation}
In the training phase of a Random Forest there is no constraint on the number of internal nodes a tree should have. The construction is usually constrained with the maximum tree depth or the minimum number of samples at a leaf node. So, each tree can have different number of internal nodes. Now, it is not difficult to extend the model given in Equation~\ref{eqn:modelI} to this scenario. However, here we take a more crude approach and set $K$ to the average number of internal nodes in the forest and use Equation~\ref{eqn:modelI}. The main reason for this is empirically we have not observed a single case where the crude model deviated by more than \%1 of the complex model. Moreover, the complex model  can be computationally very demanding and this might hinder its wide applicability. So, in the sake of simplicity we focus on the crude model in this work.  
 
Strategy I is the most commonly used strategy that is also implemented in the well known randomForest package in the R distribution.
\subsection{Training Strategy II}\label{sec:st2}
The second training scheme independently chooses a random feature subset for each tree but uses the same subset to train all the nodes within the tree. This strategy is particularly interesting for applications dealing with high-dimensional problems and utilizing parallel or distributed computing, such as medical image analysis and computer vision~\cite{Criminisi2012}.  

Modeling selection frequency in this scheme is slightly more complicated than strategy I but it is based on the same principles. We start by defining $\mathcal{C}_t^f$ to denote the event that a feature $f$ is selected in the feature subset to train the tree $t$. This event actually has the same probability distribution as $\mathcal{C}_n^f$ and it is given by 
\begin{equation}
\nonumber P(\mathcal{C}_t^f) = \dbinom{F-1}{F_n-1}\Big/\dbinom{F}{F_n}.
\end{equation}
Based on Assumption~\ref{assumption}, from Proposition~\ref{proposition}, we know that given $\mathcal{C}_t^f$ the probability that $f$ is selected as the optimal feature in any node of the tree is given by $1/F_n$. Using this we can compute the probability that $f$ gets selected as the optimal feature at $k$ nodes of a tree that has $K$ nodes in total. Let us denote this event with $\mathcal{C}_{k/t}^f$. Its probability can be computed with another binomial distribution as 
\begin{equation}
\nonumber P(\mathcal{C}_{k/t}^f) = \dbinom{K}{k}\left(\frac{1}{F_n}\right)^k\left(1 - \frac{1}{F_n}\right)^{K-k}.
\end{equation}
Combining these two probabilities we can compute the probability of $f$ being selected as the optimal feature at $k$ nodes in a given tree $t$ by $P(\mathcal{C}_{k,t}^f) = P(\mathcal{C}_{k/t}^f)P(\mathcal{C}_t^f)$. 

The value that we are really interested in is the probability that $f$ gets selected at $k$ nodes in the entire forest, $P(\mathcal{C}_{k,T}^f)$. Assuming there are $T$ trees, this condition can be written as $\sum_{t=1}^T k_t = k$, where $k_t$ is $f$'s selection count in tree $t$. It is clear that for any $k>0$ there are multiple selection count assignments that can satisfy the sum. Therefore, to compute  $P(\mathcal{C}_{k,T}^f)$ we need to add the probabilities of all possible assignments of $\{k_t\}$. To do this we use {\em integer partitions} (see \cite{Wilf2000} for a basic treatment). A partition of an integer $k$, $\lambda\vdash k$, is a list of numbers $\lambda=[l_1,l_2,\dots,l_{M(\lambda)}]$ such that $\sum_{m=1,M(\lambda)} \lambda_m = k$. All integers larger than $1$ have several partitions and we denote the set of partitions of $k$ with $\Lambda(k) = \{\lambda|\lambda\vdash k\}$. The order of numbers in a partition is not important so two lists with the same elements but in different order correspond to the same partition. To compute $P(\mathcal{C}_{k,T}^f)$ we compute the probability for each partition of $k$ and add them all up. 

To compute the probability of a given partition $\lambda$, which we denote with $P(\mathcal{C}_{\lambda}^f)$, we need the counts of each integer in $\lambda$, e.g. number of 1's, 2's and so on. This, we can simply obtain using:
\begin{equation}
\nonumber q_{\lambda}^k (\xi) = \sum_{l_m\in\lambda}\mathbf{1}\left(l_m = \xi\right),\  \xi\in\{1,\dots,k\},
\end{equation}
where $\mathbf{1}(\cdot)$ is the same as above, i.e. returns 1 if the argument is true and 0 otherwise. 
The basic idea is now to compute the probability that we have $q_{\lambda}^k(\xi)$ trees among $T$ in which $f$ is selected $\xi$ times for all $\xi$. There are two conditions here. The first one is the trivial case where $\sum_{\xi} q_{\lambda}^k(\xi) > T$ and thus $P(\mathcal{C}_{\lambda}^f) = 0$. The more interesting case $\sum_{\xi} q_{\lambda}^k(\xi) \leq T$ can be computed using the multinomial distribution. To do this we already know how to compute $P(\mathcal{C}_{\xi,t}^f)$ when $\xi > 0$ and we only need the probability that $f$ gets selected $0$ times in a tree and this is given by
\begin{equation}
\nonumber P(\mathcal{C}_0^f) = P(\mathcal{C}_{0,t}^f) + 1 - P(\mathcal{C}_t^f).
\end{equation}
Combining all we can compute the probability of a partition $\lambda$ using the multinomial distribution:
\begin{equation}
\nonumber  P(\mathcal{C}_{\lambda}^f) = \left\{
\begin{array}{cl}
\frac{T!}{\prod_{\xi=1}^k q_{\lambda}^k(\xi)!\left(T-\sum_{\xi=1}^k q_{\lambda}^k(\xi)\right)!}\prod_{\xi=1}^k \left\{P(\mathcal{C}_{\xi,t}^f)^{q_{\lambda}^k(\xi)}\right\}P(\mathcal{C}_0^f)^{\left(T - \sum_{\xi=1}^k q_{\lambda}^k(\xi)\right)}, & \sum_{\xi}q_{\lambda}^k(\xi) \leq T \\
0, & \sum_{\xi}q_{\lambda}^k(\xi) > T
\end{array}\right.
\end{equation}
where the number of internal nodes per tree is taken to be $K$, the average across the forest. 
Finally summing the probabilities for all partitions 
\begin{equation}
\label{eqn:modelII} P(\mathcal{C}_{k,T}^f) = \sum_{\lambda\in\Lambda(k)}P(\mathcal{C}_{\lambda}^f).
\end{equation}

As was the case for the model for Strategy I, we can also extend the model here to the case where each tree has different number of internal nodes. However, once again we have not observed an empirical advantage of complicating the model so we focus on the simpler case given in Equation~\ref{eqn:modelII}.
\subsection{Determining Thresholds and Associating False Positive Rates} $P(\mathcal{C}_{k,T}^f)$ is the component that we need to be able to determine a threshold for selection frequency that can split relevant from non-relevant features. It gives us an approximate probability that a feature $f$ is selected $k$ times in a forest of $T$ trees under the null hypothesis. With this, we can answer the basic question that will help us determine a threshold: {\em What is the probability that a feature gets selected more than $\kappa$ times in $T$ trees under the null hypothesis?}. This probability can simply be given as
\begin{equation}
\label{eqn:thresholdI} P(\mathcal{C}_{k>\kappa,T}^f) = 1 - \sum_{\xi=0}^\kappa P(\mathcal{C}_{\xi,T}^f).
\end{equation}
In the model $P(\mathcal{C}_{k>\kappa,T}^f)$ is the same for all the features and therefore, this probability also provides the model prediction for the expected false positive rate under the null hypothesis. Specifically,
\begin{equation}
\nonumber E[\textrm{Number of False Positives}] = E\left[\sum_{f\in\mathcal{F}}\mathbf{1}\left(n_f > \kappa\right)\right] = P(\mathcal{C}_{k>\kappa,T}^f)\cdot F.
\end{equation}
So, our proposed approach for choosing a threshold $\kappa$ is to limit the expected false positive rate. For a desired rate $\alpha$ we can compute the threshold as
\begin{equation}
\label{eqn:thresholdII} \kappa^* = \min_{\kappa\in[0,KT)} \kappa \textrm{ subject to }P(\mathcal{C}_{k>\kappa,T}^f) \leq \alpha.
\end{equation}
$\kappa$ is a non-negative integer and empirically $\kappa << KT$. So, this optimization problem can be solved using exhaustive search without much additional computational burden. 
\subsection{In the Existence of Relevant Features} 
The models given in Equations~\ref{eqn:modelI} and~\ref{eqn:modelII} rely on Assumption~\ref{assumption}. Therefore, the thresholds we compute based on these models also rely on the assumption. We already discussed the validity of this assumption for different scenarios under the null hypothesis. It is also important to note what we should expect in the presence of relevant features. In this case the assumption $P(f=f^*) = 1/F$ will change. The relevant features will be selected more often as the optimum and we will observe $P(f_{relevant} = f^*) > 1/F$. Consequently we expect a drop in the probability that a non-relevant feature gets selected $P(f_{non-relevant} = f^*) < 1/F$. As a result, the thresholds determined using Equations~\ref{eqn:thresholdI} and~\ref{eqn:thresholdII} will yield lower false positive rates than the predicted values. Our experimental analysis show that this is indeed the case. 
\section{Experiments}\label{sec:experiments}
In this section we present a thorough empirical analysis of the proposed method. We use synthetic datasets, where quantitative analysis of false positive rates and false negative rates is possible. We use two different models to generate data: one where features are independent and identically distributed, and the other where feature vectors are synthetically simulated cortical thickness maps that have a complex correlation structure. In both types of experiments we work on binary classification problems, i.e. $y\in\{0,1\}$. However, we would like to note that since the proposed method does not depend on the exact definition of $G(\cdot,\cdot,\cdot)$, the method can also be used for other types of problems, e.g. continuous regression or clustering. In all the experiments we use sample bagging, subsampling {\em without} replacement, and feature subsampling during the training of the Random Forest. In our analysis we observed that there is not much difference between bagging rates of $0.632$ and $0.5$, so we only present results for $0.5$. 

As the forest implementation we use the publicly available software neighbourhood approximation forests (NAF) ~\cite{Konukoglu2013} \footnote{Available with a python interface at http://www.nmr.mgh.harvard.edu/$\sim$enderk/software.html}, which is a variant of Random Forests that can be applied to a variety of problems using the same objective function $G$.  

We first provide details on the data generation models and then present the results. 
\subsection{Independent Case} 
\subsubsection{Data Generation Model}
The datasets with independent features are generated with the following model. 
\begin{eqnarray}
\nonumber \mathbf{x}_{non-relevant}&\propto& \mathcal{N}(\overrightarrow{\mathbf{0}}_{F-N},\sigma^2\mathbf{I}_{F-N}),\ \sigma=5.0\\
\label{eqn:relevant} \mathbf{x}_{relevant}&\propto& \mathcal{N}(2y\rho\sigma/\sqrt{1-\rho^2}\overrightarrow{\mathbf{1}}_{N},\sigma^2\mathbf{I}_N),\ \sigma=5.0\\
\nonumber \mathbf{x} &=& [\mathbf{x}_{nonrelevant}^T, \mathbf{x}_{relevant}]^T\\
\nonumber y&\propto& \mathcal{B}(1, 1/2),
\end{eqnarray}
where $N$ is the number of relevant features, $\overrightarrow{\mathbf{0}}_N$ and $\overrightarrow{\mathbf{1}}_{N}$ are vectors of zeros and ones of dimension $N$ respectively, $\rho$ is the ``relevance'' parameter, $\mathbf{I}_N$ is the identity matrix of dimension $N\times N$, $\mathcal{N}$ denotes the normal and $\mathcal{B}$ denotes the binomial distribution. The mean for each relevant variable $2y\rho\sigma/\sqrt{1-\rho^2}$ ensures that the Pearson's correlation coefficient between a relevant feature and the label is $\rho_{[\mathbf{x}_{relevant}]_i,y} = \rho$. With this model we generate datasets of different size $S$, different $N$ and different feature dimensions. Each dataset is balanced where there are $S/2$ samples with $y=1$ and $y=0$. 

\subsubsection{Low-dimensional Problems: $S >> F$}\label{sec:simple_problem}
The first set of analysis are on ``low-dimensional'' problems in which there are more samples than features. These problems are statistically ``easier'' for feature selection because it is less likely to have spurious statistical relationships when $S > F$. In this setting we expect Assumption~\ref{assumption} to be very close to reality and the models described in Sections~\ref{sec:st1} and~\ref{sec:st2} to give accurate false positive rate estimates for a given selection frequency threshold $\kappa$. So we evaluate the models' false positive rate estimates by comparing them to the observed rates when there are no relevant features, i.e. when the null hypothesis is true. 

We set the sample size and feature dimension as $(S,F)=(200,20)$ and experiment with different number of features per node around $\sqrt{20}$, $F_n=(3,5,7)$, and different number of trees, $T = (20,40)$. We set $N=0$ to obtain datasets with no relevant features. For each of six parameter settings we generate 100 datasets and for each dataset we train a forest and record false positive rates for different selection frequency thresholds. For each forest we also compute model predictions of the false positive rate estimates for the same thresholds. We repeat the same experiments for both of the training strategies. 
\begin{table}[!htb]
\small
\begin{tabular}{|c|ccc|}
\hline
\multirow{3}{*}{\begin{sideways}Strategy I\end{sideways}} &
\includegraphics[width=0.28\linewidth]{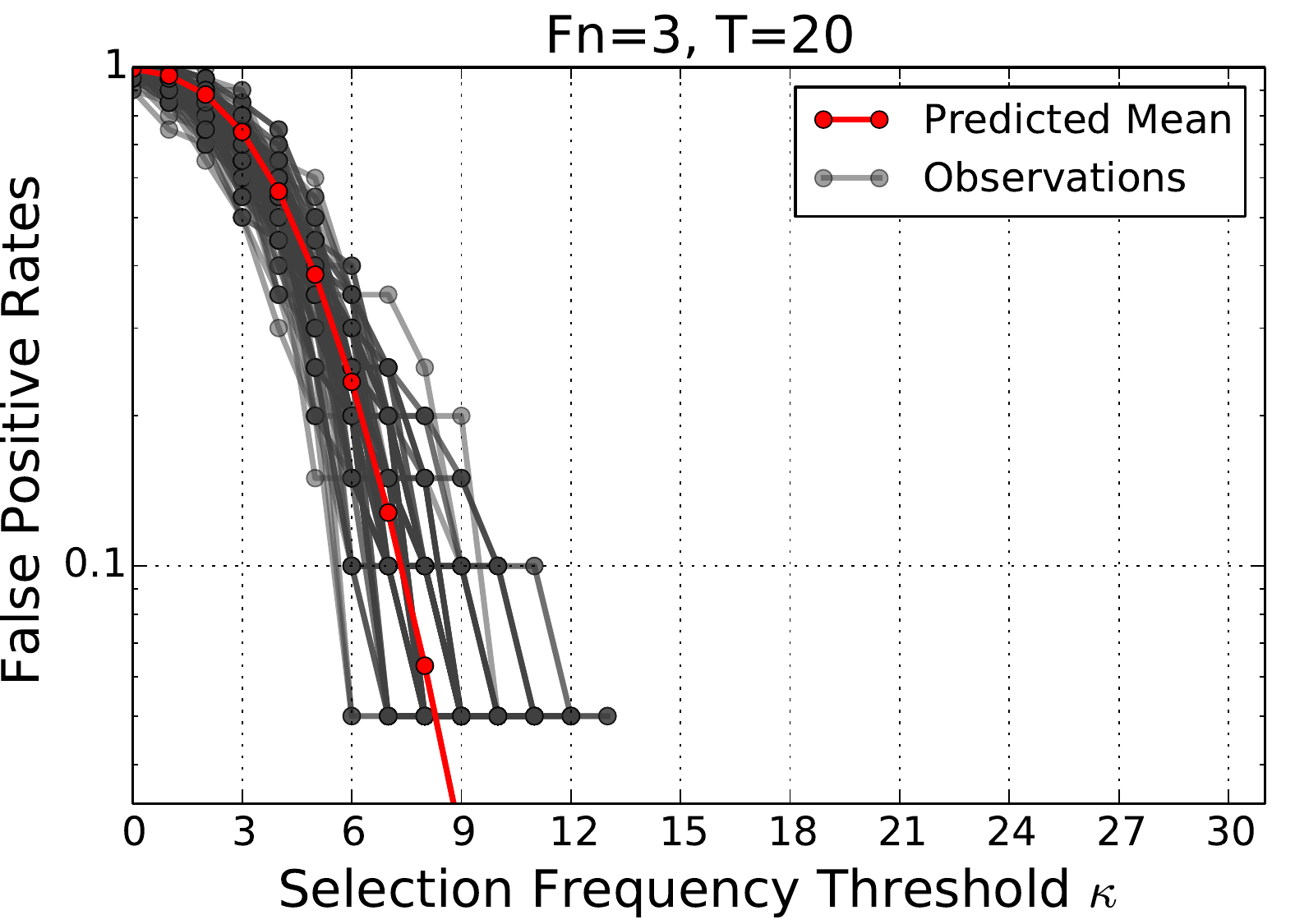} & 
\includegraphics[width=0.28\linewidth]{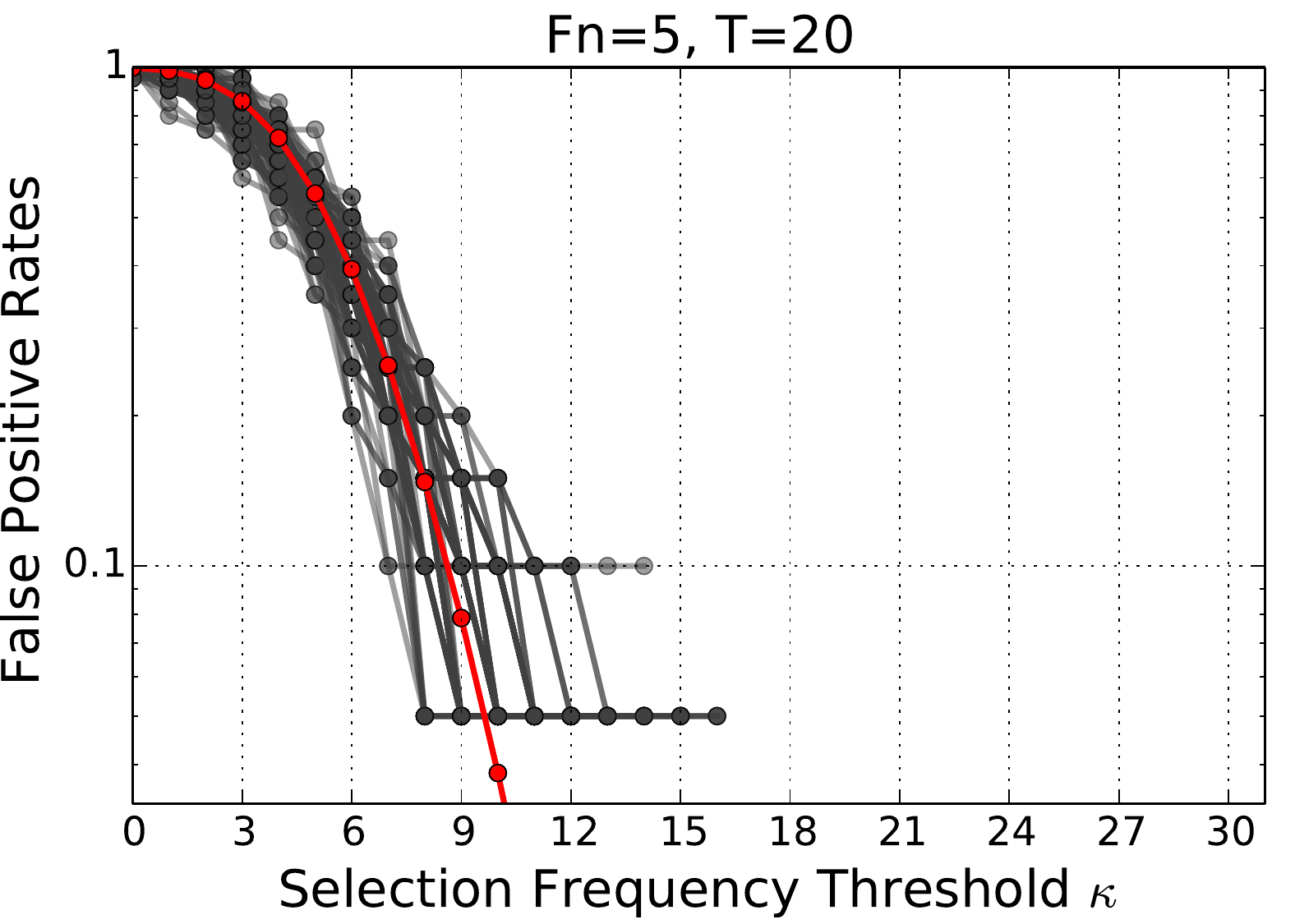} & 
\includegraphics[width=0.28\linewidth]{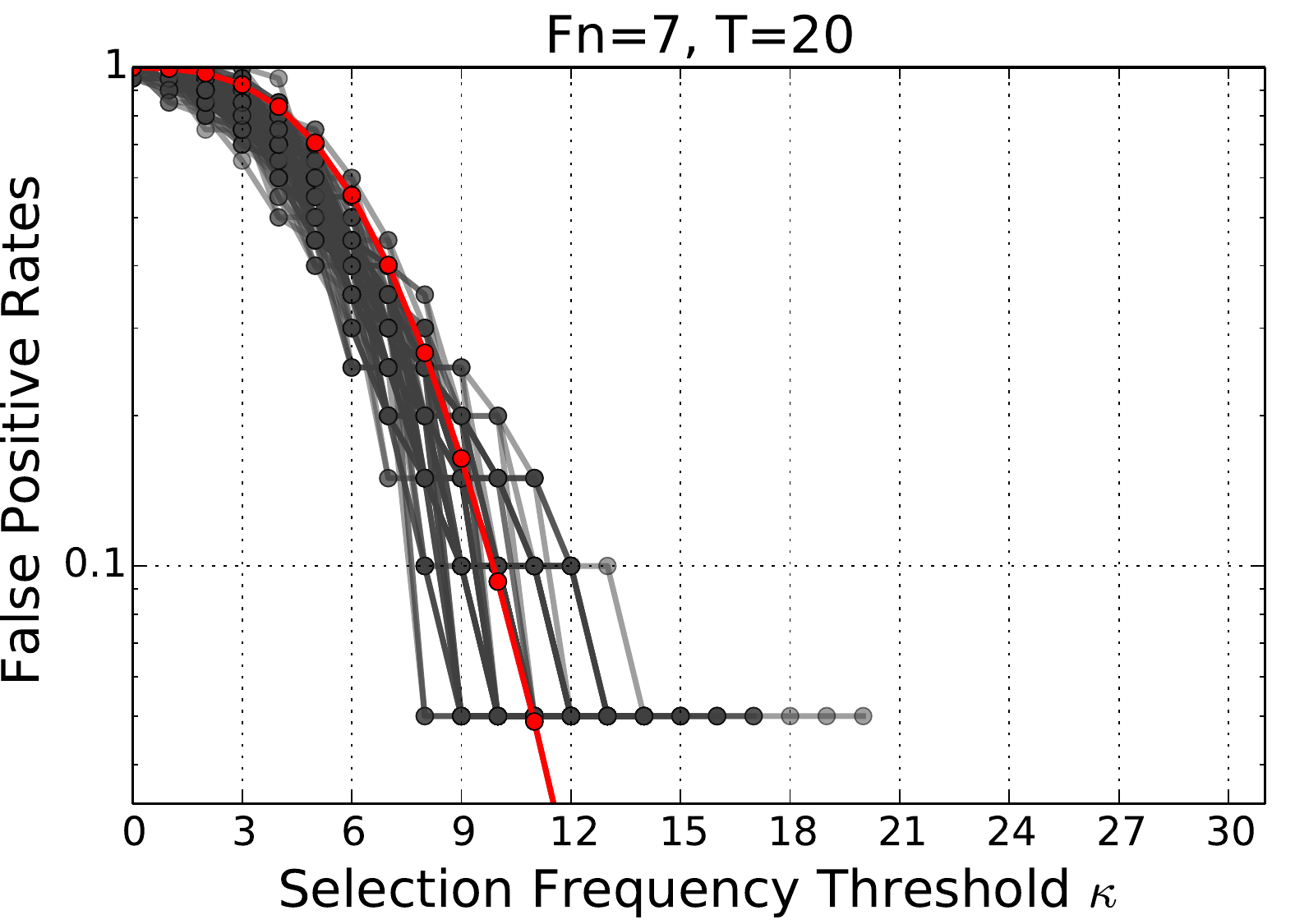} \\
&
\includegraphics[width=0.28\linewidth]{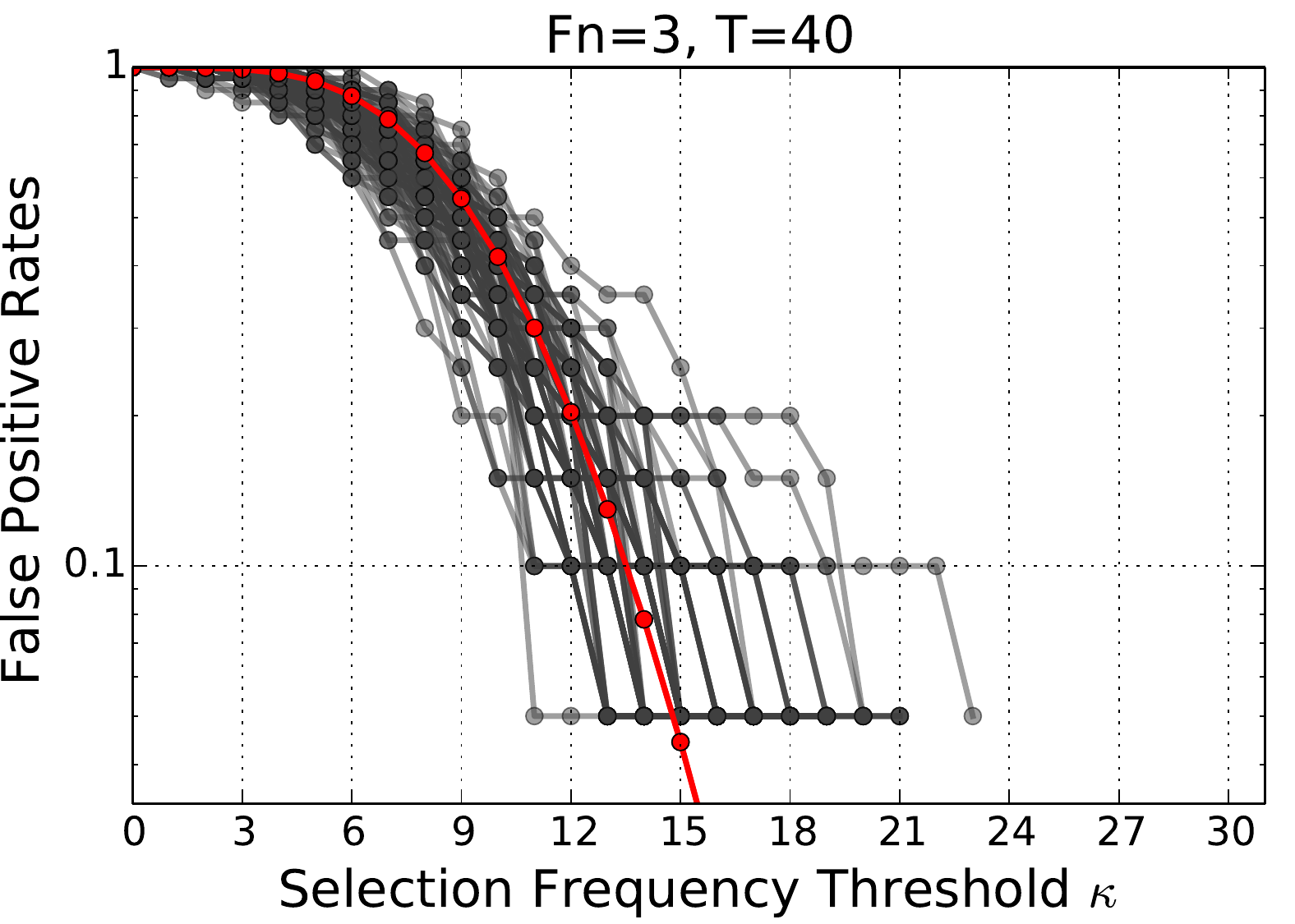} & 
\includegraphics[width=0.28\linewidth]{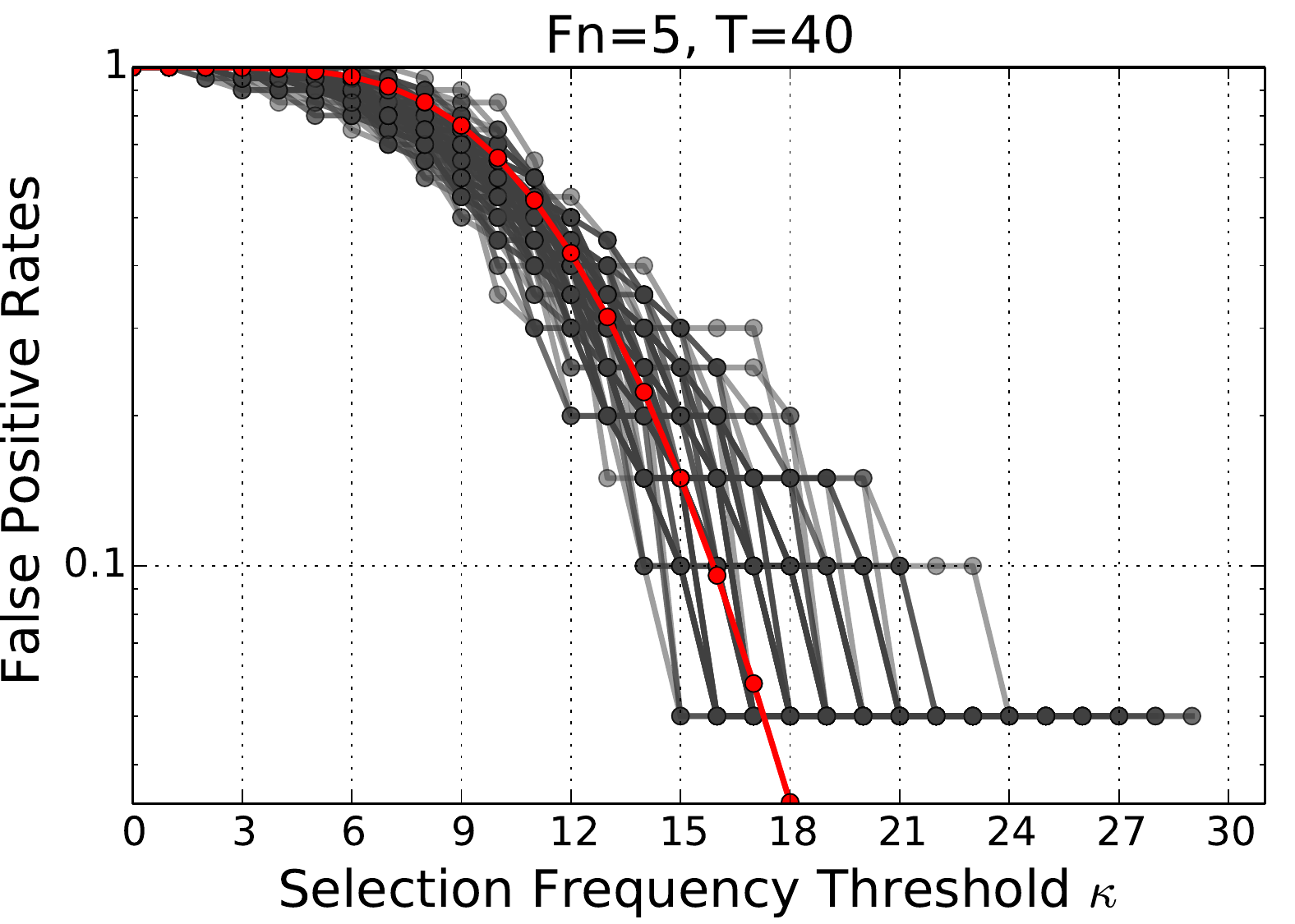} & 
\includegraphics[width=0.28\linewidth]{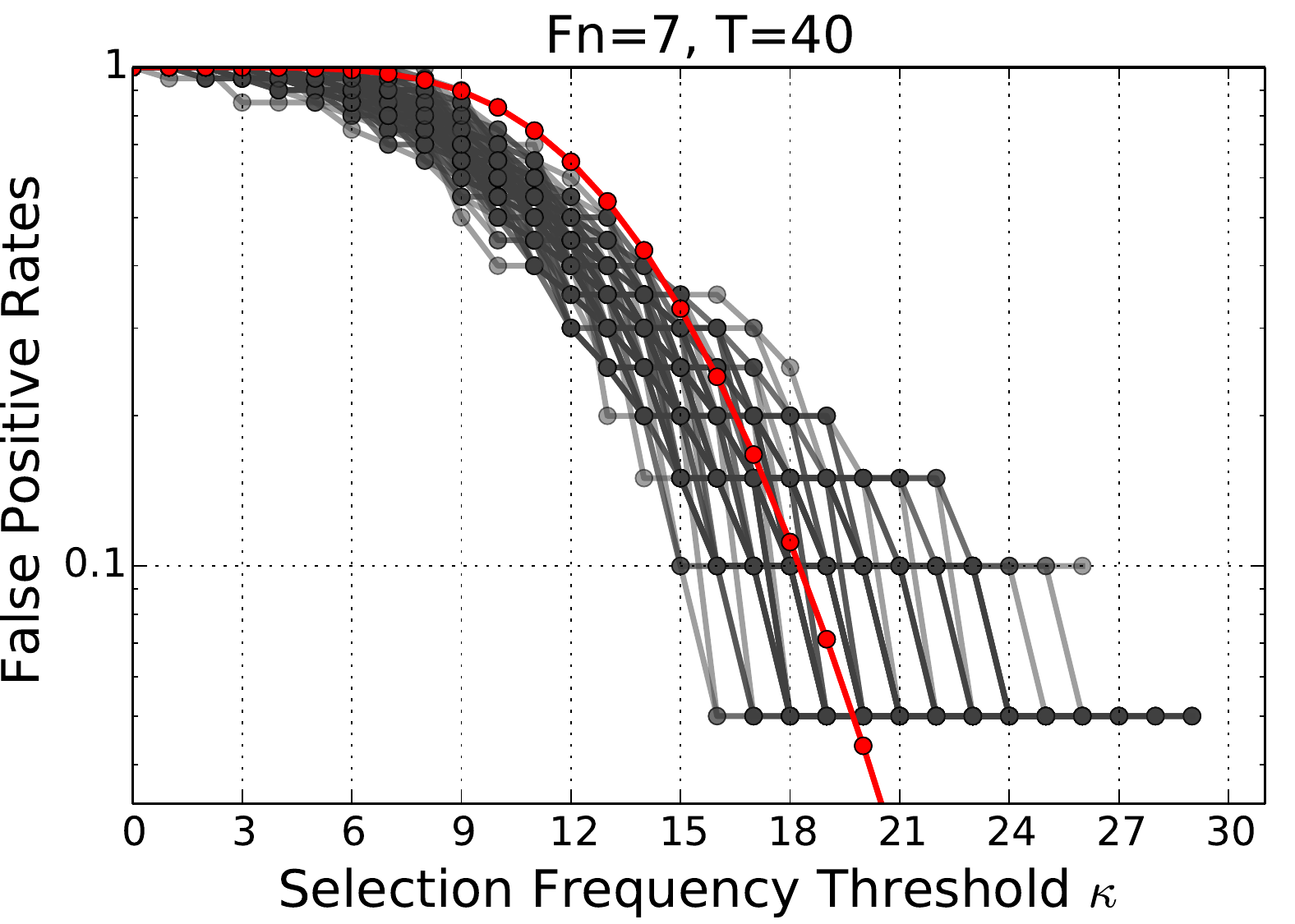}\\
\hline
\multirow{3}{*}{\begin{sideways}Strategy II\end{sideways}} &
\includegraphics[width=0.28\linewidth]{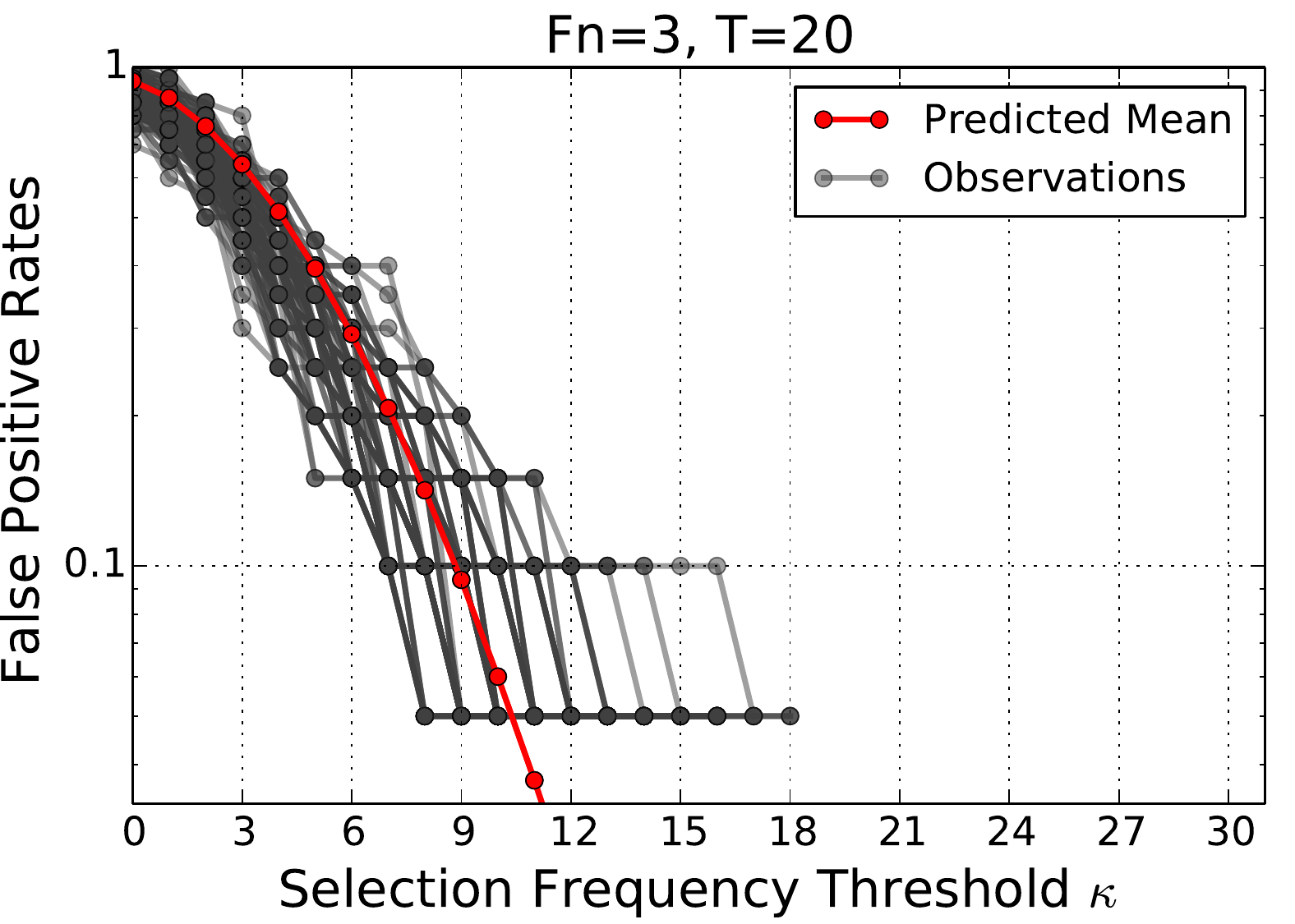} & 
\includegraphics[width=0.28\linewidth]{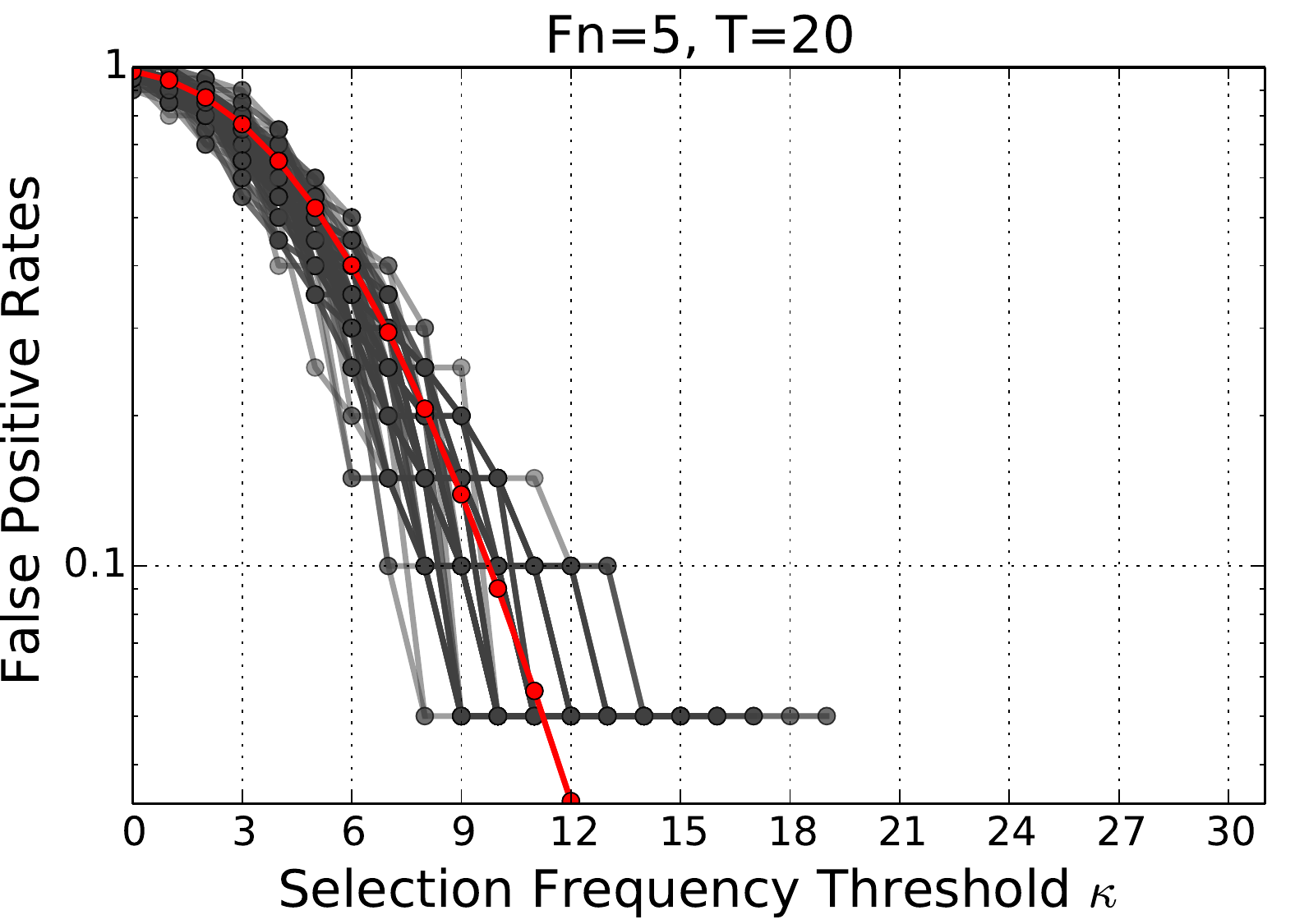} & 
\includegraphics[width=0.28\linewidth]{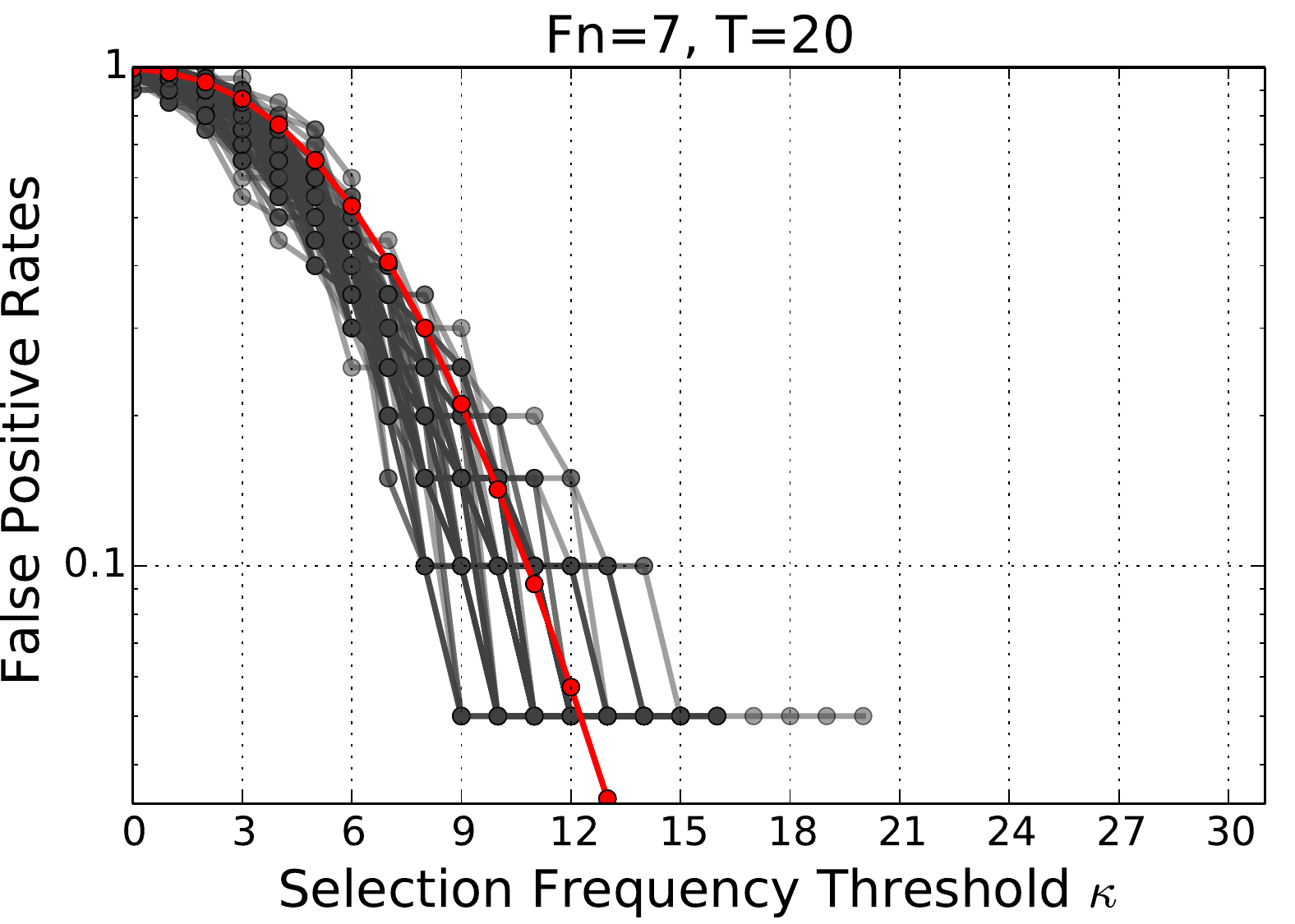} \\
&
\includegraphics[width=0.28\linewidth]{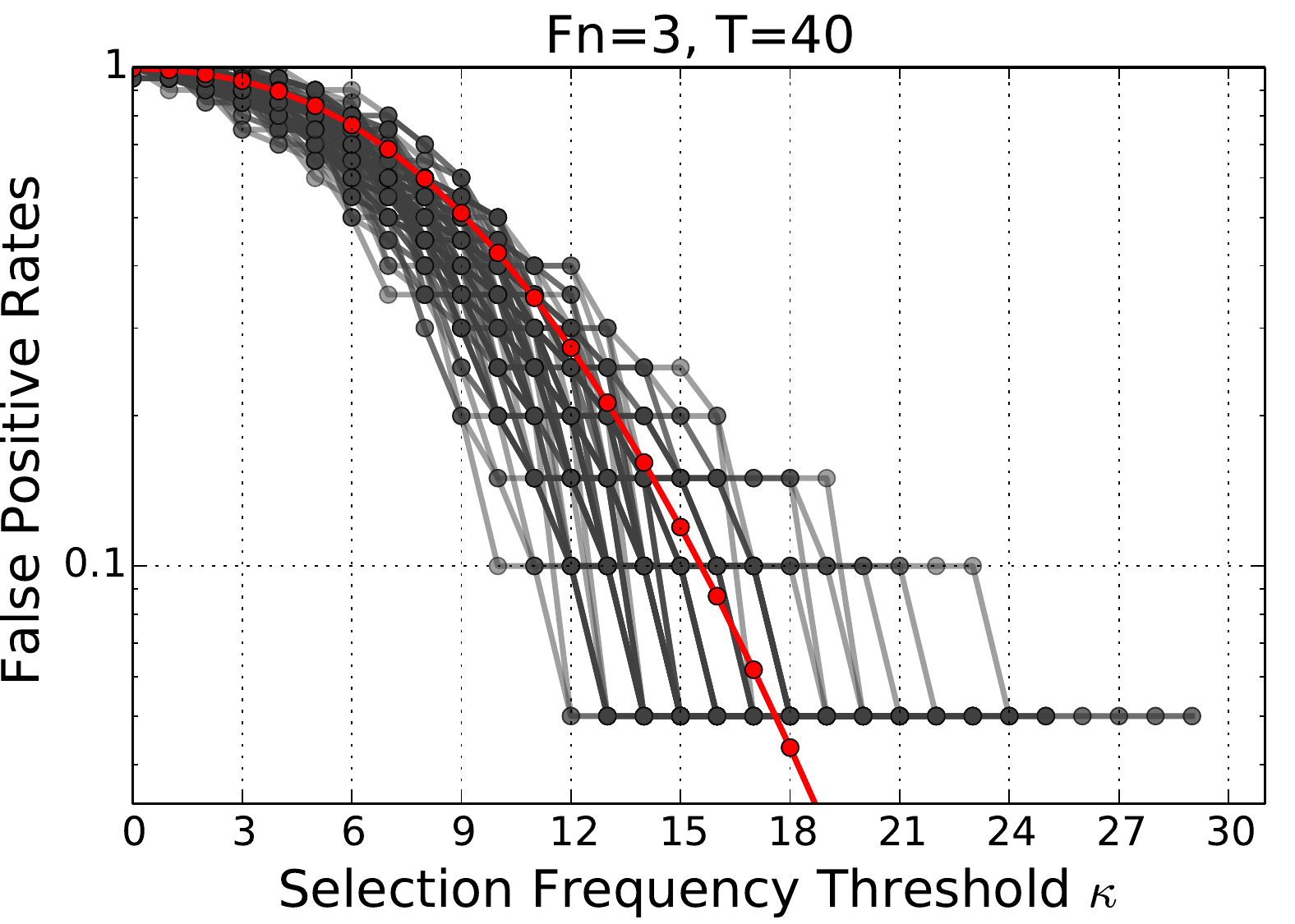} & 
\includegraphics[width=0.28\linewidth]{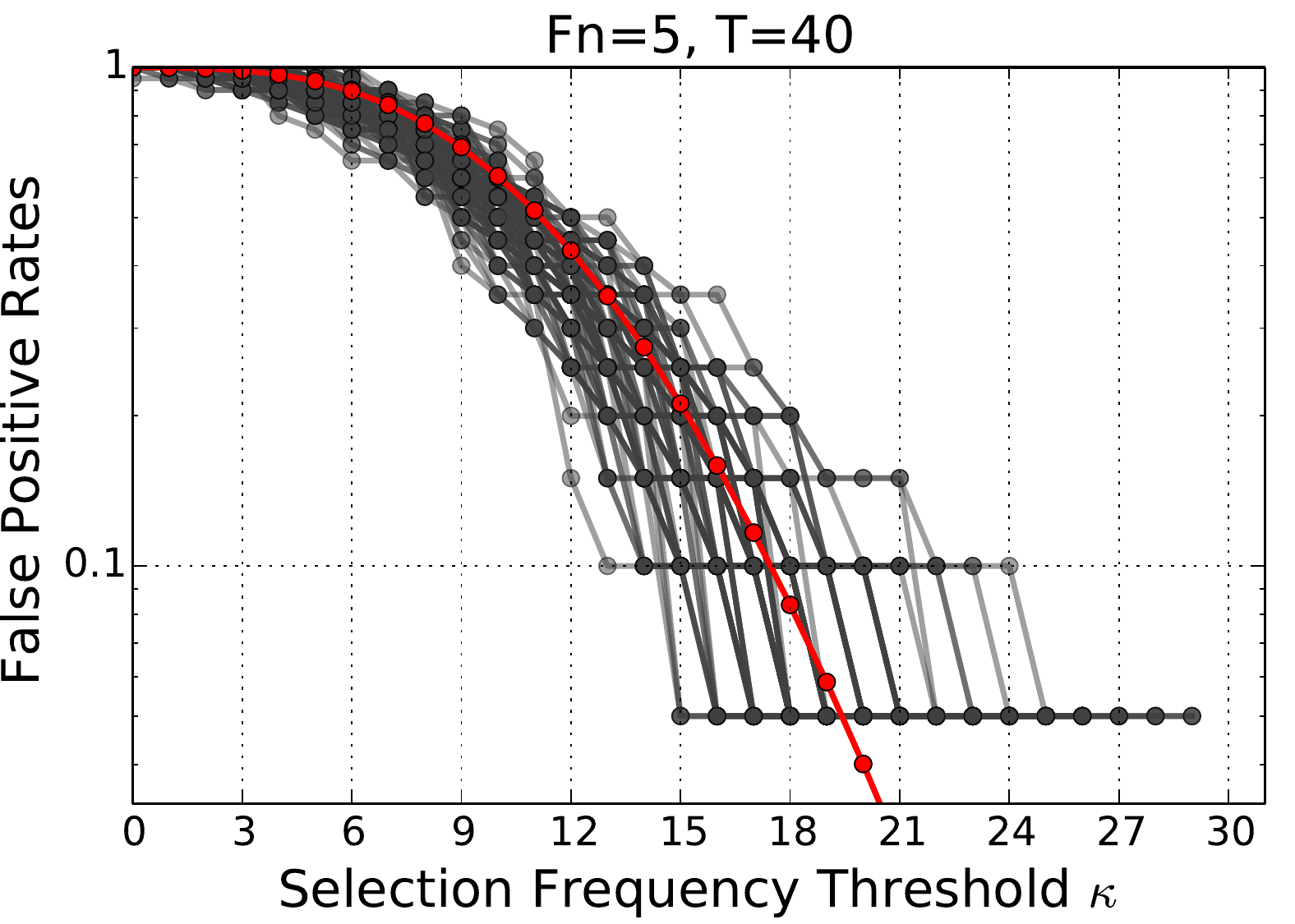} & 
\includegraphics[width=0.28\linewidth]{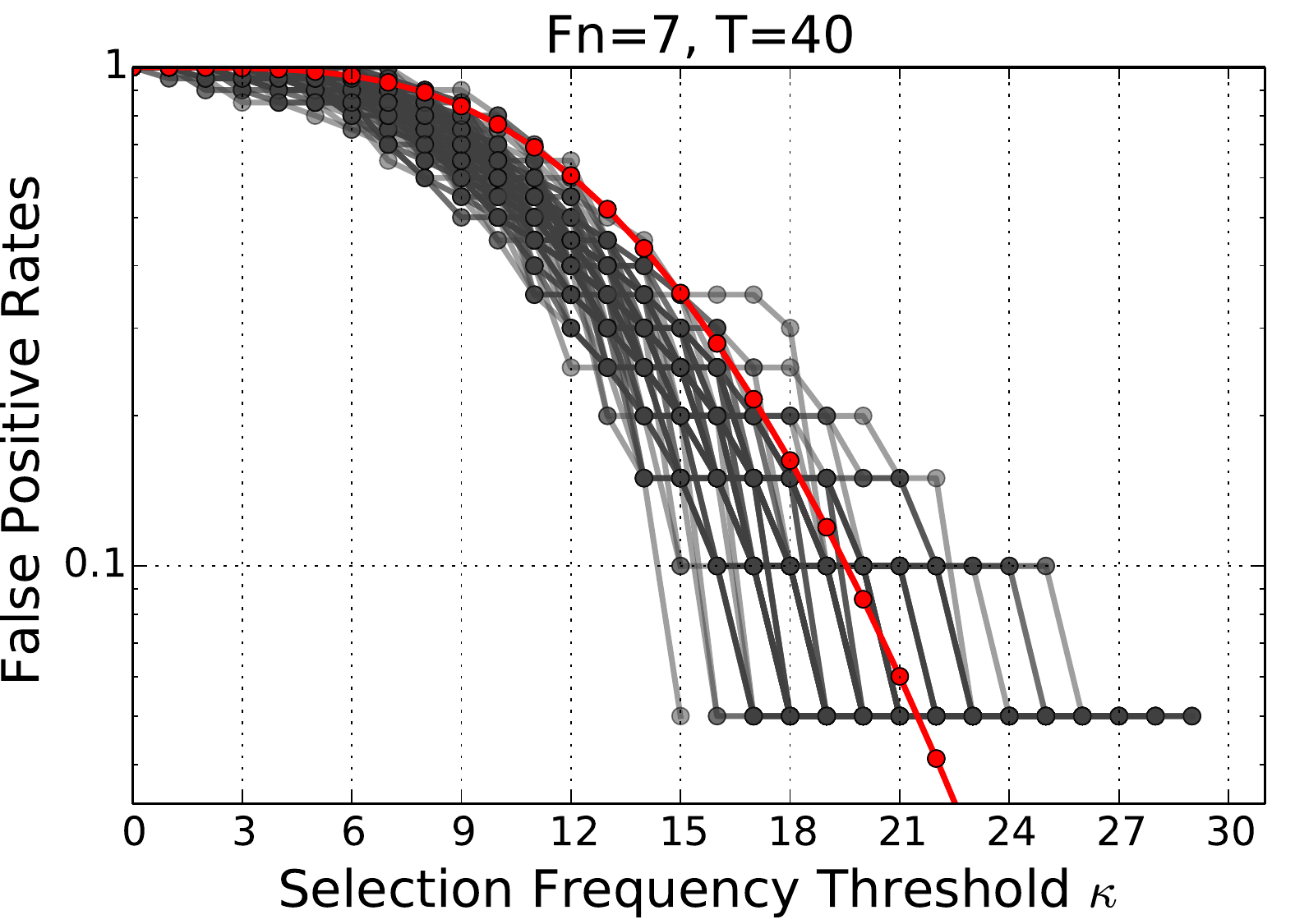}\\
\hline
\end{tabular}
\caption{\label{tab:basic_problem}{\small Low-Dimensional Problem: Comparing observed false positive rates with models' predictions. The graphs plot all the 100 simulations as well as the theoretical prediction. The theoretical predictions capture the mean behavior for various parameter settings.}}
\end{table}

Table~\ref{tab:basic_problem} displays the results for different parameter settings and different training strategies. Each graph plots the false positive rates vs. selection frequency threshold. Observed false positive rates from all experiments are plotted using gray lines and the model prediction is plotted in red. Overall, we observe that the model prediction is accurately capturing the mean behavior. The quality of the approximation does not change with changing $F_n$ or $T$ and seems to be similar for both strategies. The models are able to capture the substantial effect of increasing number of trees $T$ as well as the subtle effects of changing $F_n$. It might be surprising that the model for Strategy I is able to capture the effect of changing $F_n$ since Equation~\ref{eqn:modelI} does not explicitly depend on this parameter. However, the model takes into account the average number of nodes per tree $K$ and this quantity changes with $F_n$, which seems to be the main effect of $F_n$ on the false positive rates. In summary, these experiments demonstrate that the proposed models can accurately capture the behavior of selection frequency under the null hypothesis for low-dimensional problems. In the rest of the article we focus on high-dimensional problems.
\begin{table}[!htb]
\small
\begin{tabular}{|c|ccc|}
\hline
\multirow{2}{*}{\begin{sideways}Strategy 1\end{sideways}} &
\includegraphics[width=0.28\linewidth]{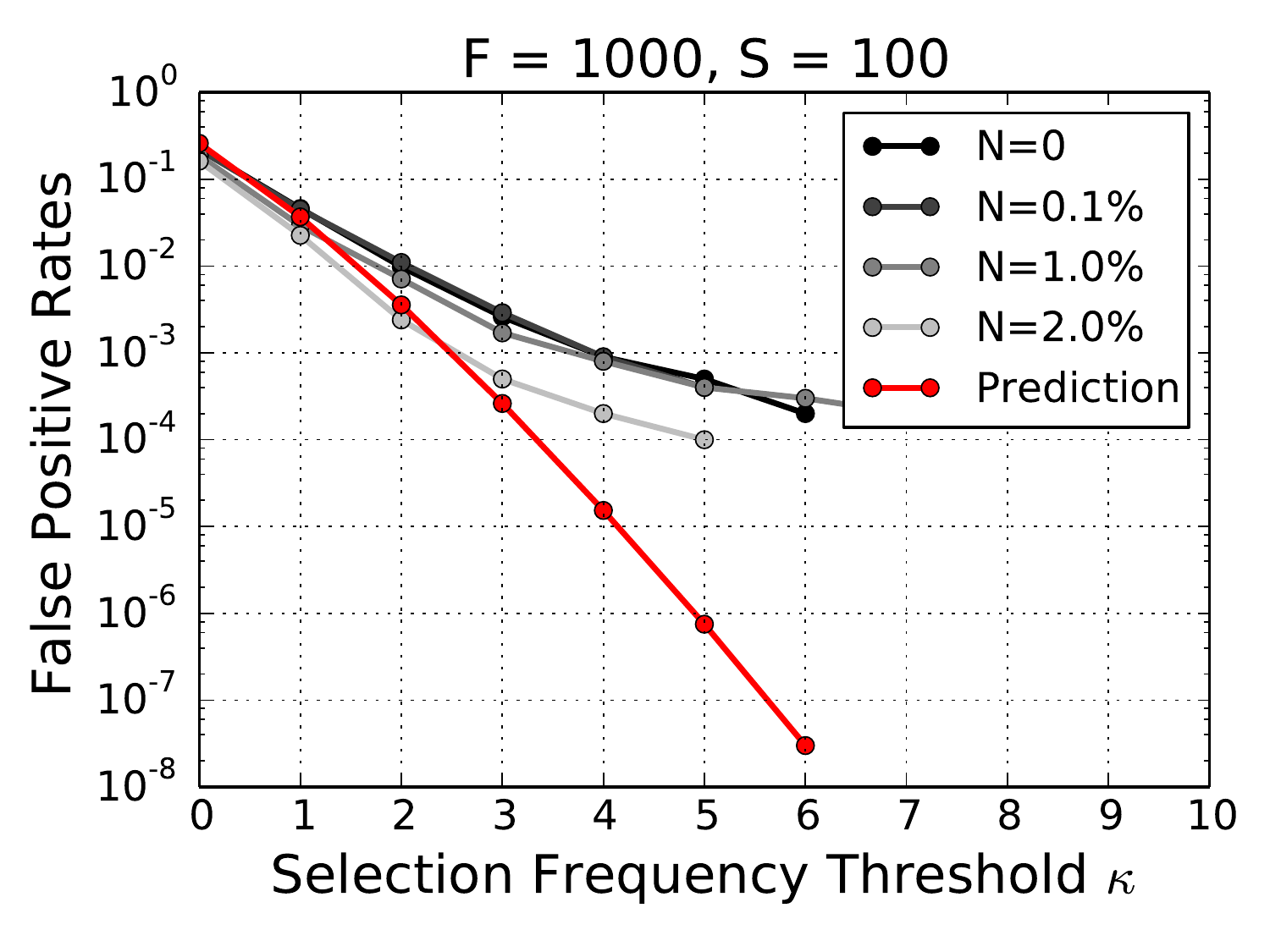} & 
\includegraphics[width=0.28\linewidth]{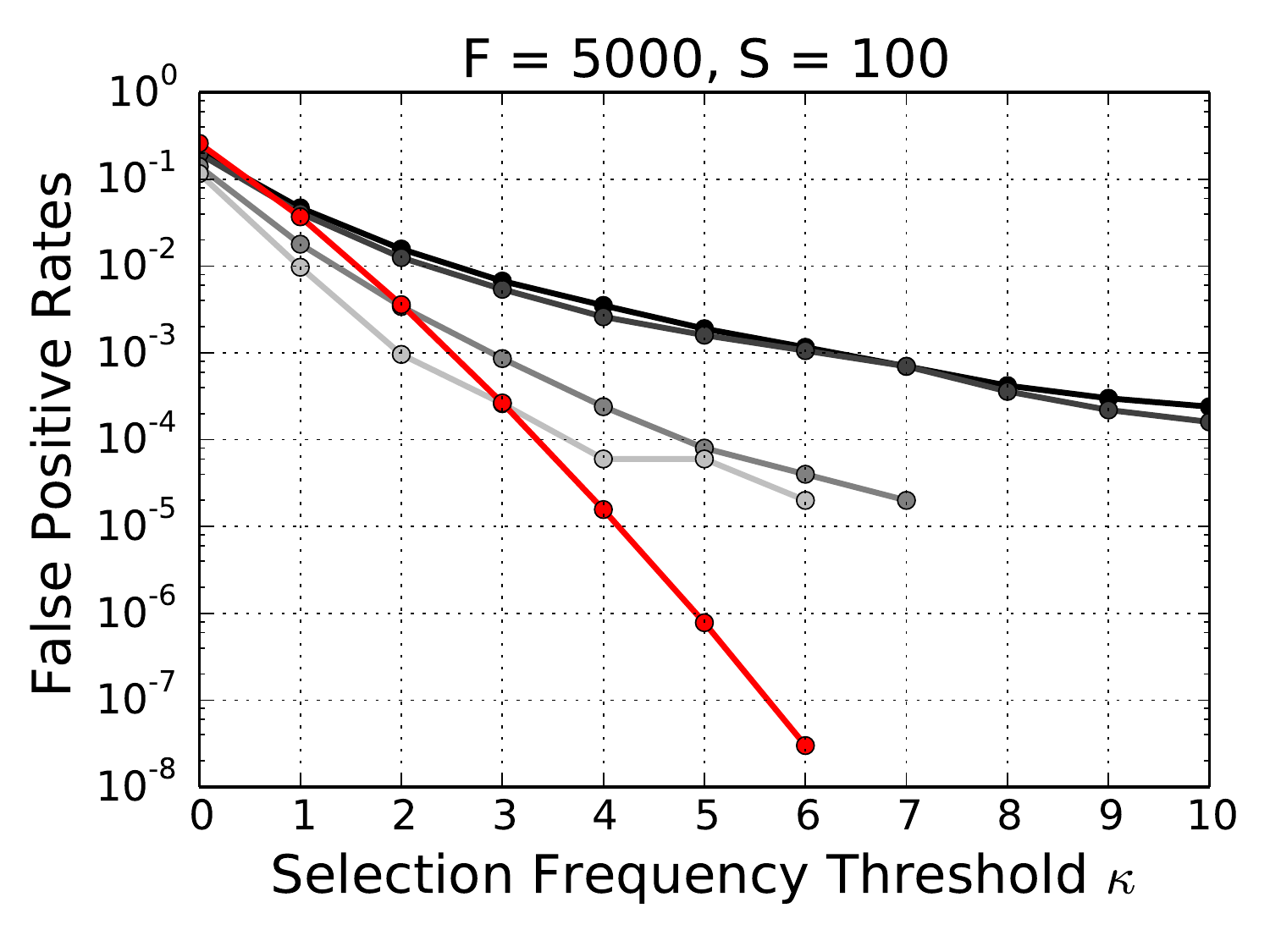} & 
\includegraphics[width=0.28\linewidth]{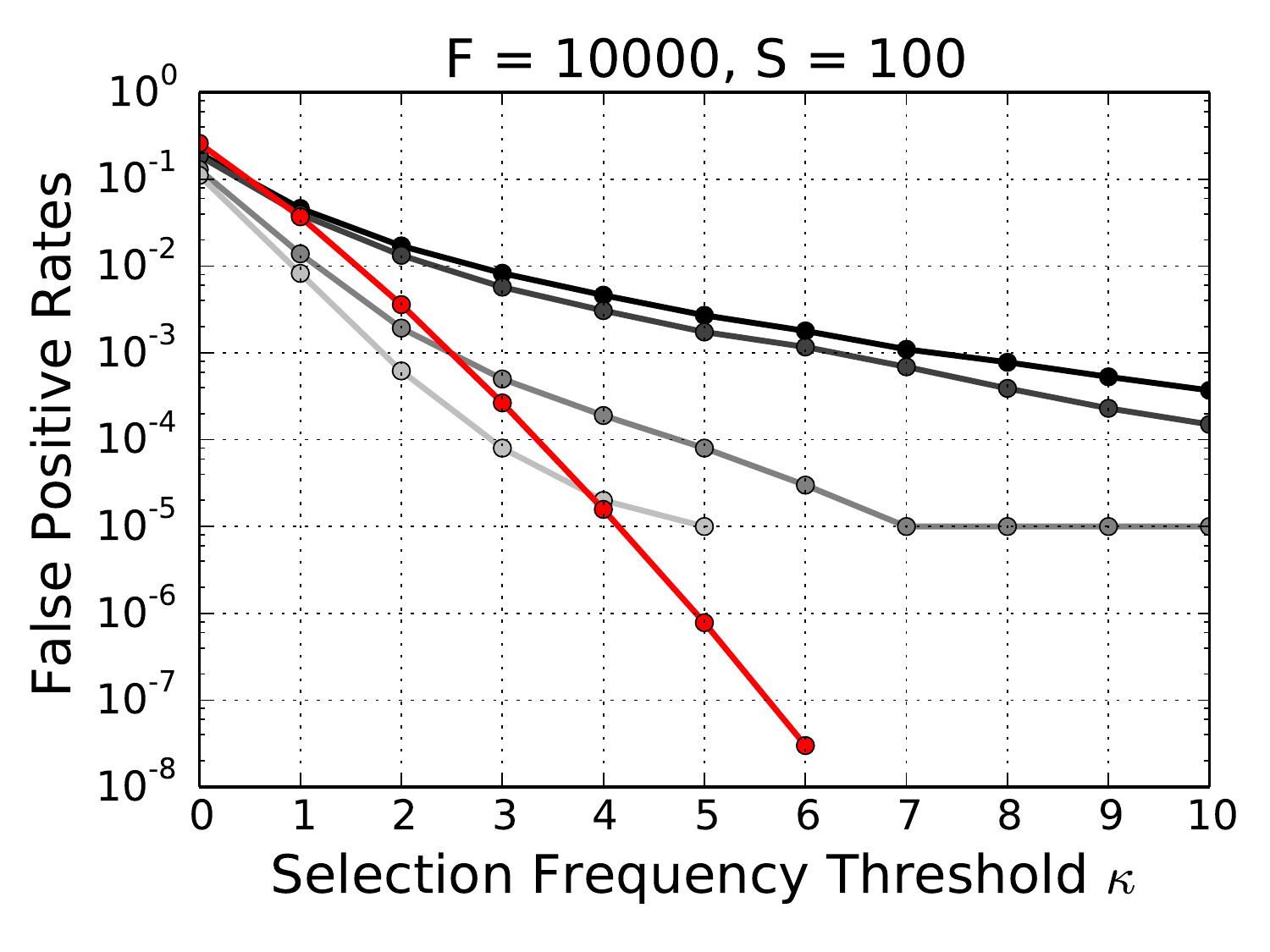} \\
&\includegraphics[width=0.28\linewidth]{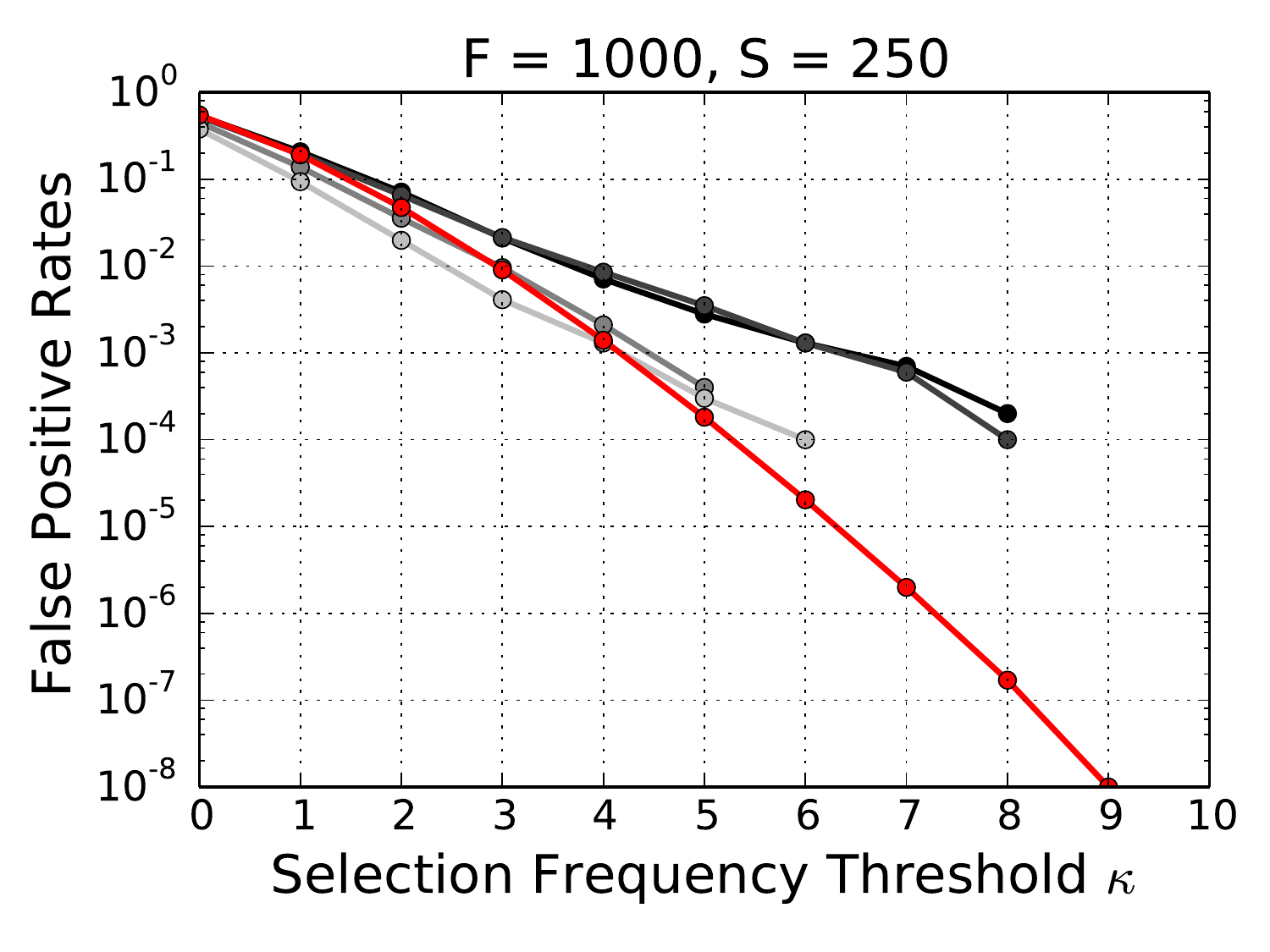} & 
\includegraphics[width=0.28\linewidth]{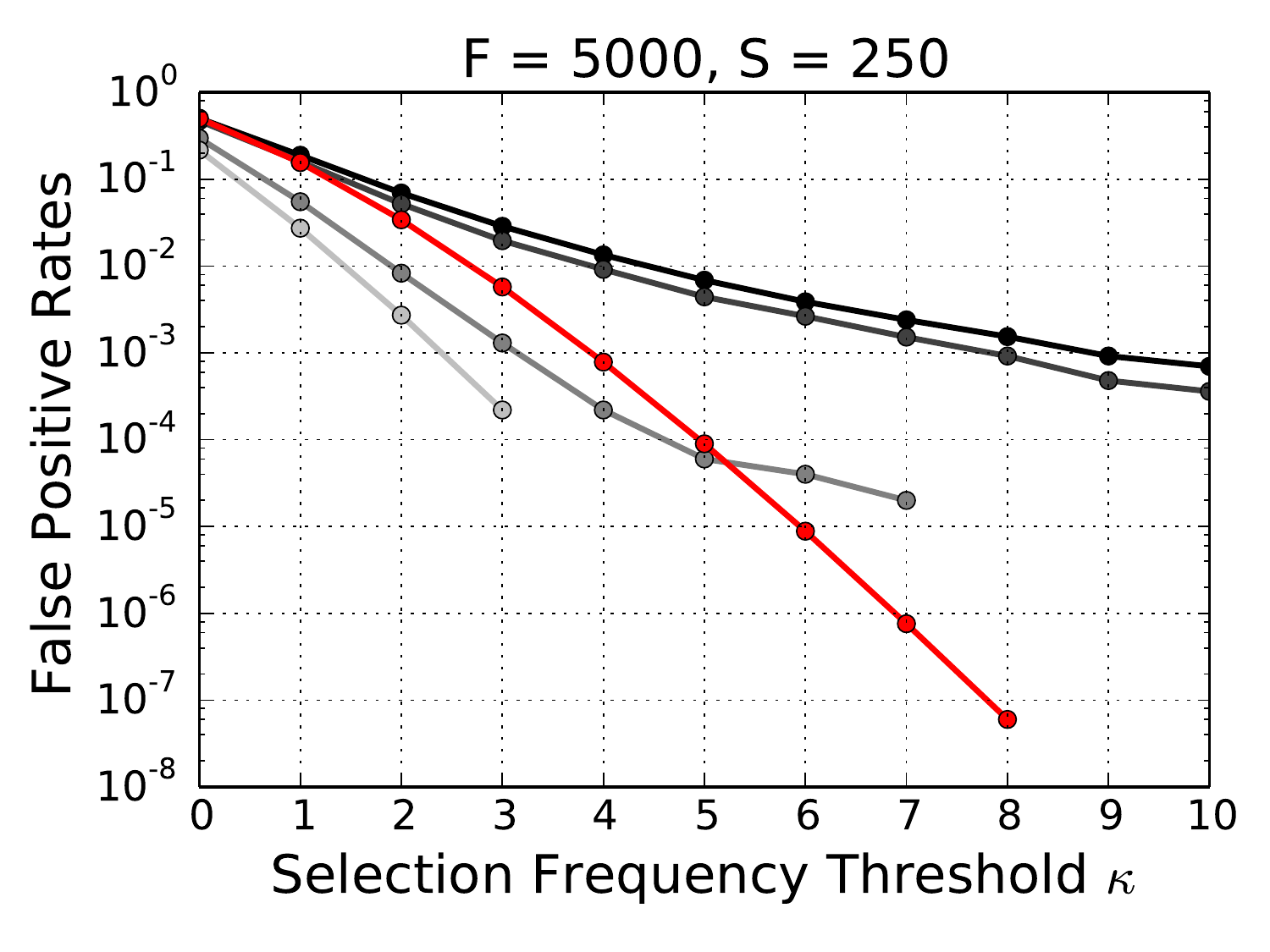} & 
\includegraphics[width=0.28\linewidth]{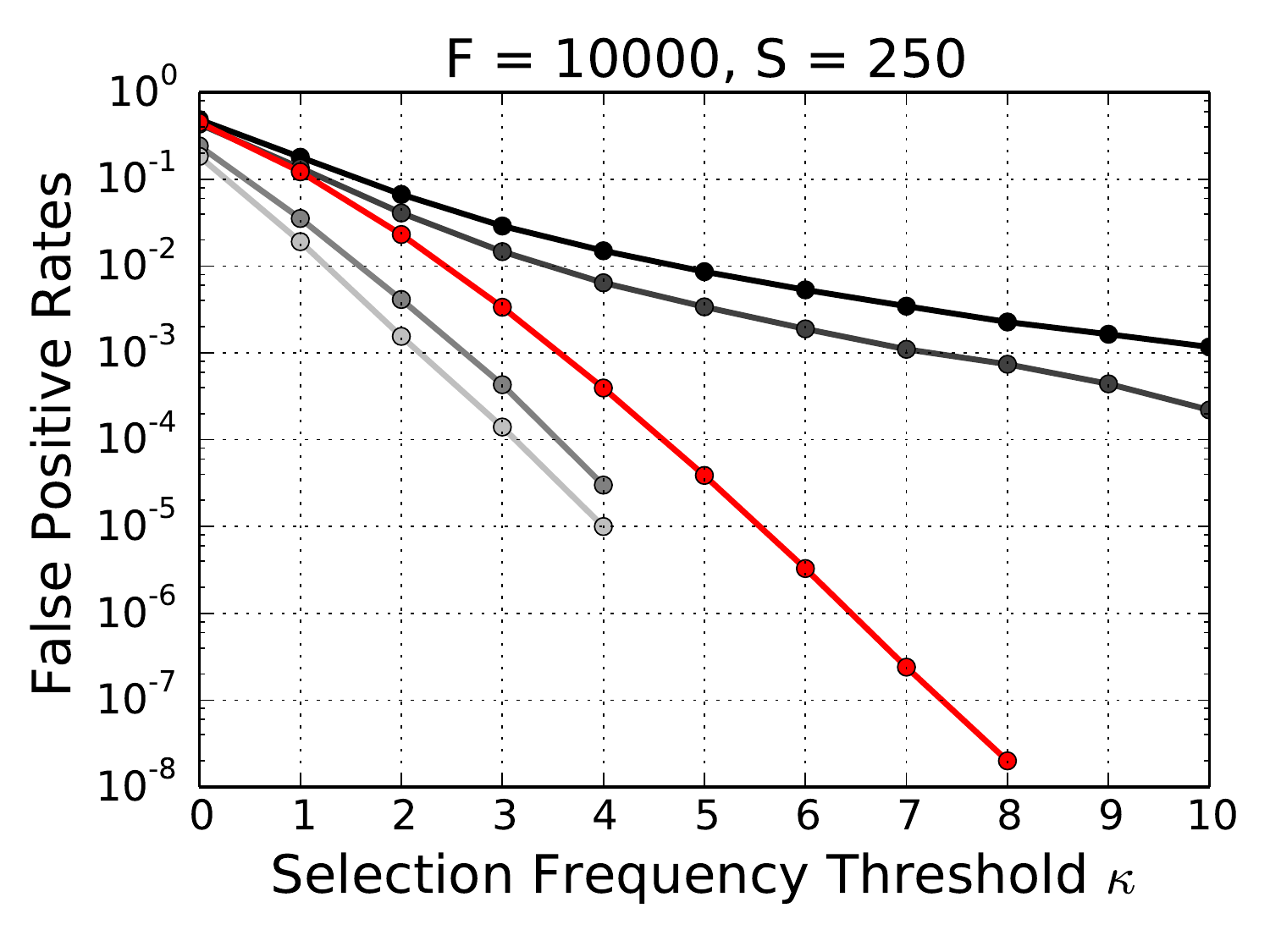}\\
\hline
\multirow{2}{*}{\begin{sideways}Strategy 2\end{sideways}} &
\includegraphics[width=0.28\linewidth]{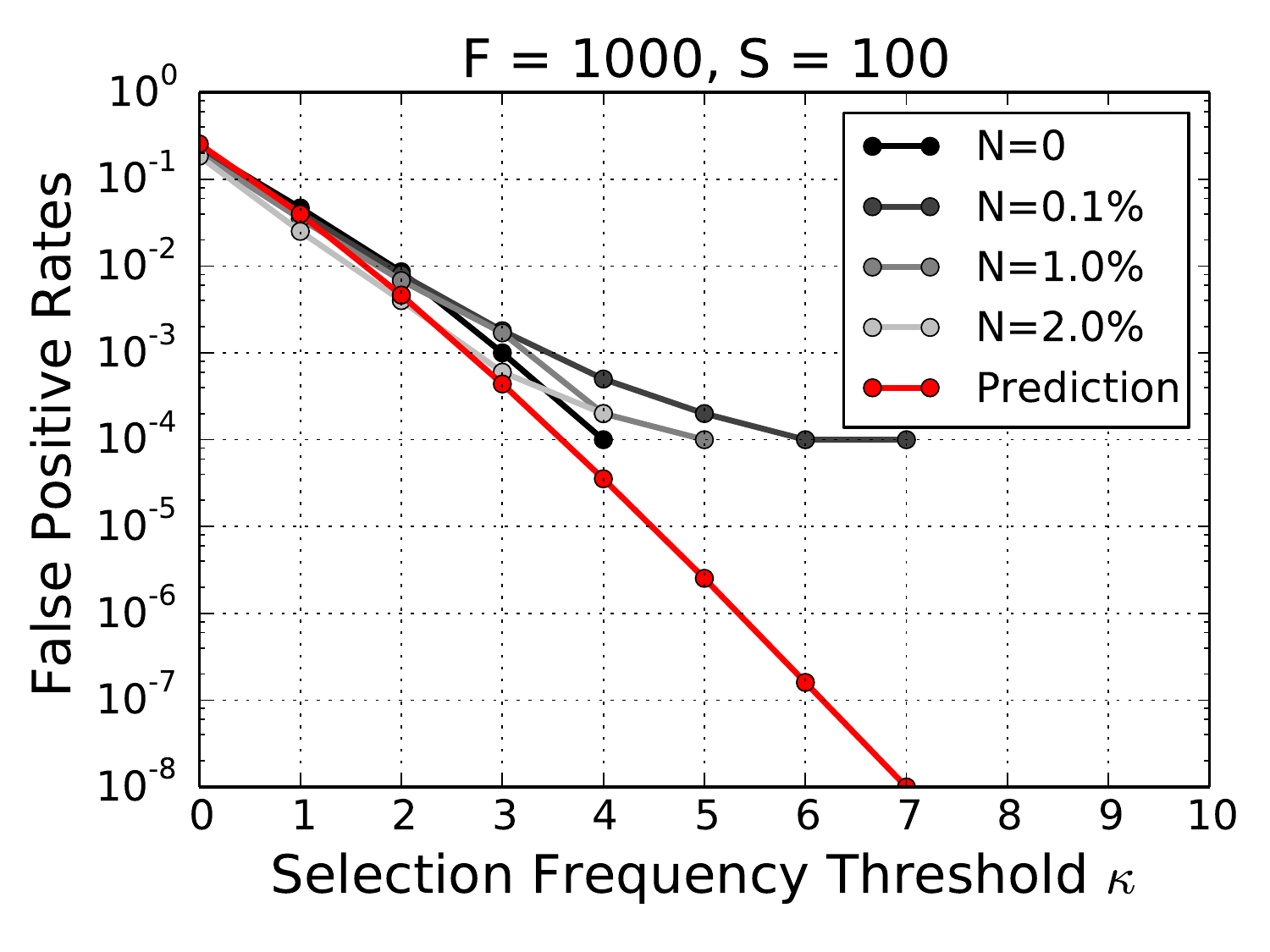} & 
\includegraphics[width=0.28\linewidth]{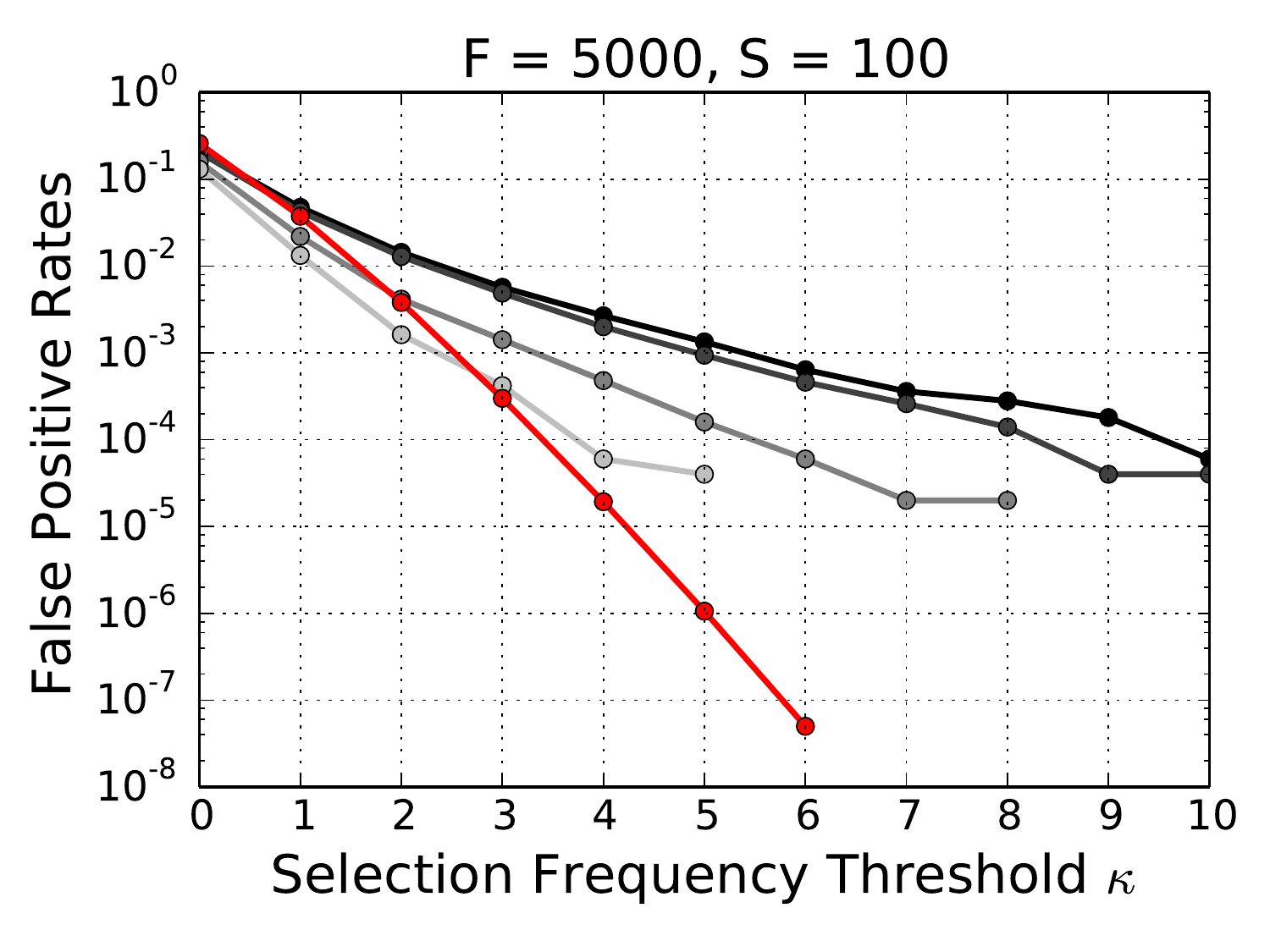} & 
\includegraphics[width=0.28\linewidth]{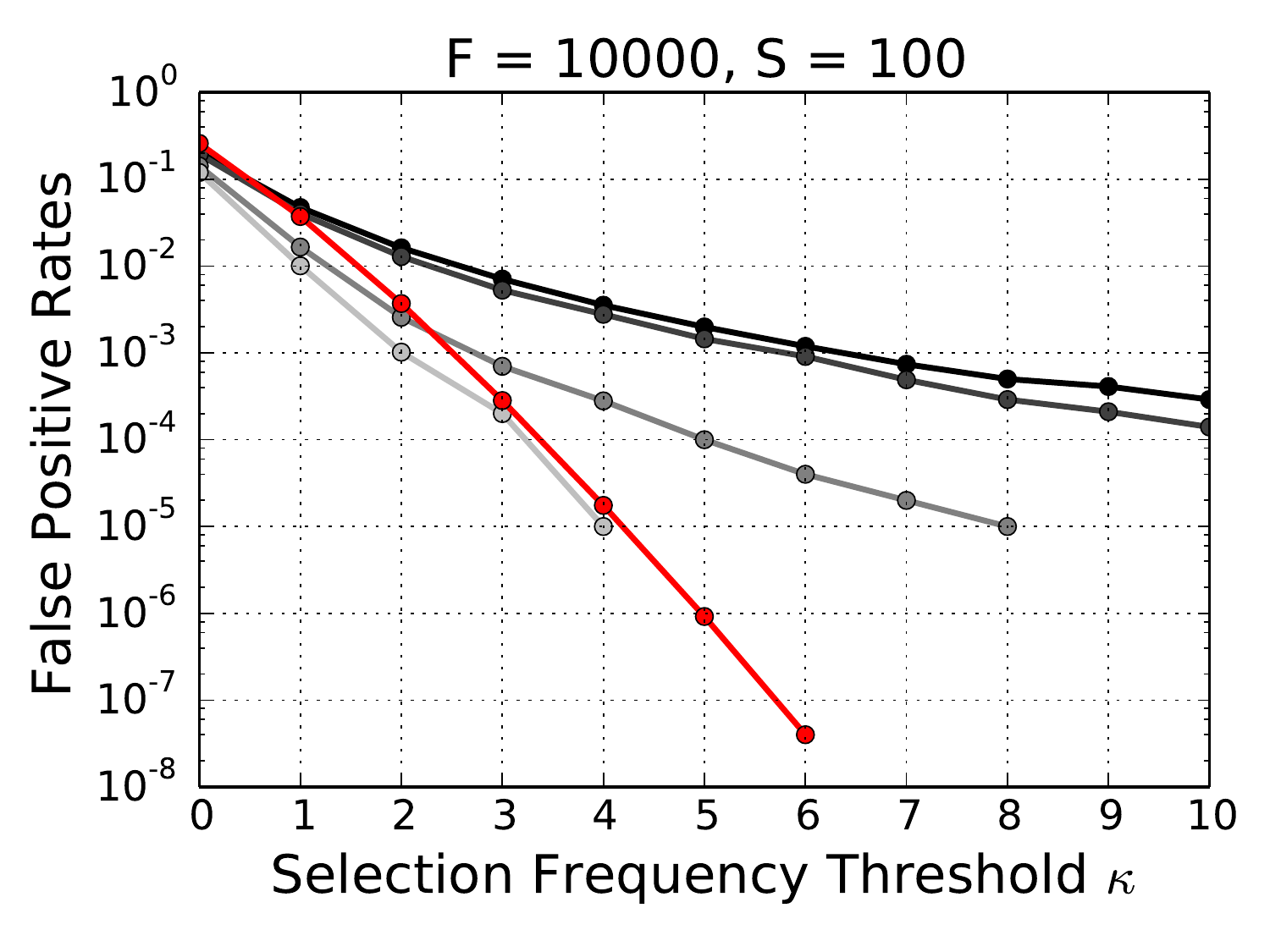} \\
&\includegraphics[width=0.28\linewidth]{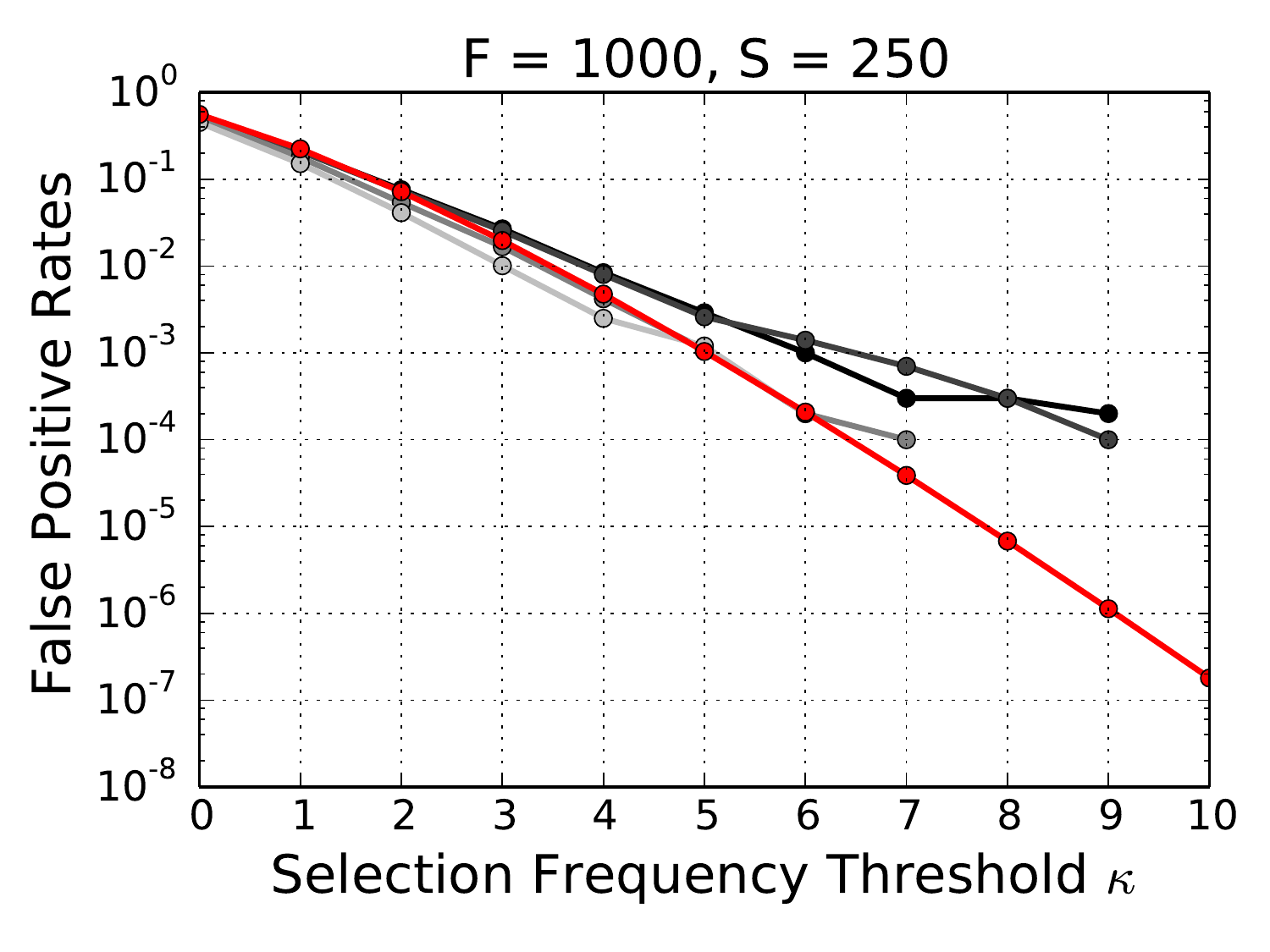} & 
\includegraphics[width=0.28\linewidth]{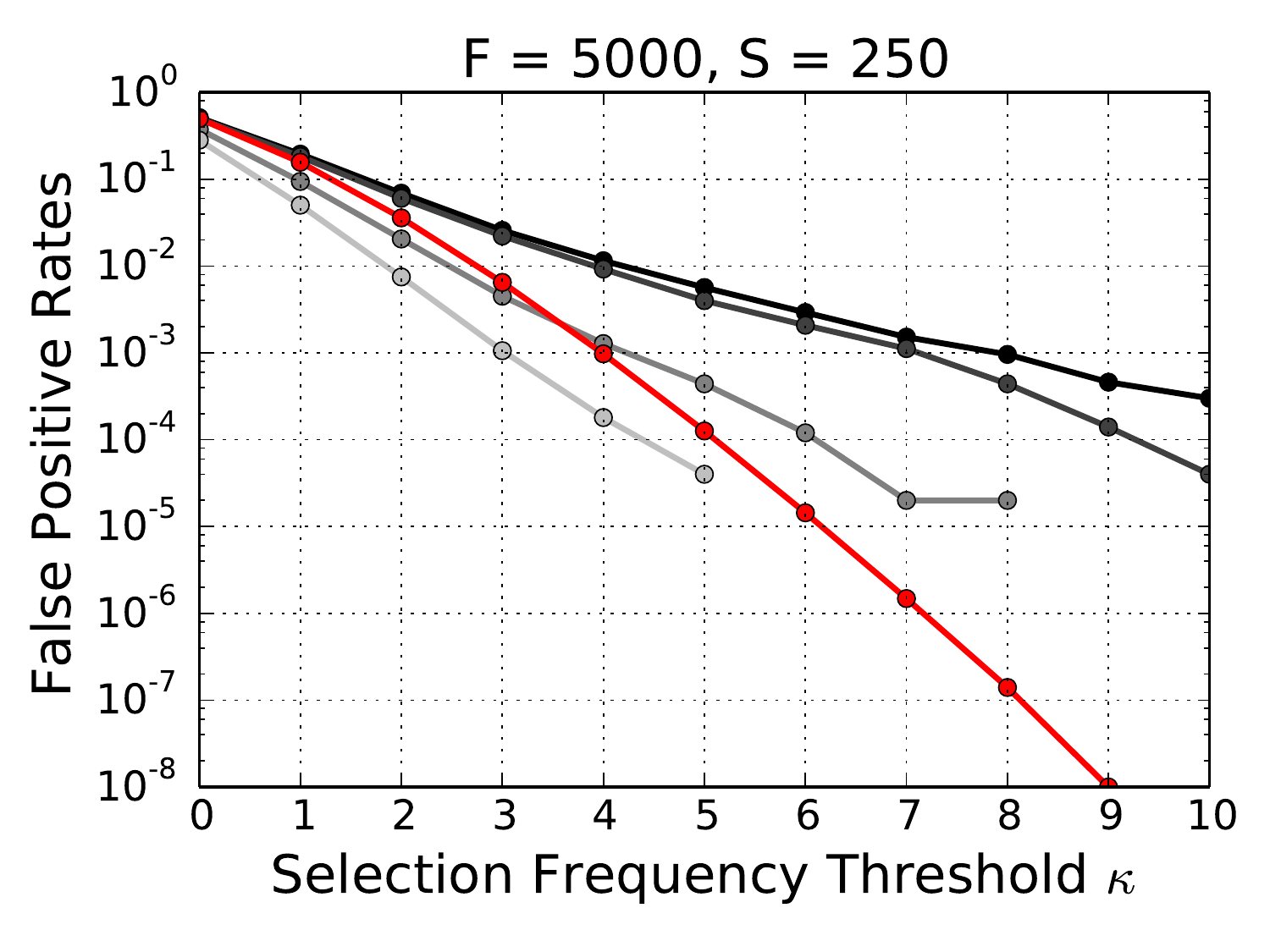} & 
\includegraphics[width=0.28\linewidth]{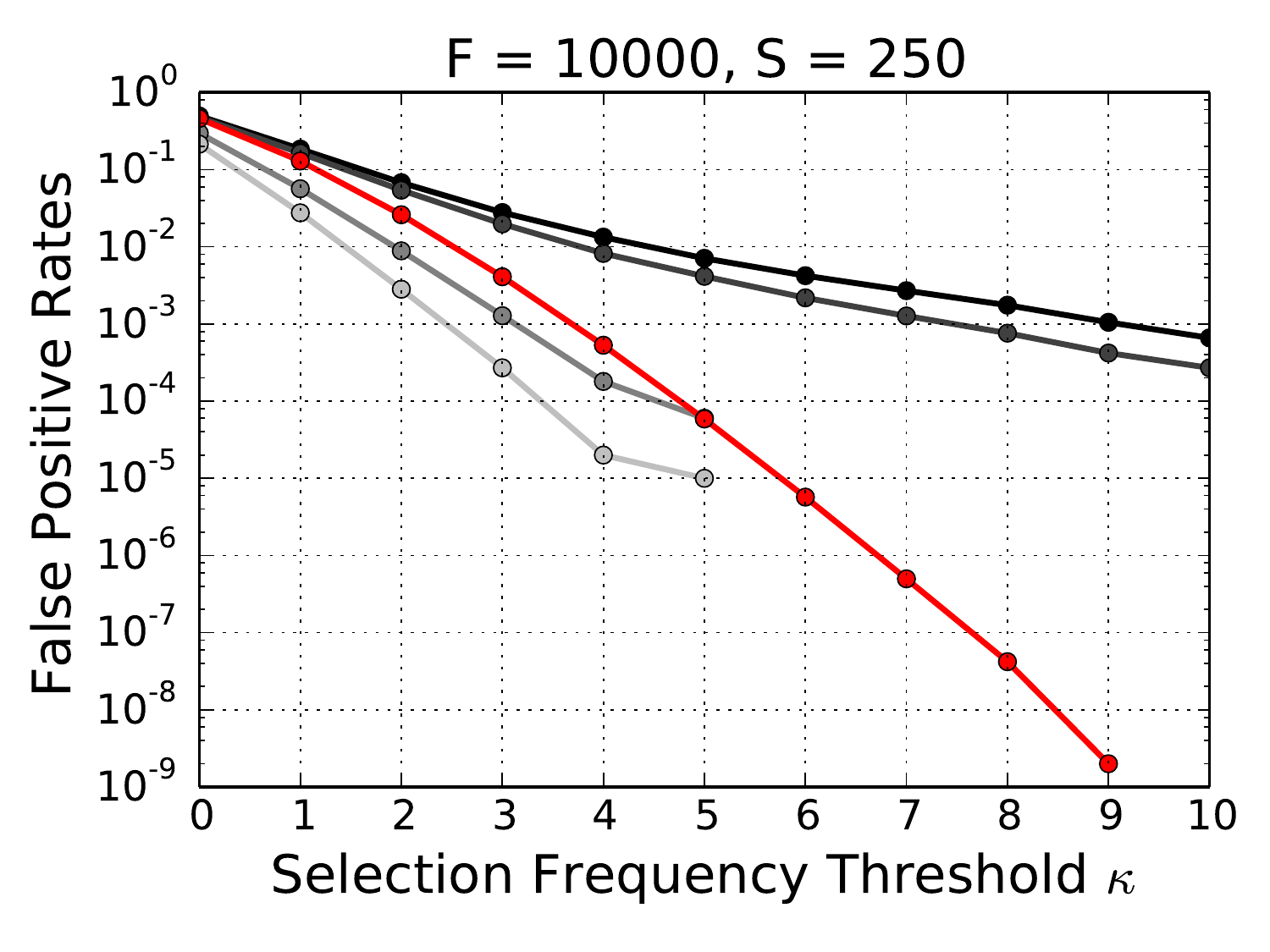}\\
\hline
\end{tabular}
\caption{\label{tab:FPR_predictions}\small High-dimensional problems: Comparing observed false positive rates with models' predictions. Each black curve is the result of averaging 20 simulations. Red curves are the theoretical predictions.}
\end{table}
\subsubsection{High-dimensional Problems: $S << F$}\label{sec:high_dim}
Most of the problems in neuroimaging and biology fall into the category of high-dimensional problems, where there are many more measurements per sample than the number of samples. In this setting, it is more probable to have non-relevant features that manifest strong but spurious statistical relationships with the label. Such relationships naturally threaten the validity of Assumption~\ref{assumption}, since these statistically related non-relevant features will have a higher probability to be selected as optimal. In this section, we present empirical analyses evaluating the proposed method in the high-dimensional case. 
\paragraph{Assessing the theoretical FPR predictions.}
First, we assess the theoretical models by comparing the predicted expected false positive rates with the observed ones for different threshold values. We experiment with different feature dimensions $F$ and sample sizes $S$ keeping the other parameters fixed: $T = F / 10$, $F_n = F / 20$ and the bagging ratio of $1/2$. In addition, we also experiment with having different number of relevant features $N = (0,0.001F,0.01F,0.02F)$, where each relevant feature is correlated to the label with $\rho=0.5$. For each parameter setting we ran 20 experiments, computed the false positive rates for varying thresholds and present the average rate for each threshold. 

Table~\ref{tab:FPR_predictions} presents the experimental results. First of all, as expected we observe that as $F / S$ ratio increases the model predictions diverge from the observed false positive rates for $N=0$. For the same feature dimension increasing sample size makes the model prediction better. Secondly, as predicted presence of relevant features decreases the observed false positive rates. Importantly, for $N=0.01F$ and $N=0.02F$ the models' predictions of false positive rates are fairly close to the observed rates. Even when $F=10k$ and $S=100$ we observe that the predictions and observed rates are very similar for rates higher than $10^{-3}$ for both of the training strategies. These results suggest that the proposed model can indeed be used to determine thresholds that will limit false positive rates. In the next section, we present further analysis on this point. 
\paragraph{Controling False Positive Rates.}
In a real experiment one is interested in determining a threshold that will limit the false positive rate while retaining the relevant features. Here, we present analysis evaluating the proposed method in determining such a threshold. In particular, we simulate the situation where the user sets a desired false positive rate limit $\alpha$ and determines the corresponding threshold $\kappa$ using Equation~\ref{eqn:thresholdII}. We then compute the real false positive rates for this $\kappa$ and compare with the desired $\alpha$. We also compute the false negative rates to evaluate the statistical power of the threshold. We use the same simulations as in the previous part and experiment with $\alpha=0.01,0.001$.

Tables~\ref{tab:StrategyI} and~\ref{tab:StrategyII} summarize the results for training strategy I and II respectively. In graphs on the left column, each curve plots the observed false positive rates for different feature dimensions for a specific sample size averaged over 20 simulations. Different curves correspond to different sample sizes and the color coding is displayed in the legends. In graphs on the right column, the graphs plot the false negative rates, i.e. relevant features that fall below the threshold and gets marked as non-relevant, in the same fashion as false positive rate graphs. In the titles the user set $\alpha$ values and the amount of relevant features, given as a percentage of the entire feature set, are given. For $N=0$ case there is no false negative rate plot. \\

\begin{table}
\small
\centering
\begin{tabular}{|c|c|c|}
\hline
& False Positive Rates & False Negative Rates\\
\hline
\begin{sideways}$N=0$\end{sideways} &
\includegraphics[width=0.37\linewidth]{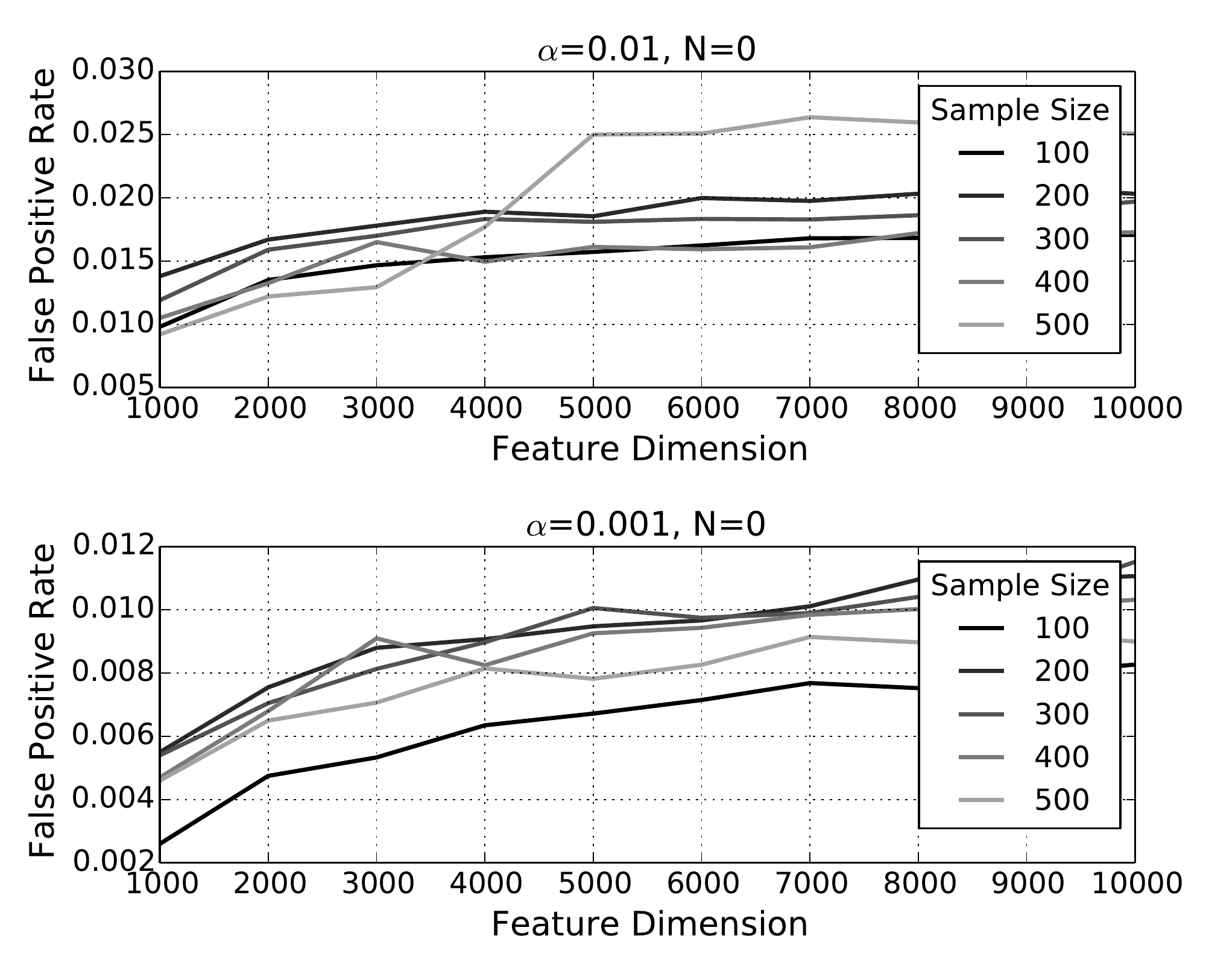}& N/A \\
\hline
\begin{sideways}$N=0.1\%$\end{sideways} &
\includegraphics[width=0.37\linewidth]{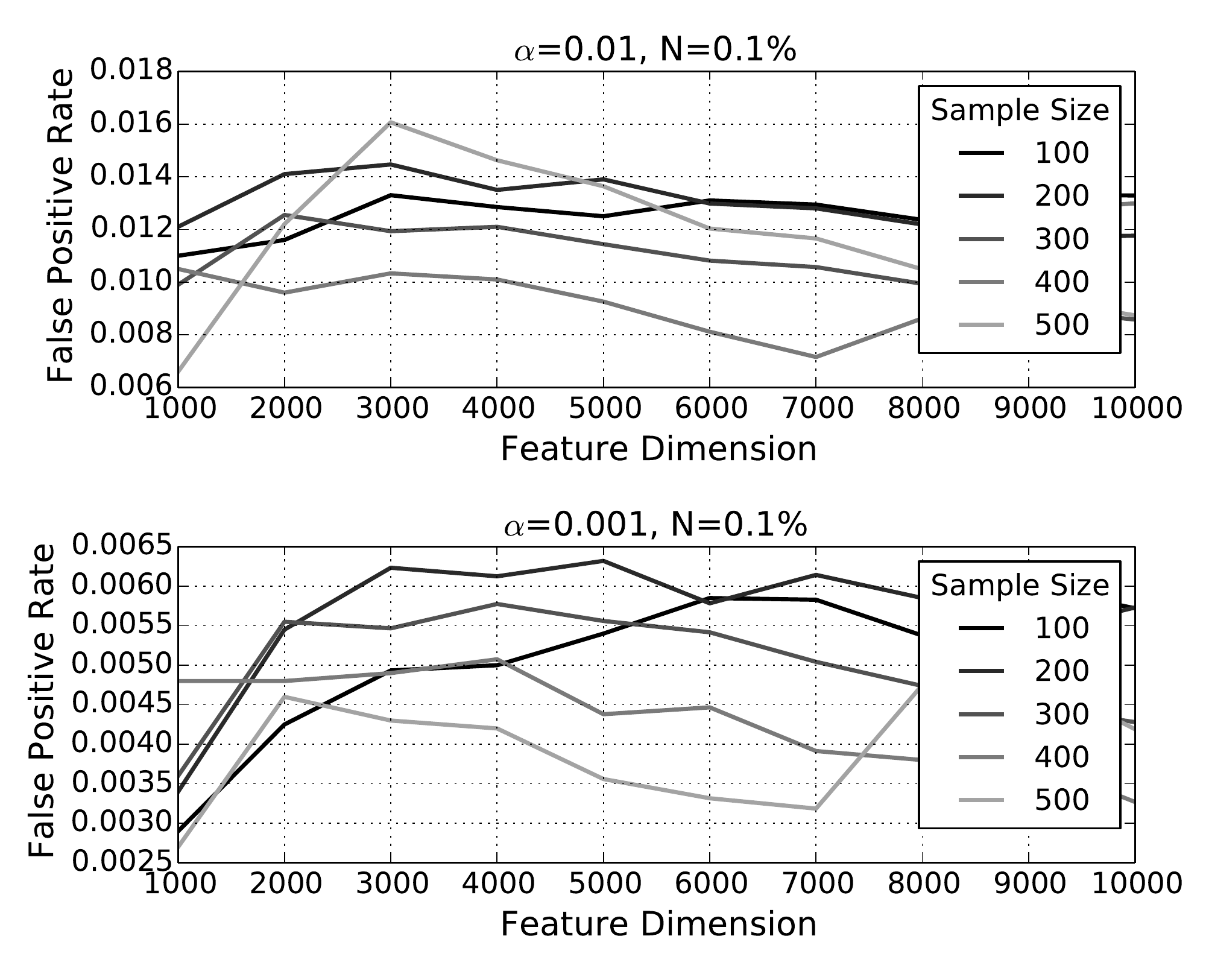}&
\includegraphics[width=0.37\linewidth]{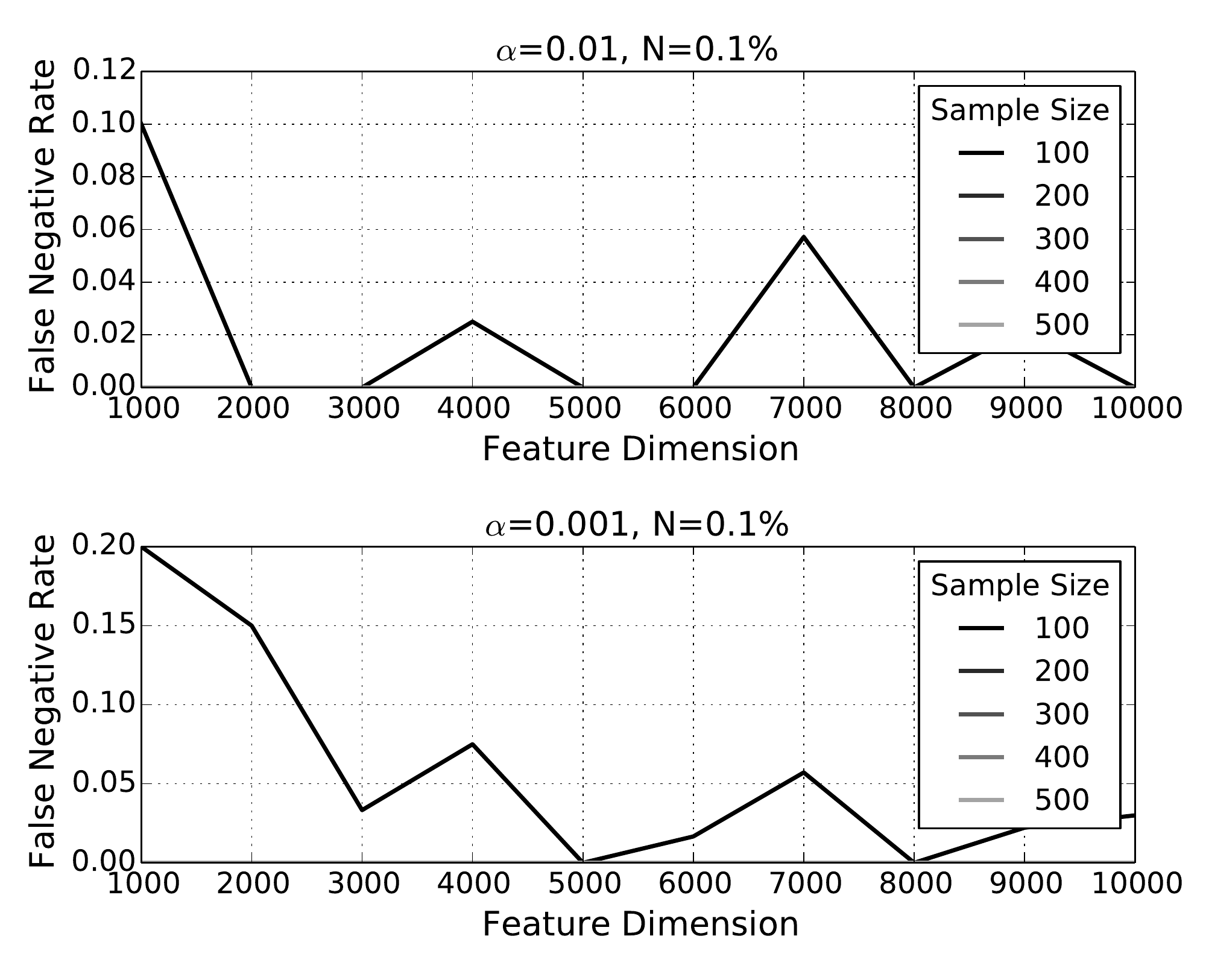}\\
\hline
\begin{sideways}$N=1\%$\end{sideways} &
\includegraphics[width=0.37\linewidth]{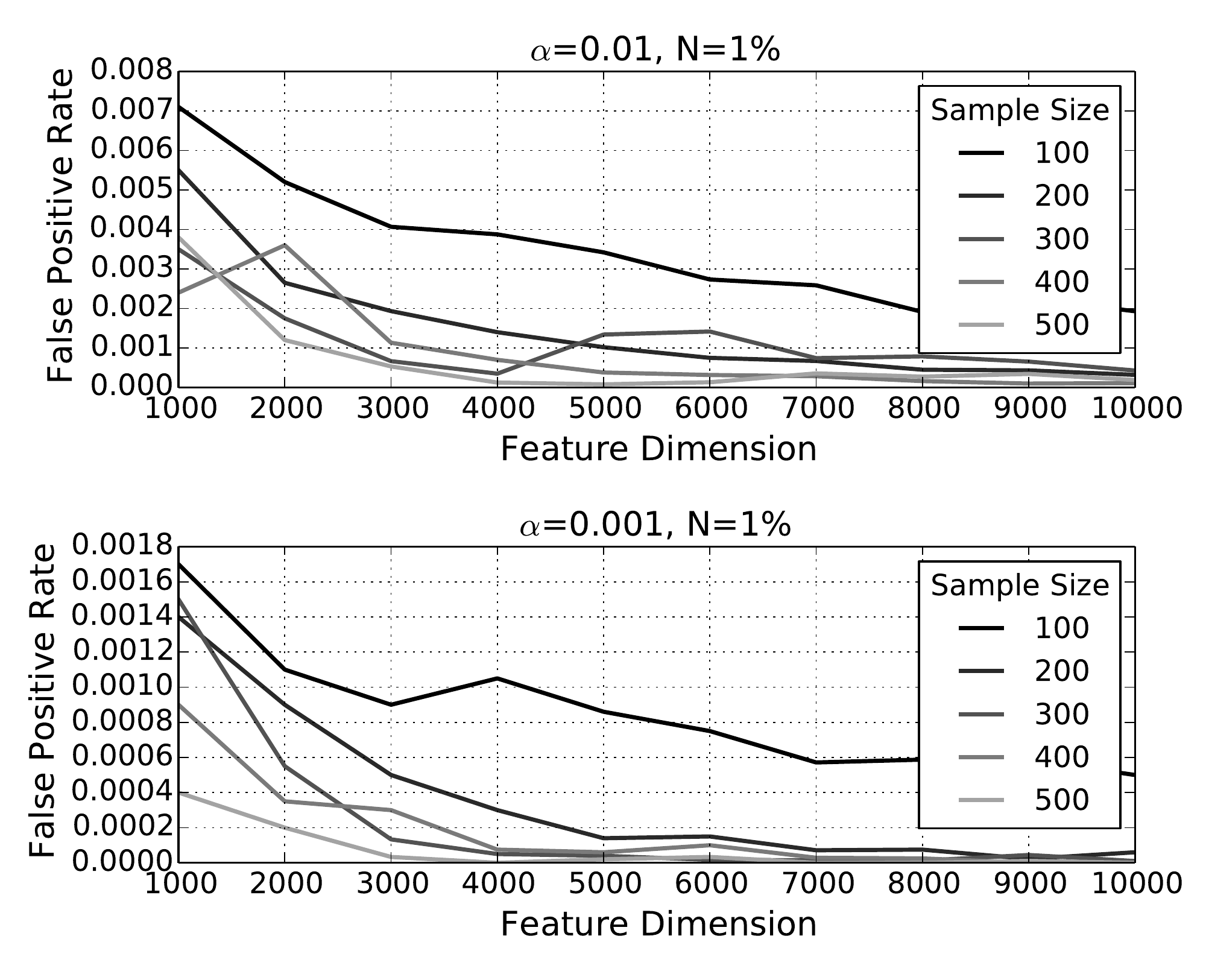}&
\includegraphics[width=0.37\linewidth]{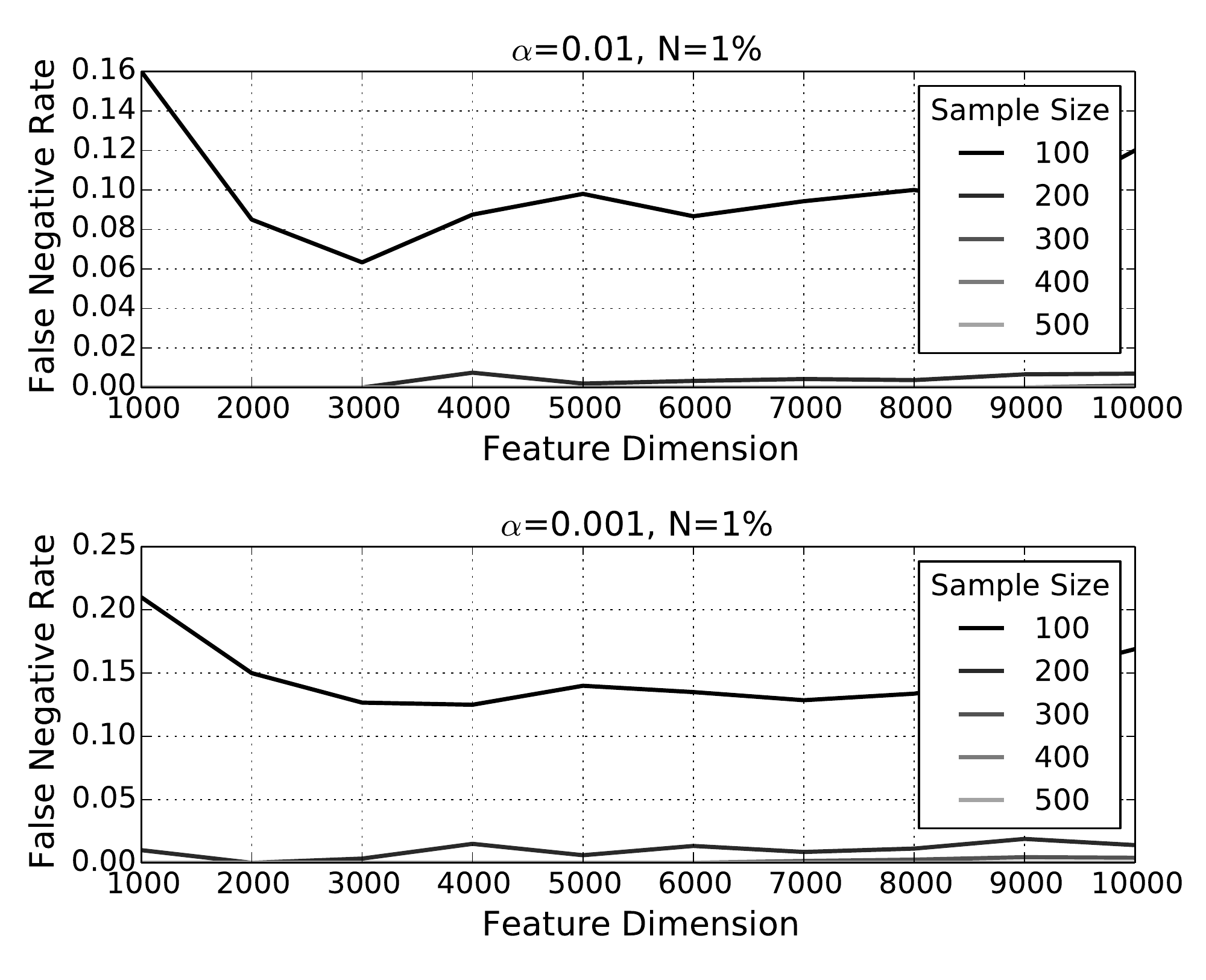}\\
\hline
\begin{sideways}$N=2\%$\end{sideways} &
\includegraphics[width=0.37\linewidth]{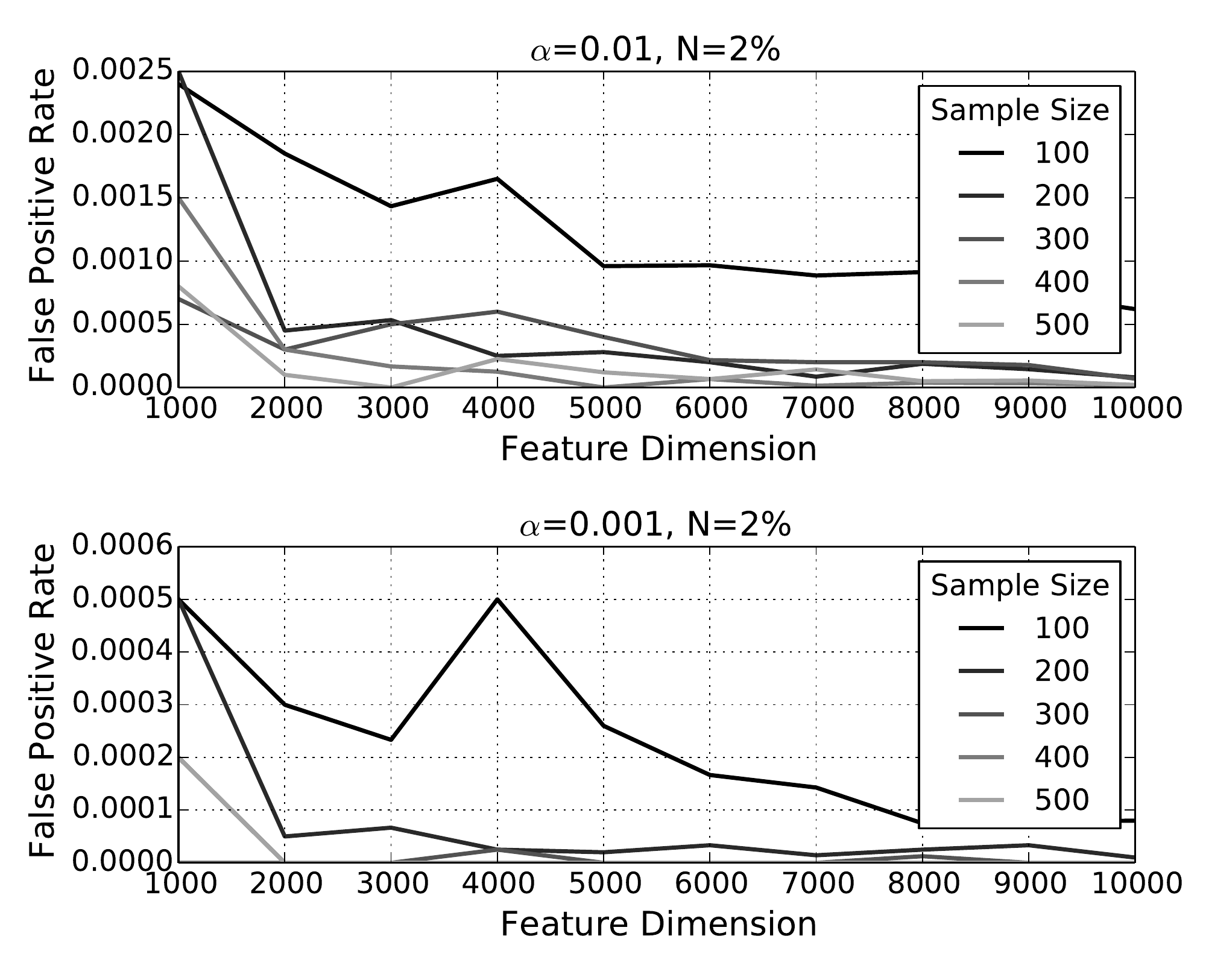}&
\includegraphics[width=0.37\linewidth]{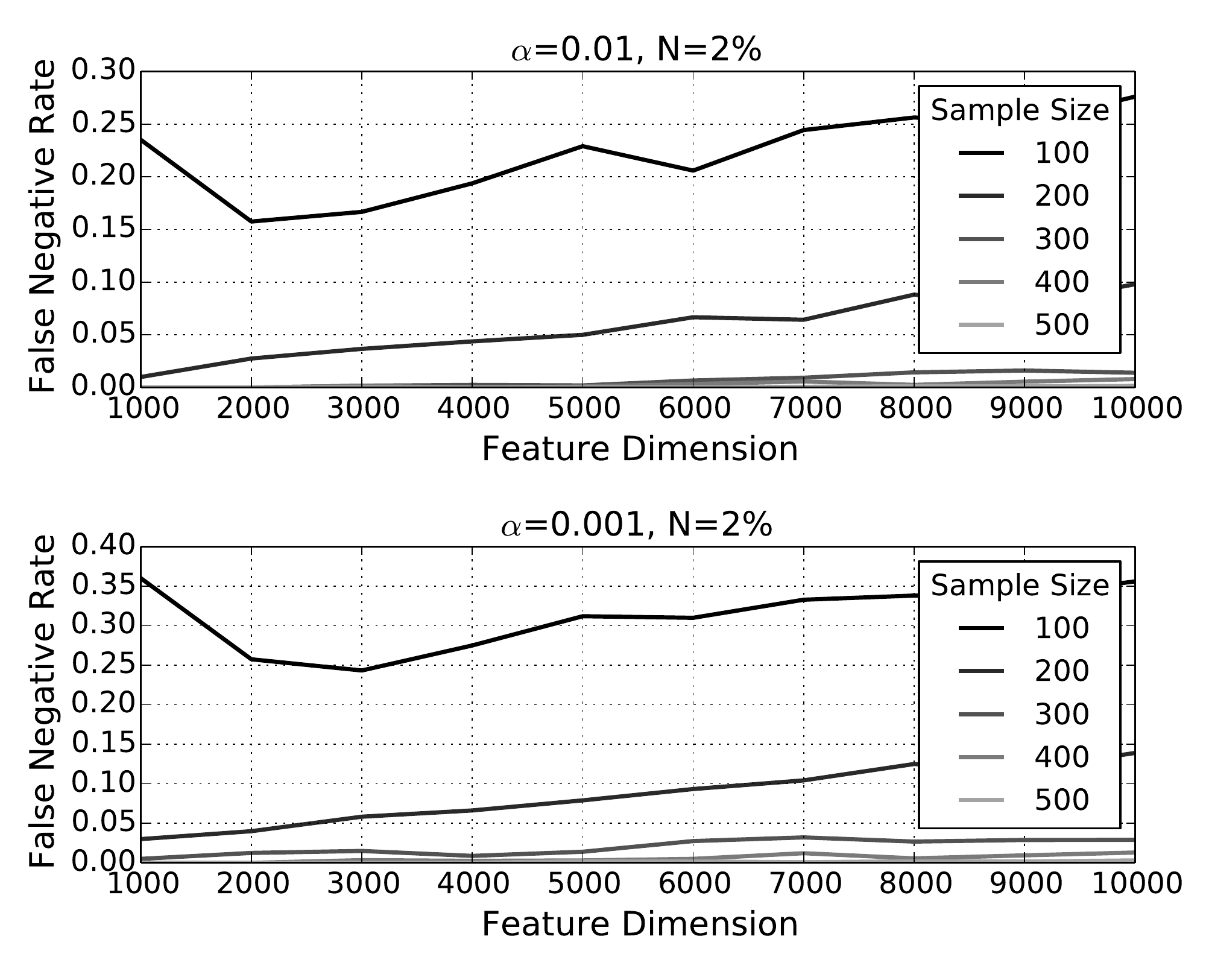}\\
\hline
\end{tabular}
\caption{\label{tab:StrategyI}\small Evaluating the proposed method in determining a good threshold for selection frequency that can limit the false positive rates at the desired level $\alpha$ while keeping the false negative rates low for training strategy I. Left column shows the graphs plotting observed false positive rates for different feature dimensions, sample sizes, proportion of relevant features and different desired false positive rates. The right columns shows the corresponding false negative rates.}
\end{table}
\begin{table}
\small
\centering
\begin{tabular}{|c|c|c|}
\hline
& False Positive Rates & False Negative Rates\\
\hline
\begin{sideways}$N=0$\end{sideways} &
\includegraphics[width=0.37\linewidth]{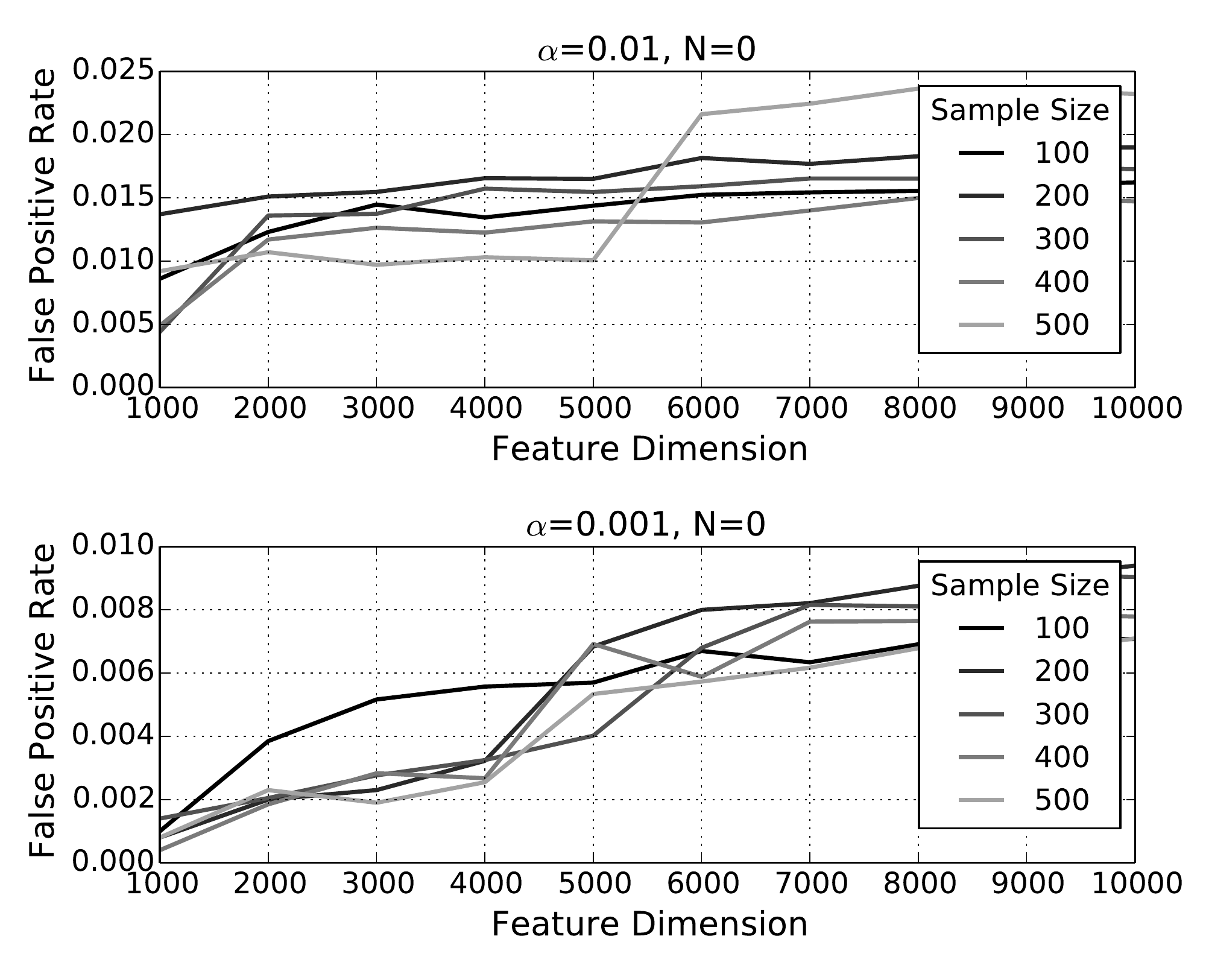}& N/A\\
\hline
\begin{sideways}$N=0.1\%$\end{sideways} &
\includegraphics[width=0.37\linewidth]{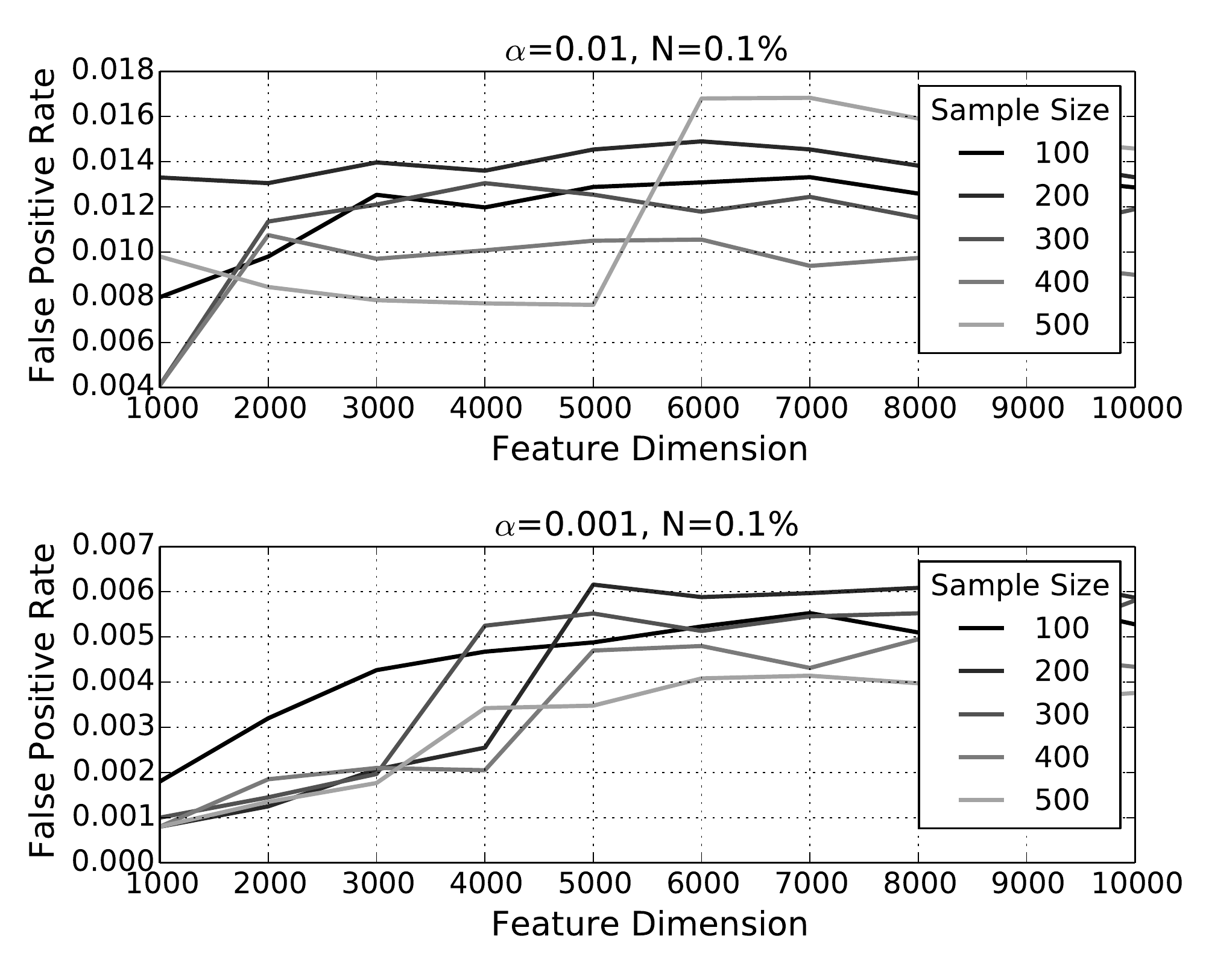}&
\includegraphics[width=0.37\linewidth]{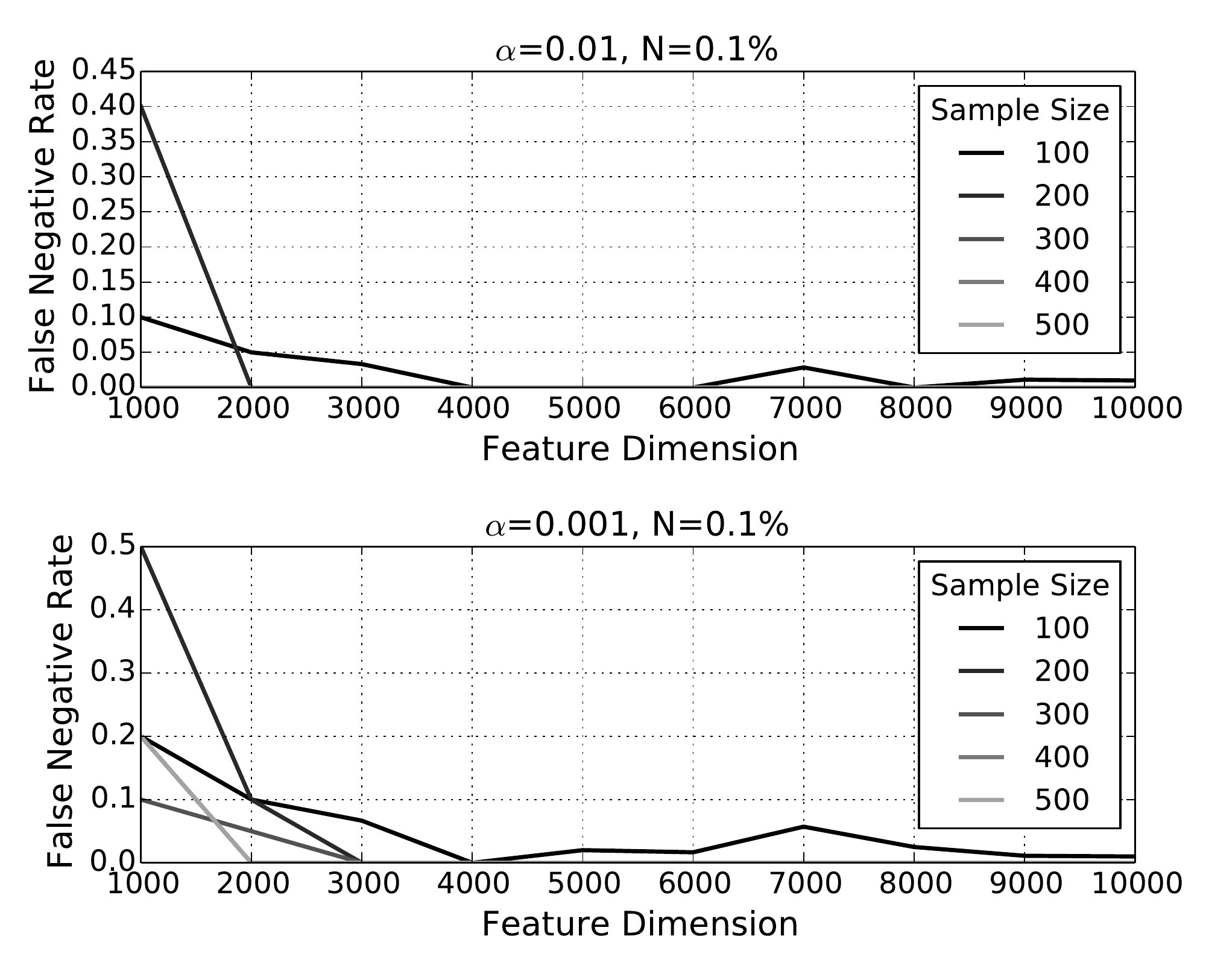}\\
\hline
\begin{sideways}$N=1\%$\end{sideways}&
\includegraphics[width=0.37\linewidth]{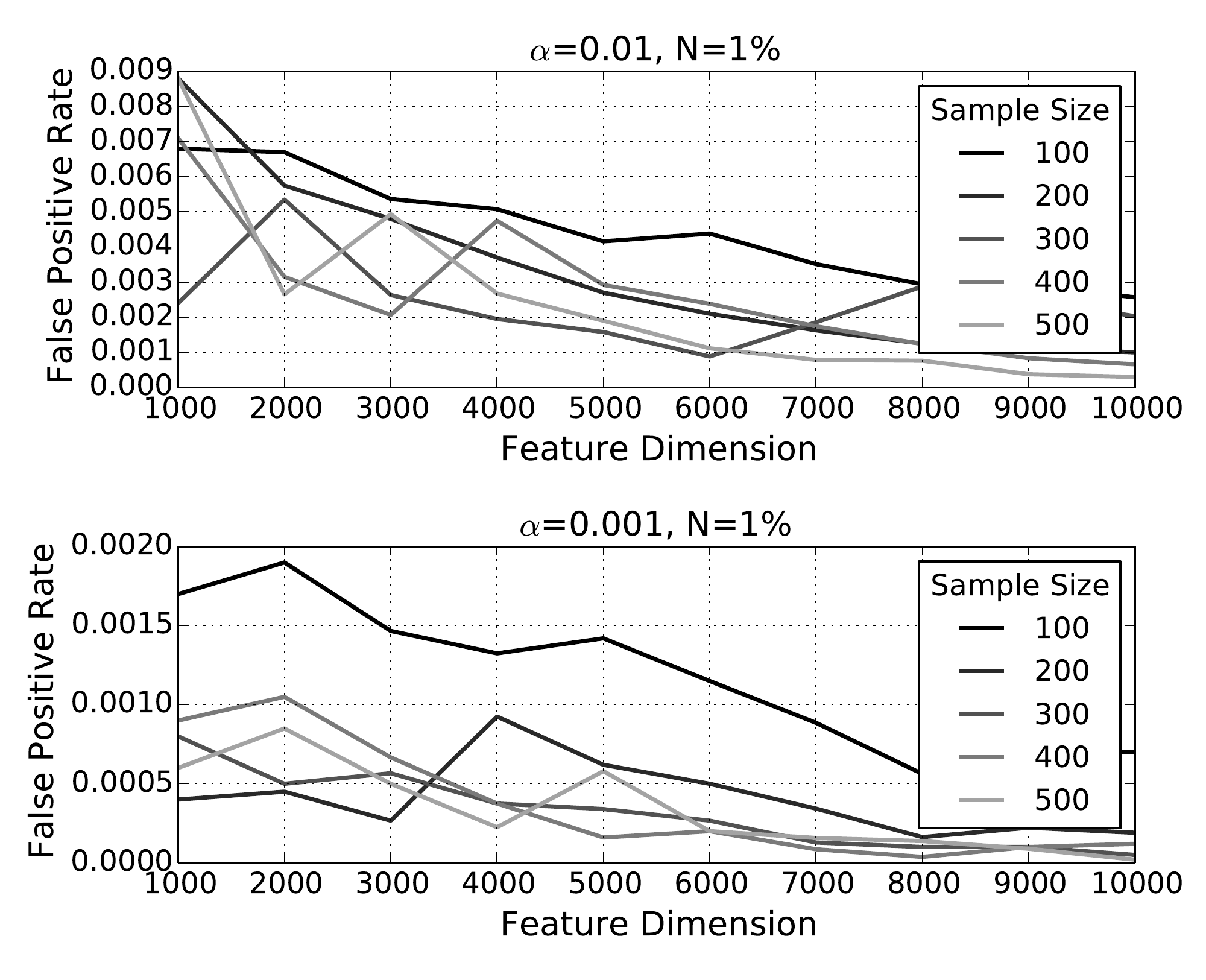}&
\includegraphics[width=0.37\linewidth]{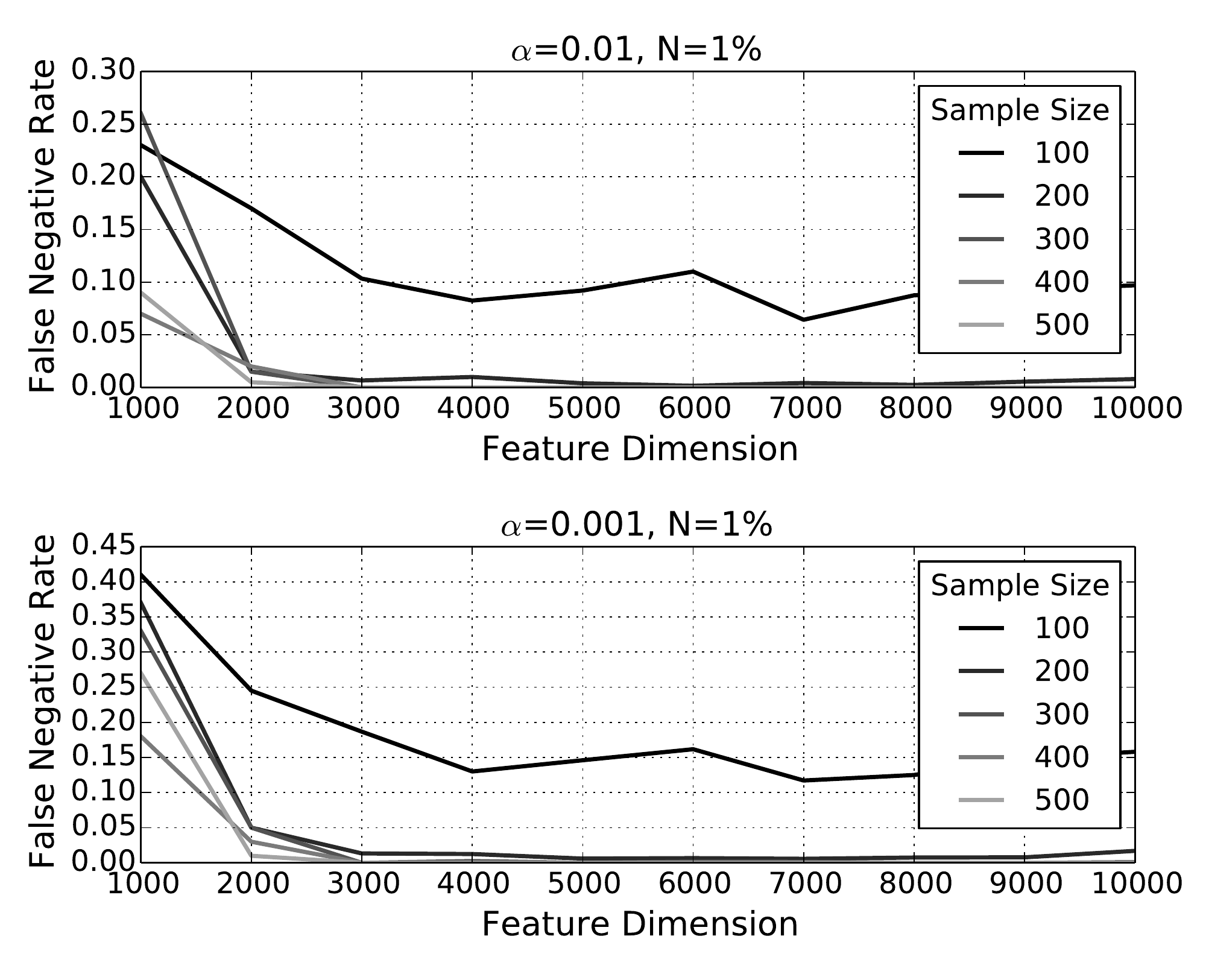}\\
\hline
\begin{sideways}$N=2\%$\end{sideways}&
\includegraphics[width=0.37\linewidth]{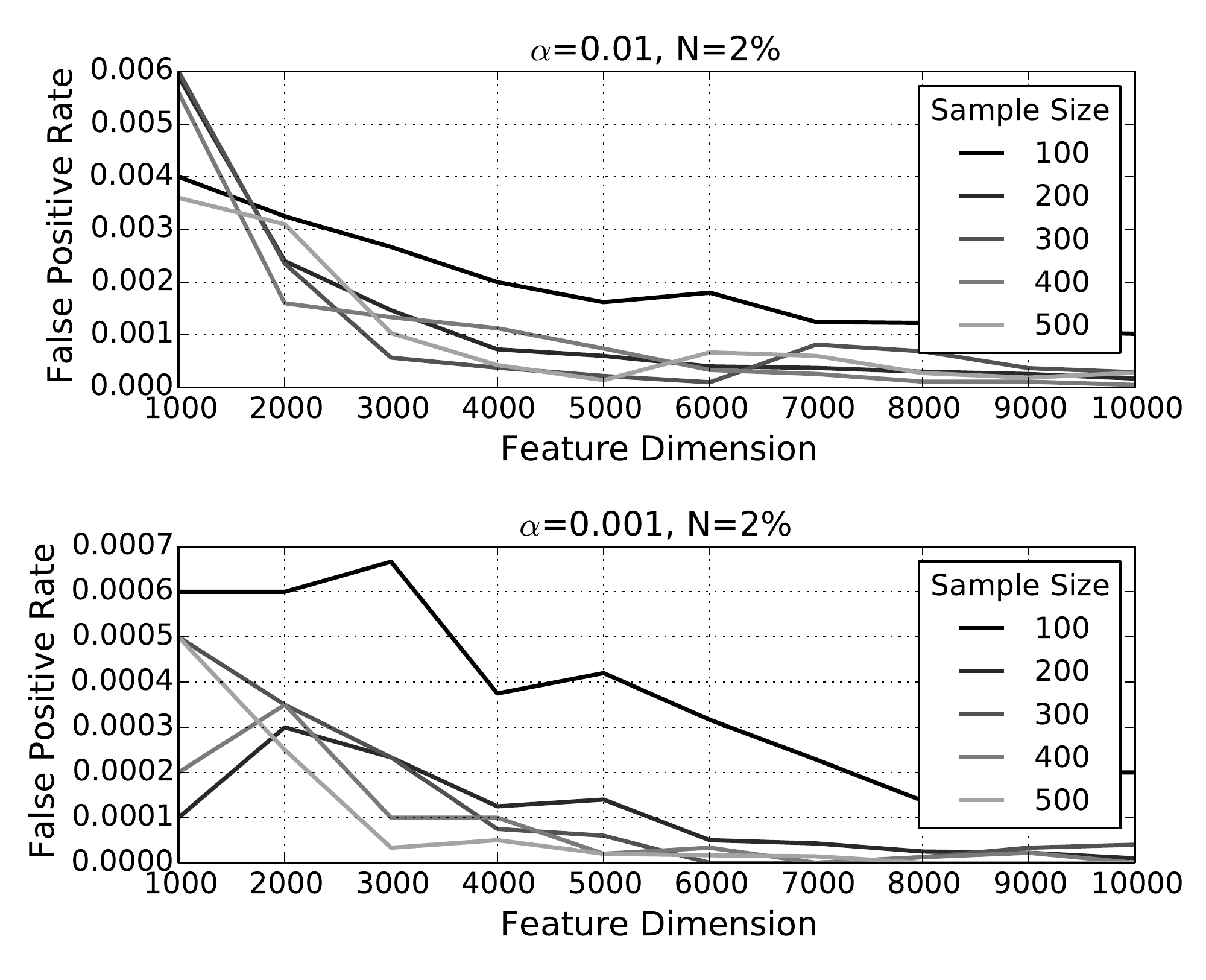}&
\includegraphics[width=0.37\linewidth]{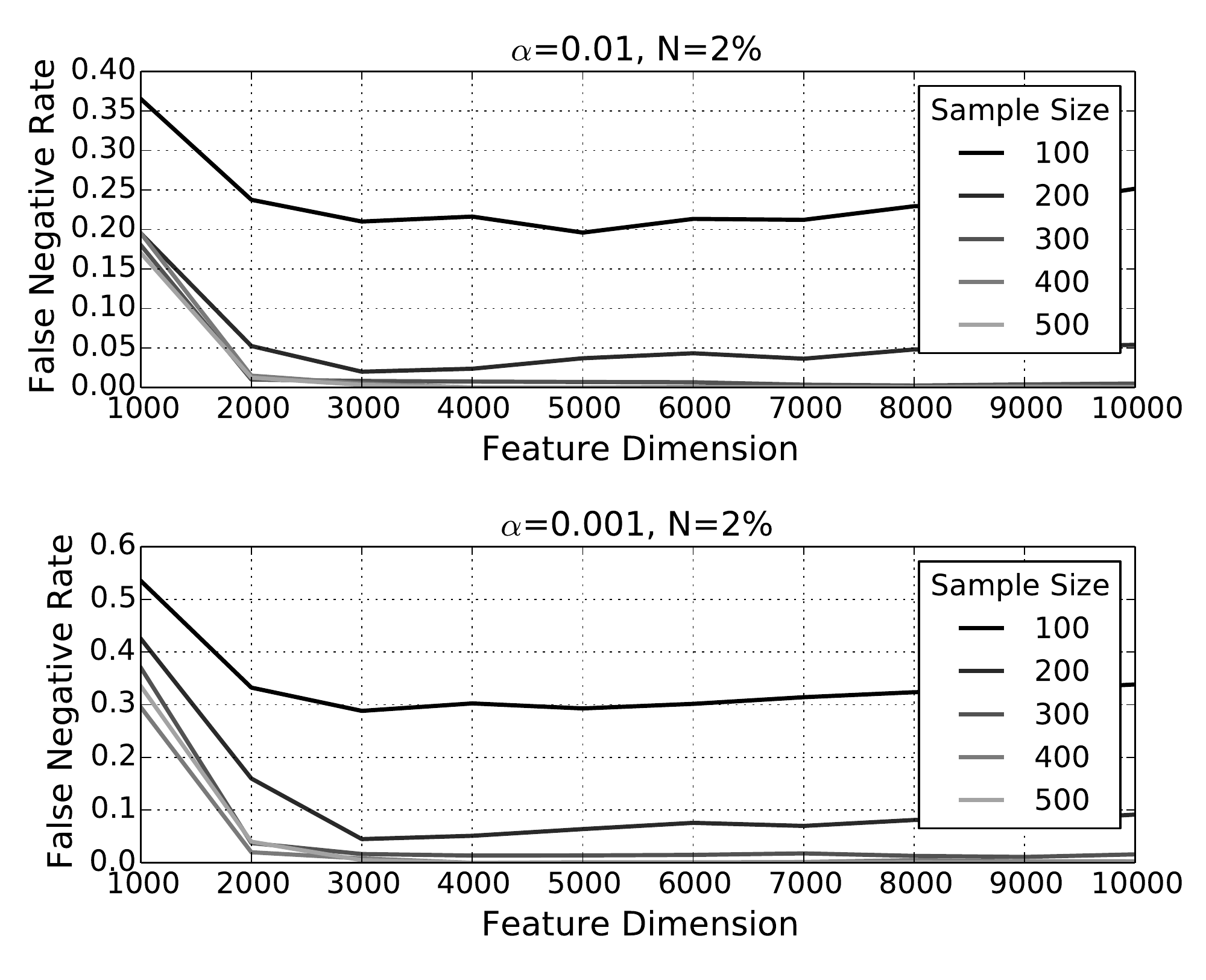}\\
\hline
\end{tabular}
\caption{\label{tab:StrategyII}\small Same plots as in Table~\ref{tab:StrategyI} but for the training strategy II.}
\end{table}
\noindent{\em Observations on False Positive Rates}: For the $N=0$ case the observed false positive rates are higher than the desired $\alpha$ value and they increase with increasing feature dimensions, which is expected. There is no consistent trend with changing sample size. The reason for this is when the sample size increases on the one hand the spurious correlations become scarcer but on the other hand the average number of nodes per tree, $K$, increases, which means more node optimizations and higher false positive rates. These conflicting factors seem to cancel out on average. For $N=0.1\%$ the false positive rates are still higher than the desired $\alpha$ but they are much lower than the $N=0$ case. For $\alpha=0.01$ we see that observed rates are in the vicinity of the desired value.  Changing sample size does not seem to have a consistent effect on the false positive rates in this case neither. 

For $N=1\%$ and $N=2\%$ we see a different behavior. First of all, for $\alpha=0.01$ all observed false positive rates are lower than the desired value for both $N$'s. For $\alpha=0.001$, when $N=2\%$ again all observed values are lower than $\alpha$ and when $N=1\%$ the majority is lower. The observed false positive rates decrease with increasing feature dimensions and this is due to the increase in the absolute number of relevant features. False positive rates also seem to decrease with increasing sample size and this we believe is due to the increasing statistical strength of relevant features, which leads to lower selection counts for non-relevant features. \\

\noindent{\em Observations on False Negative Rates}: First of all, and most importantly, the graphs show that false negative rates are fairly low for $S\geq200$ and $F\geq2000$, which demonstrates that the thresholds determined by the models are indeed able to split between relevant and non-relevant features. Secondly, the false negative rates are decreasing with increasing sample size, so, the proposed method does not share the same drawbacks as the statistical test Breiman and Cutler proposed in~\cite{Breiman2008}. 

We also see some trends with respect to changing $\alpha$, $N$ and $F$. First, as one would expect the false negative rates increase with decreasing $\alpha$. Second, the rates are increasing slightly with increasing $N$ and $F$. This we believe is due to the competition between relevant features. As the number of relevant features increases the ones that manifest lower empirical correlations with the label will be selected less often and might not beat the threshold.

One aspect we have not yet analyzed is the influence of the strength of the correlation between relevant features and the label on the false negative rates. As strength varies we expect false negative rates to change as well and in the next section we present the power analysis that focuses on this relationship. 
\paragraph{Power Analysis.} In this part we analyze the dependence of false negative and positive rates to the strength of the statistical relationship between relevant features and the label. To do this we experiment with different values of $\rho$ in Equation~\ref{eqn:relevant}, which is the theoretical Pearson's correlation coefficient between a relevant feature and the label. For the experiments we fix the sample size, the feature dimension and the number of features per node to $(S,F,F_n) = (150,5000,250)$. We then experiment with $N=[25, 50]$, $\rho=[0.1,0.2,\dots,0.9]$ and $T=[250,500,\dots,1750]$. For each parameter setting we train a forest, determine the selection frequency thresholds for $\alpha=0.01$ and $0.001$, and compute false negative and positive rates. We repeat the same experiment 20 times and present the average results in Table~\ref{tab:power_analysis}.

Each quadrant of Table~\ref{tab:power_analysis} has four graphs plotting the {\em true positive rates}, which is simply $1 - $ false negative rates, and false positive rates versus relevance strength parameter $\rho$ for $\alpha=0.01$ and $0.001$. Each graph plots the results for different forests composed of different number of trees. 
\begin{table}[!htb]
\centering
\small
\begin{tabular}{|c|cc|cc|}
\hline
& \multicolumn{2}{c|}{$N=25\ (0.5\%)$} & \multicolumn{2}{c|}{$N=50\ (1.0\%)$}\\
\hline
\multirow{2}{*}{\begin{sideways}Strategy I\end{sideways}} &
\includegraphics[width=0.21\linewidth]{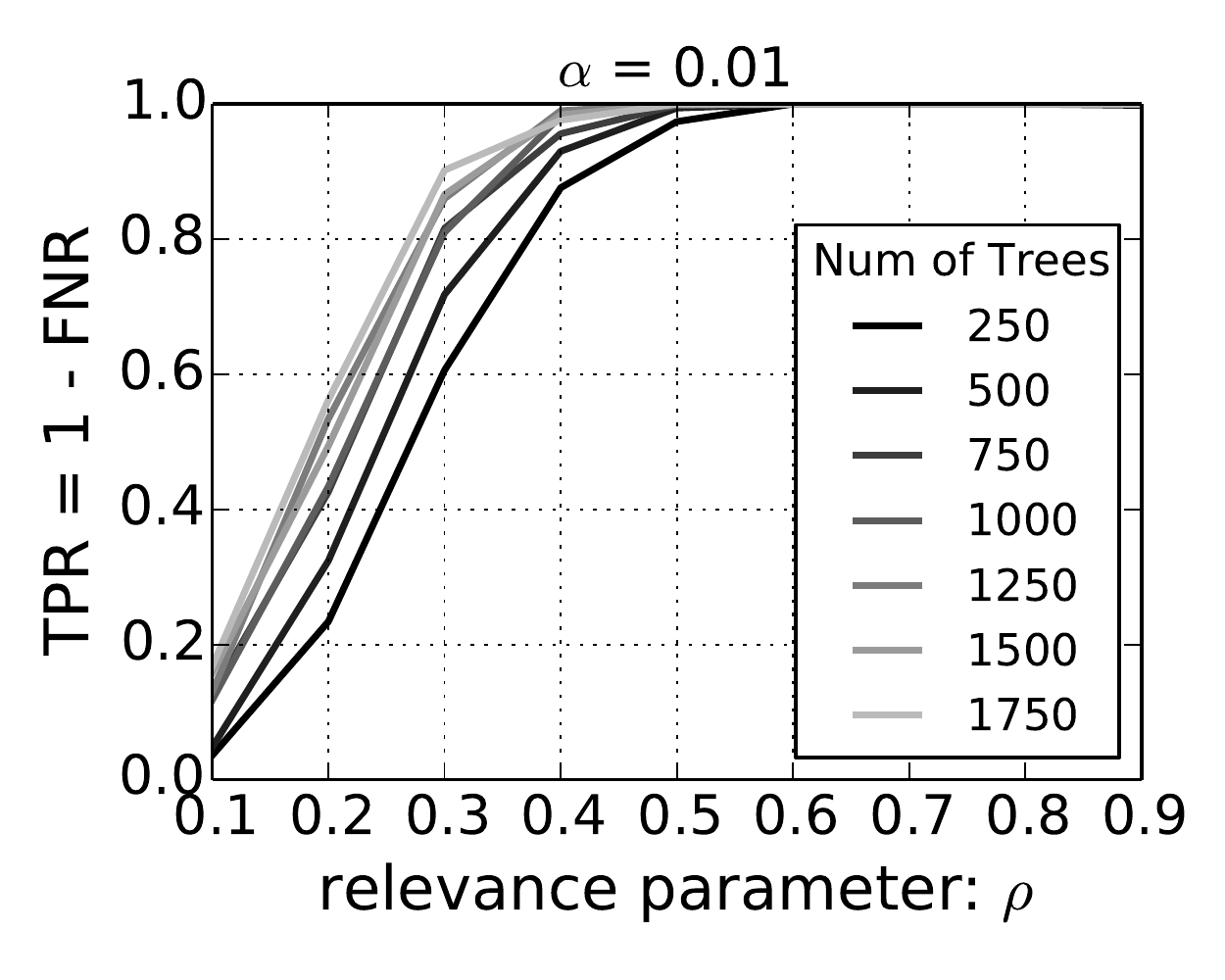} &
\includegraphics[width=0.21\linewidth]{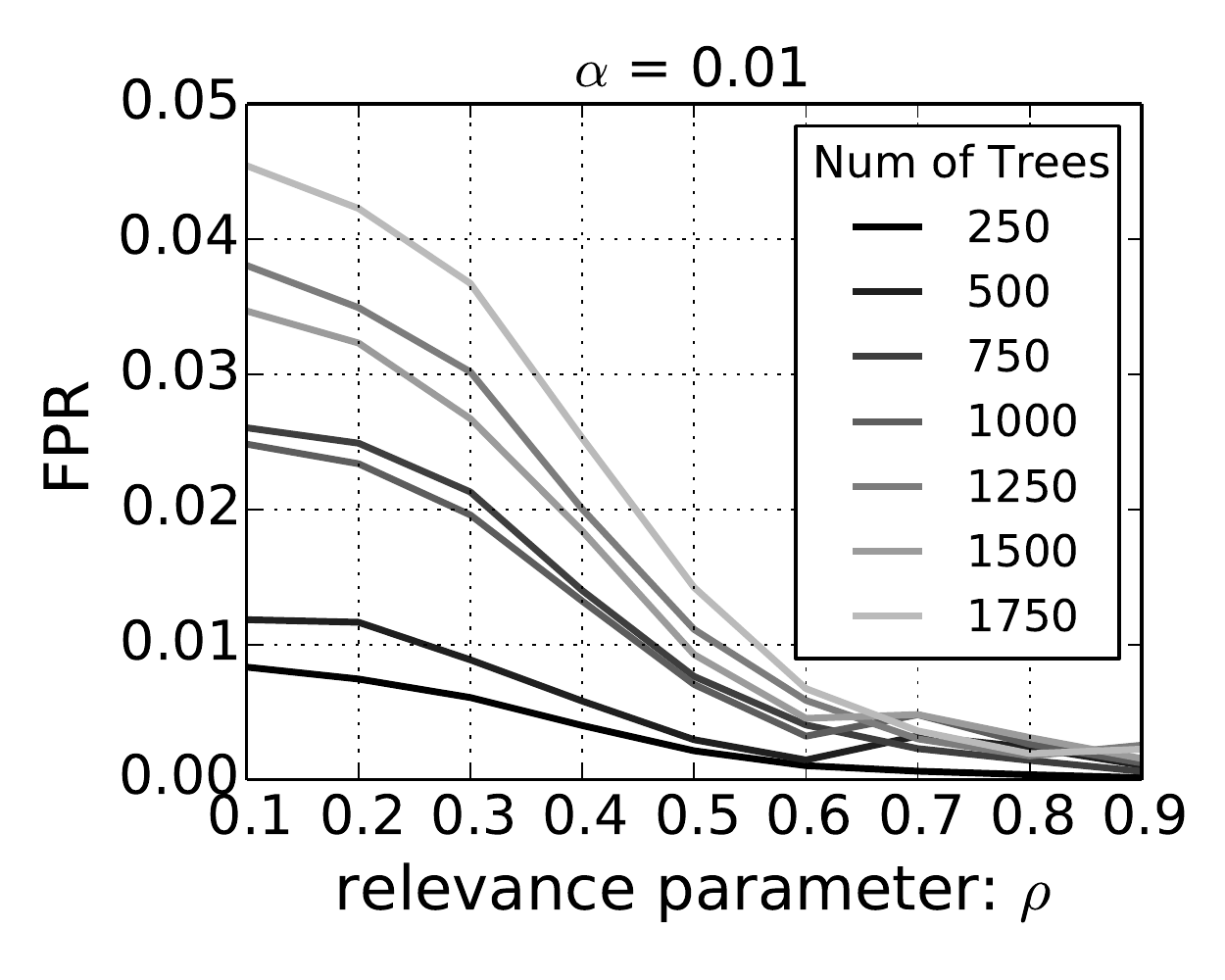} &
\includegraphics[width=0.21\linewidth]{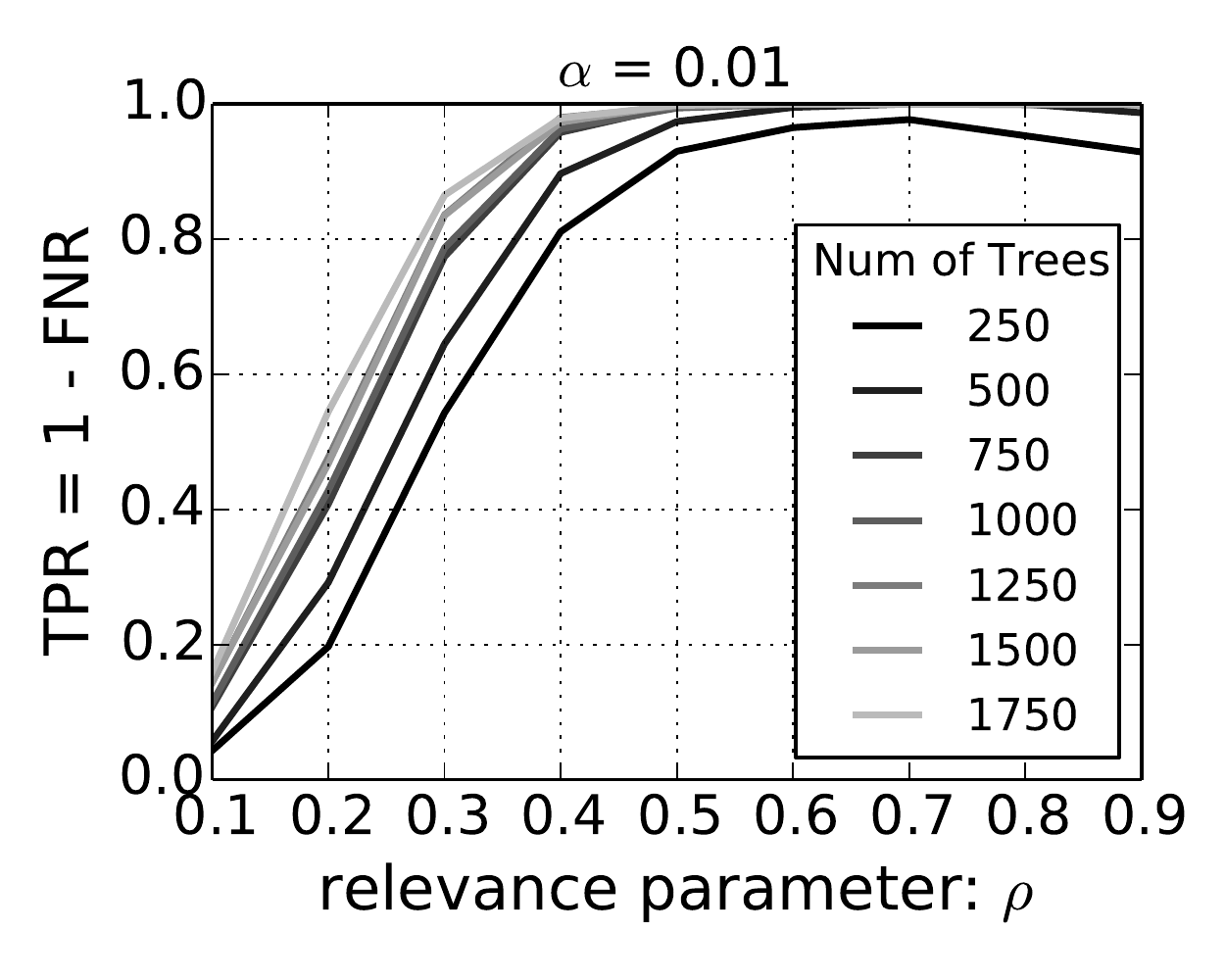} &
\includegraphics[width=0.21\linewidth]{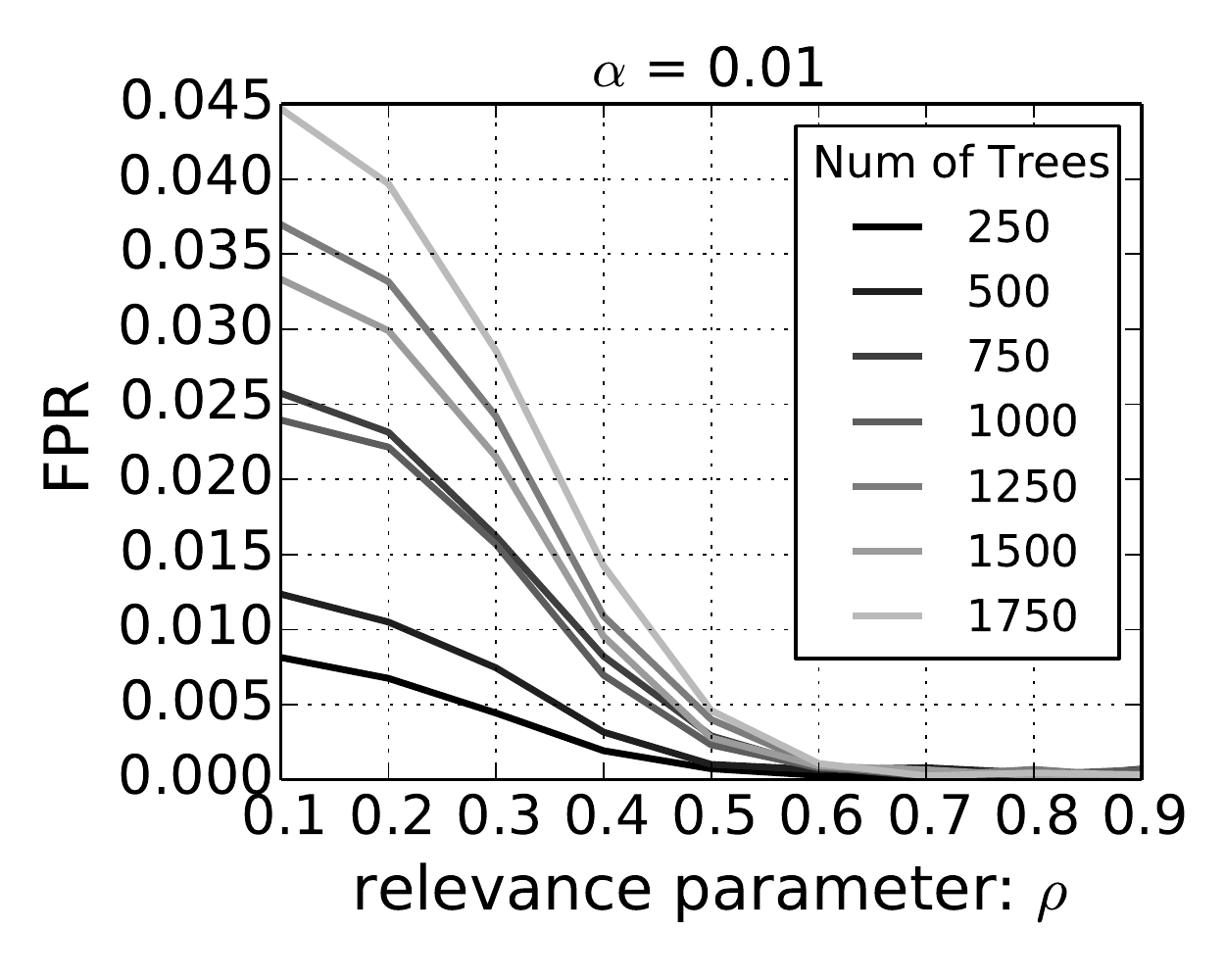} \\
&
\includegraphics[width=0.21\linewidth]{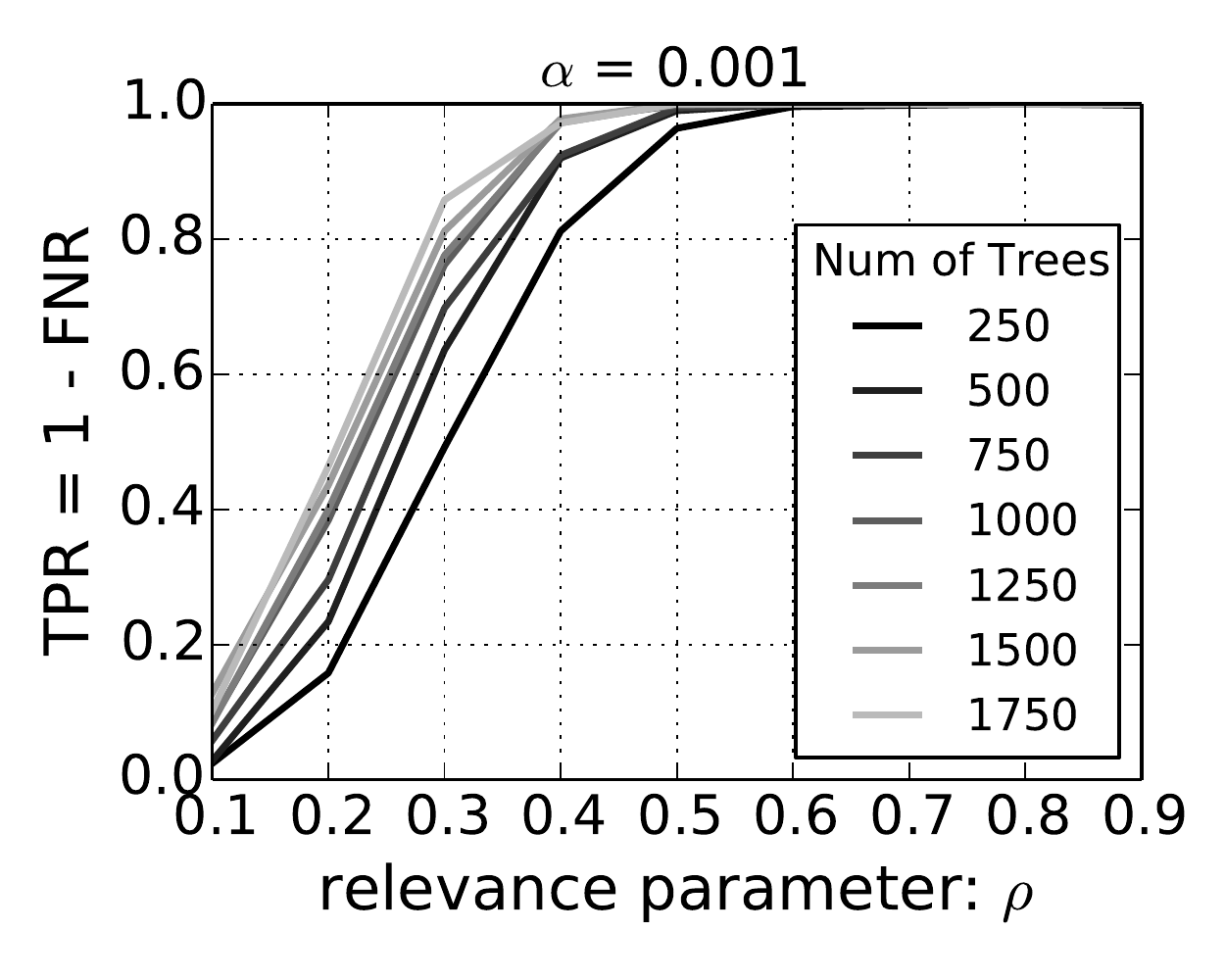} &
\includegraphics[width=0.21\linewidth]{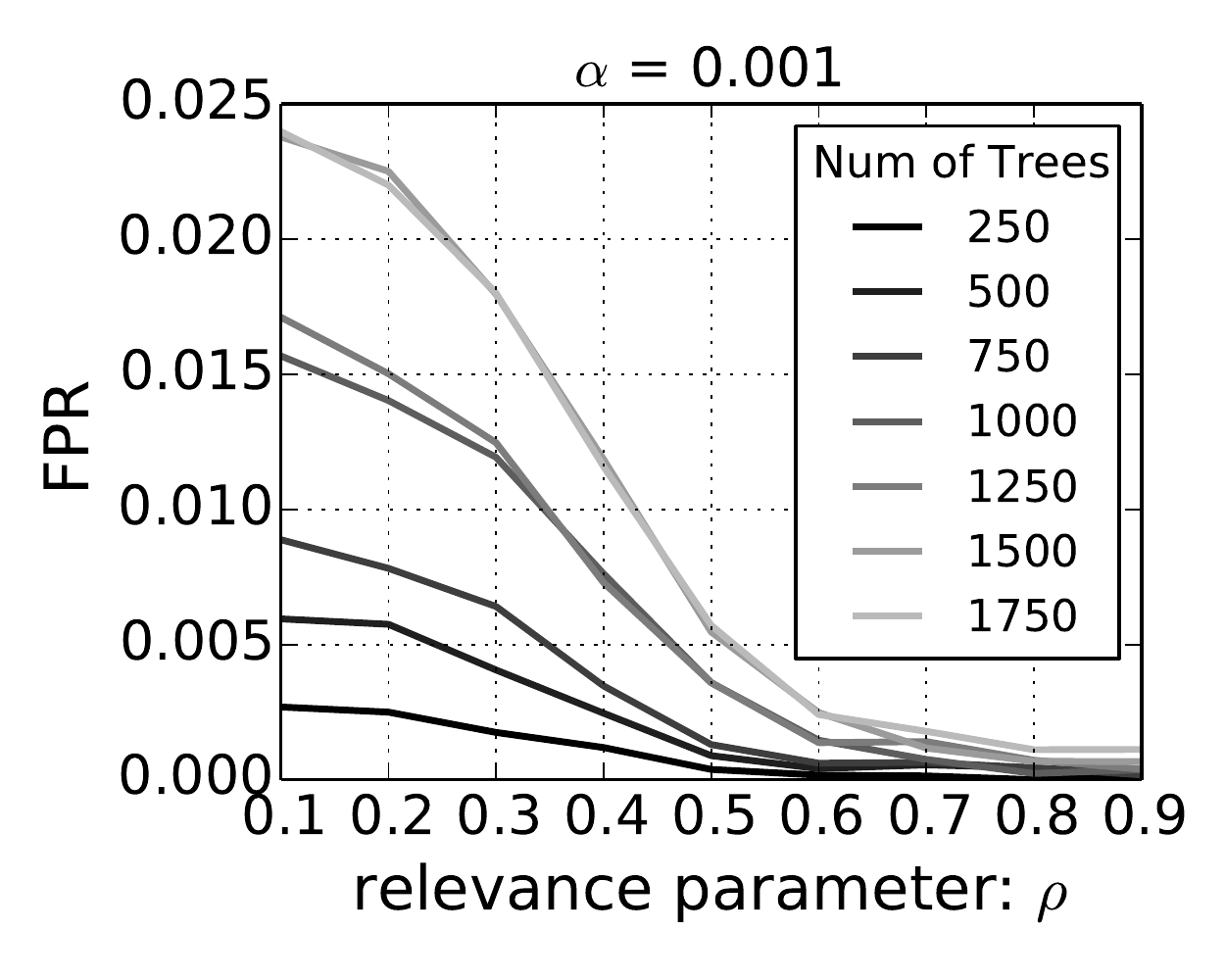} &
\includegraphics[width=0.21\linewidth]{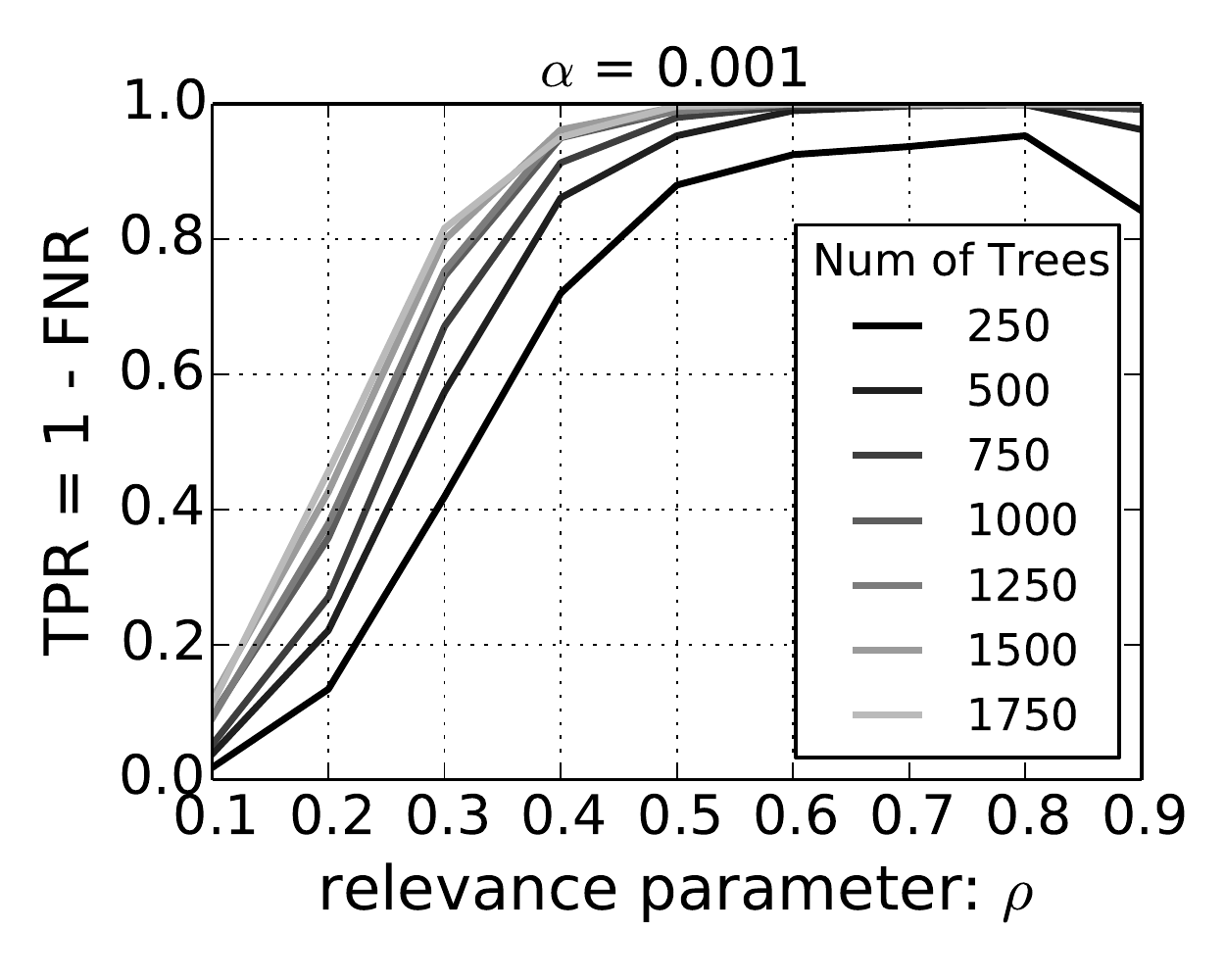} &
\includegraphics[width=0.21\linewidth]{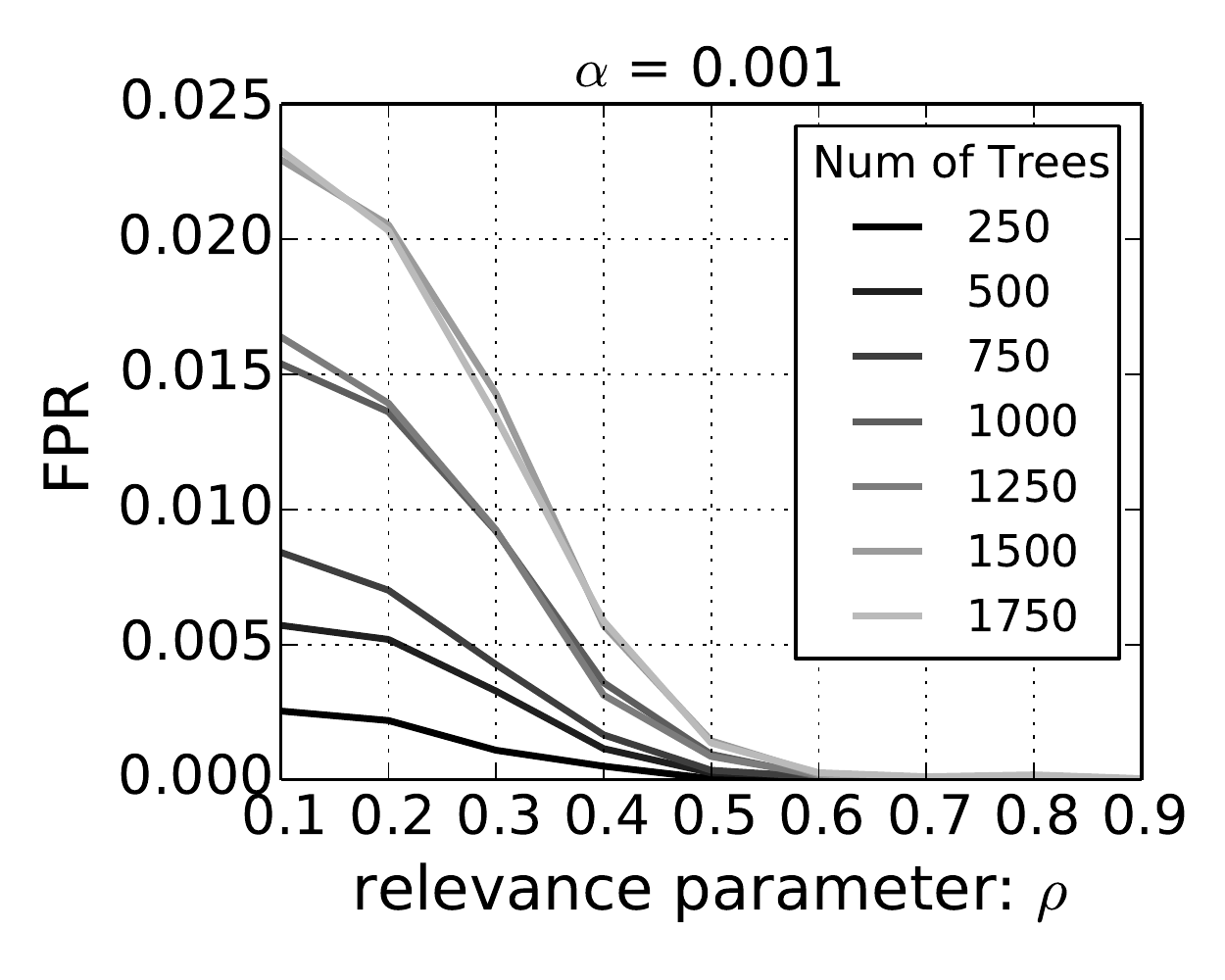} \\
\hline
\multirow{2}{*}{\begin{sideways}Strategy II\end{sideways}} &
\includegraphics[width=0.21\linewidth]{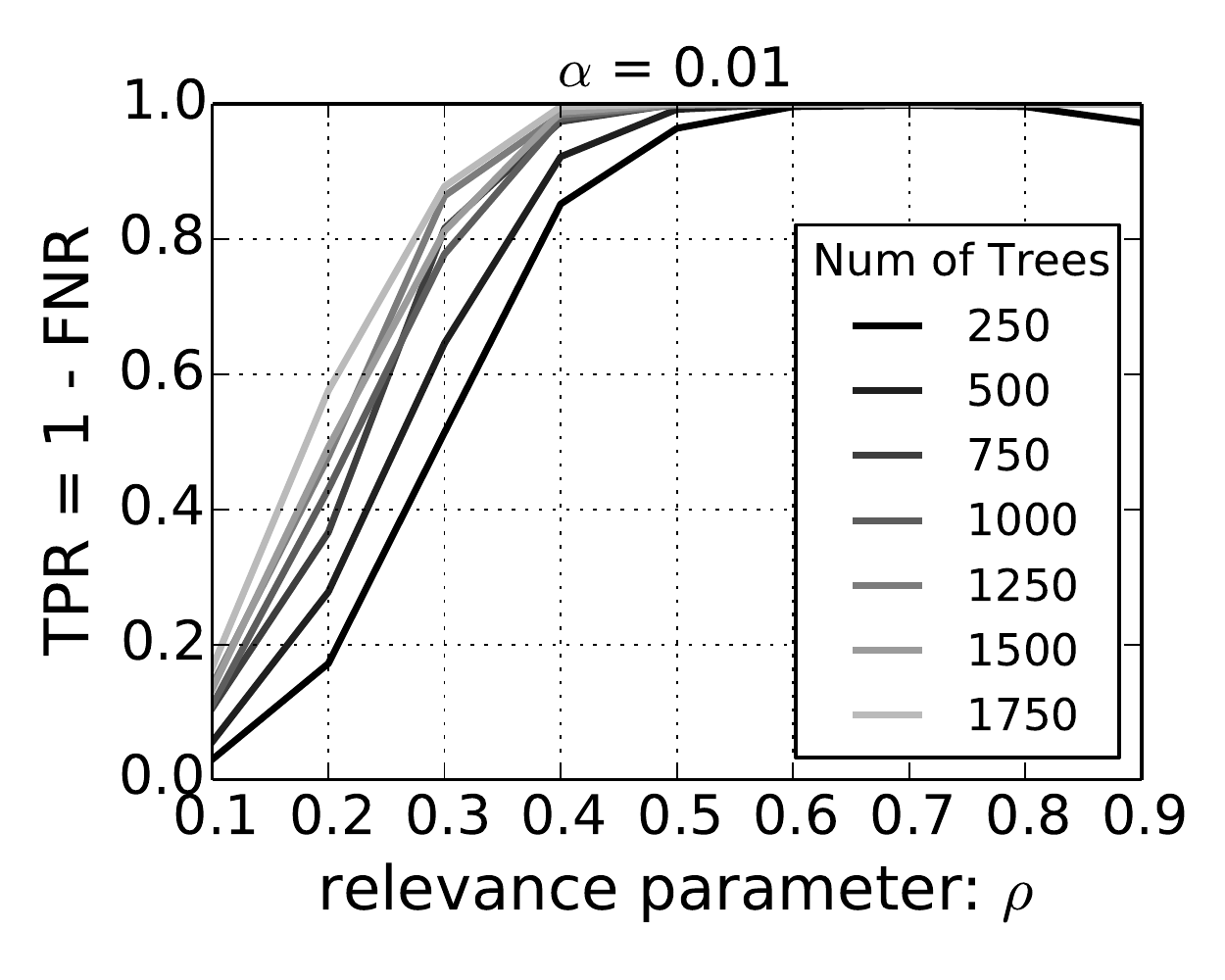} &
\includegraphics[width=0.21\linewidth]{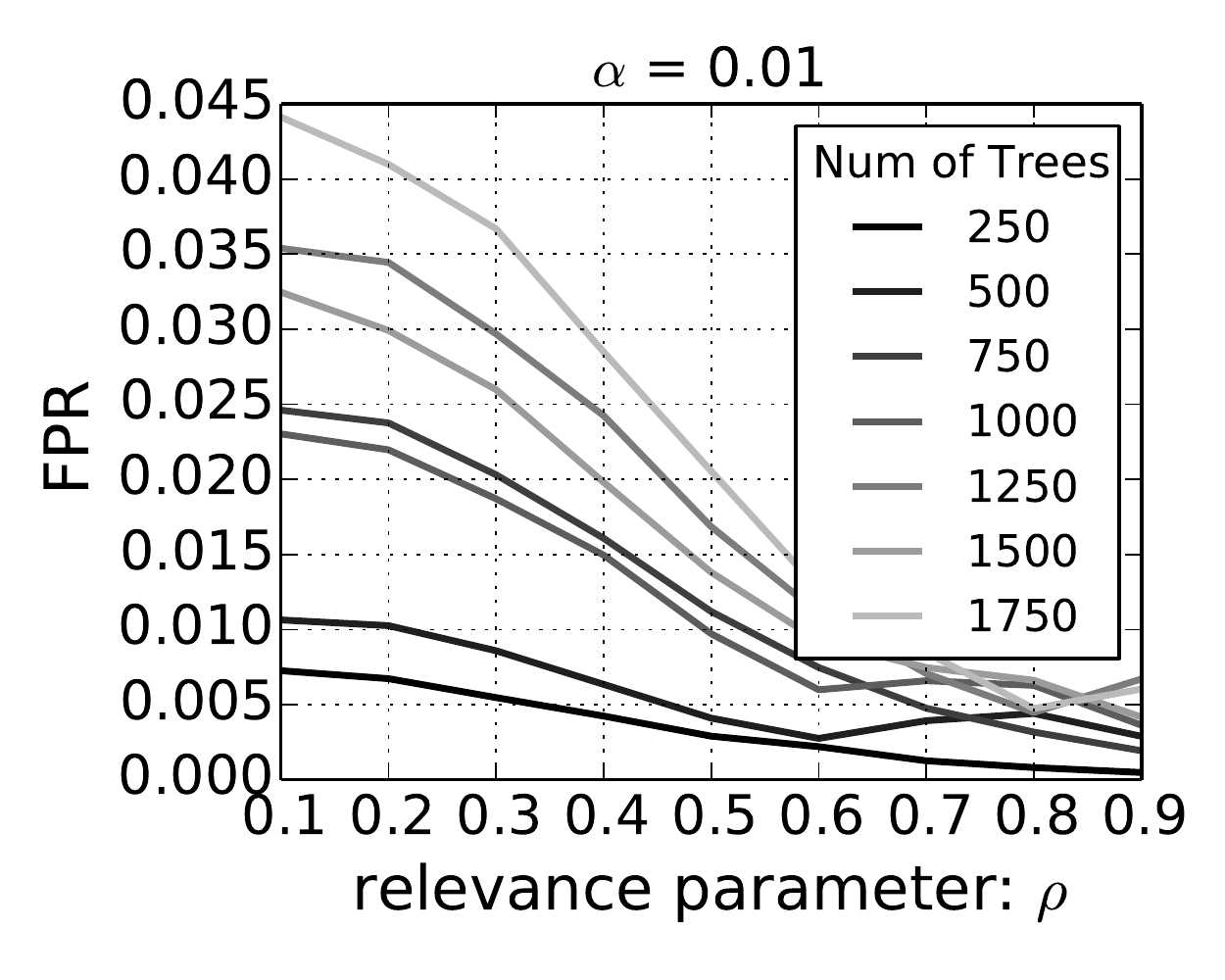} &
\includegraphics[width=0.21\linewidth]{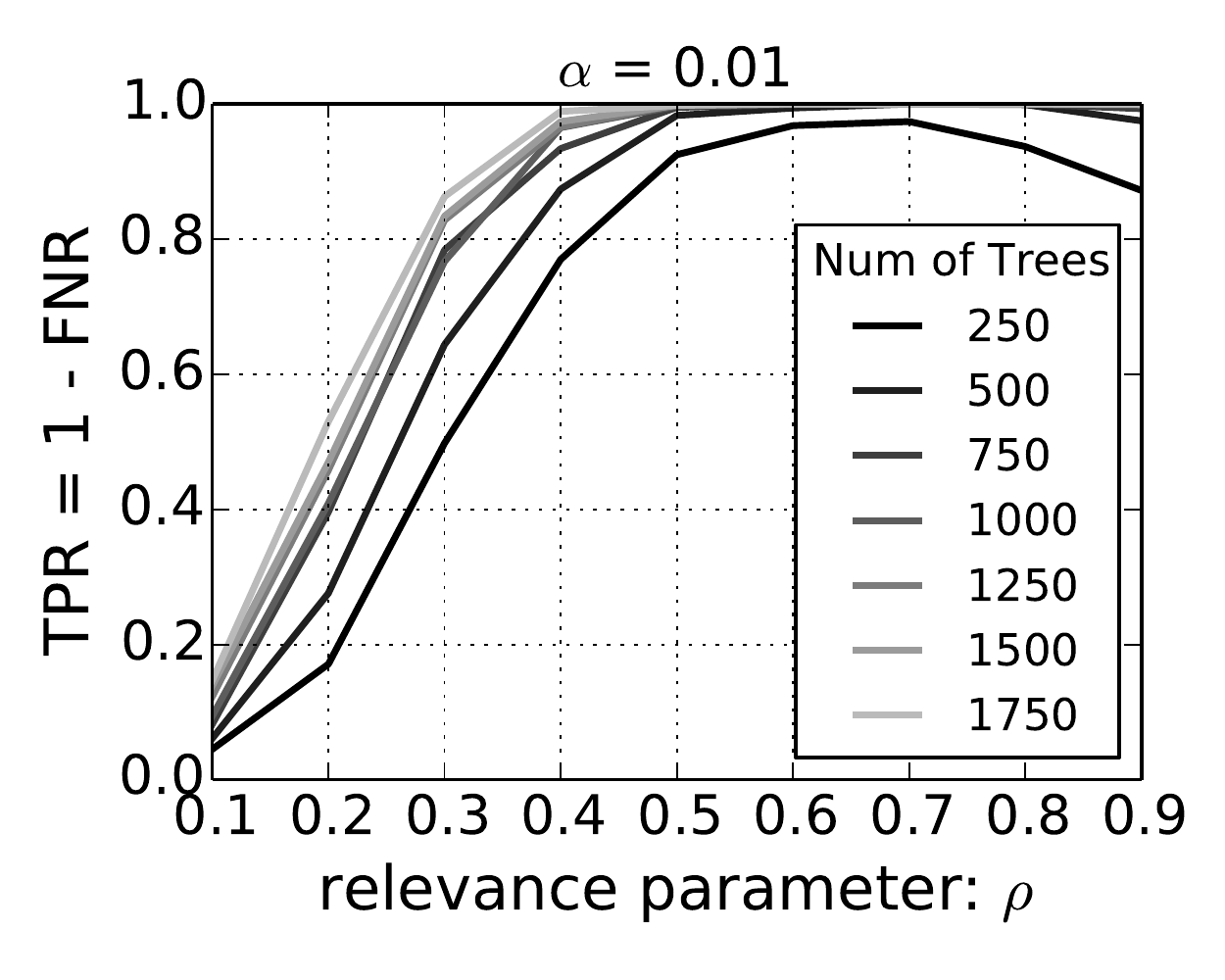} &
\includegraphics[width=0.21\linewidth]{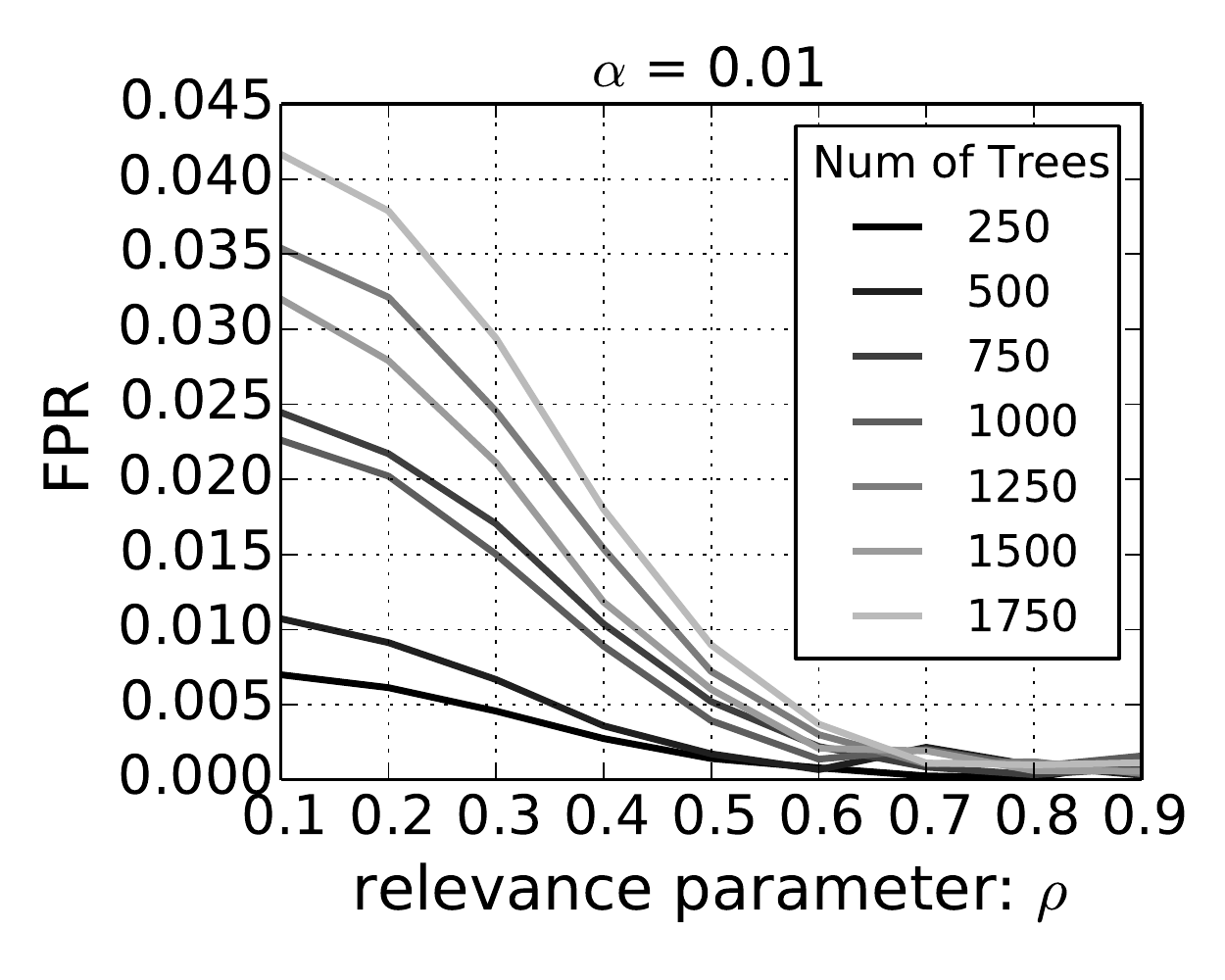} \\
&
\includegraphics[width=0.21\linewidth]{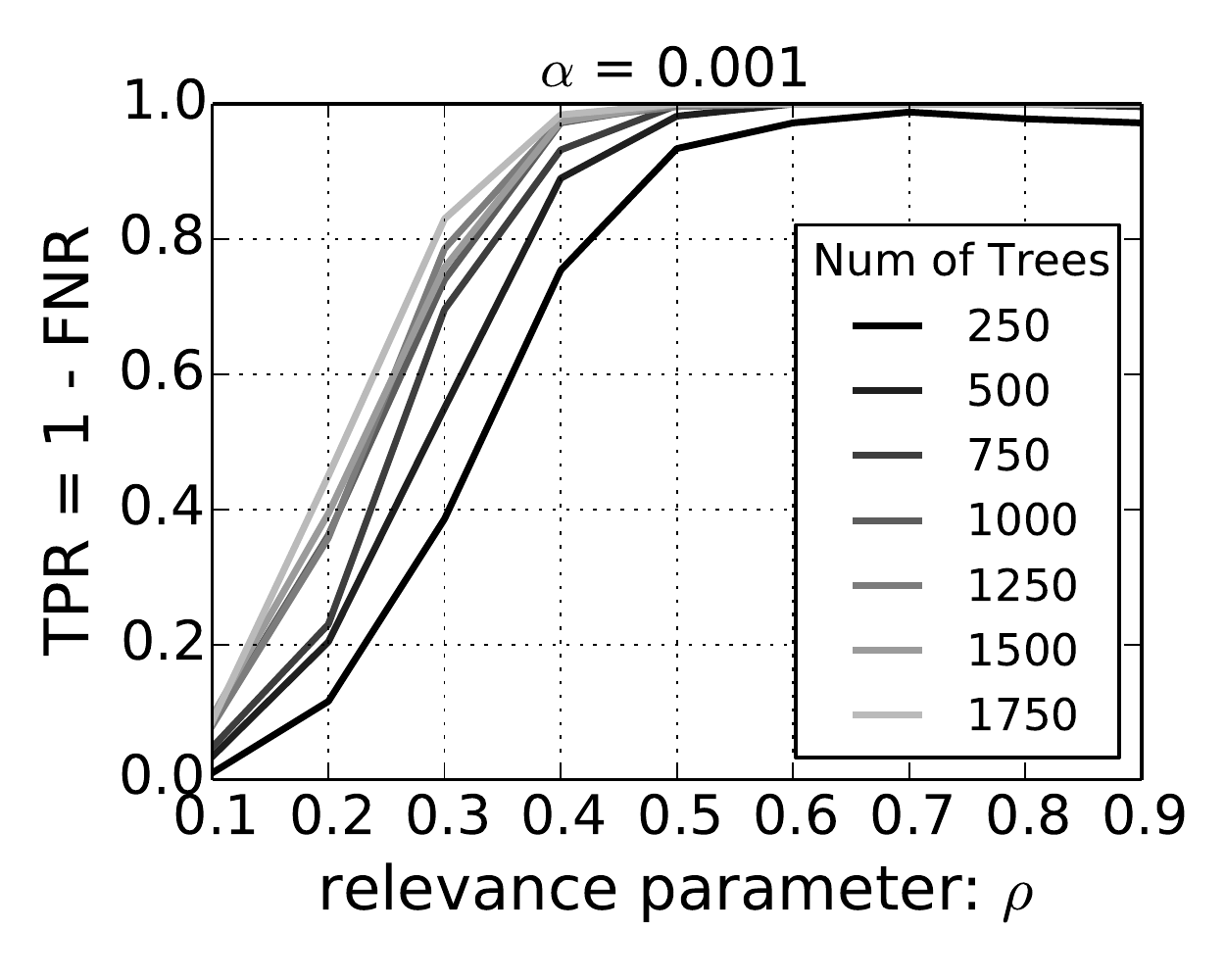} &
\includegraphics[width=0.21\linewidth]{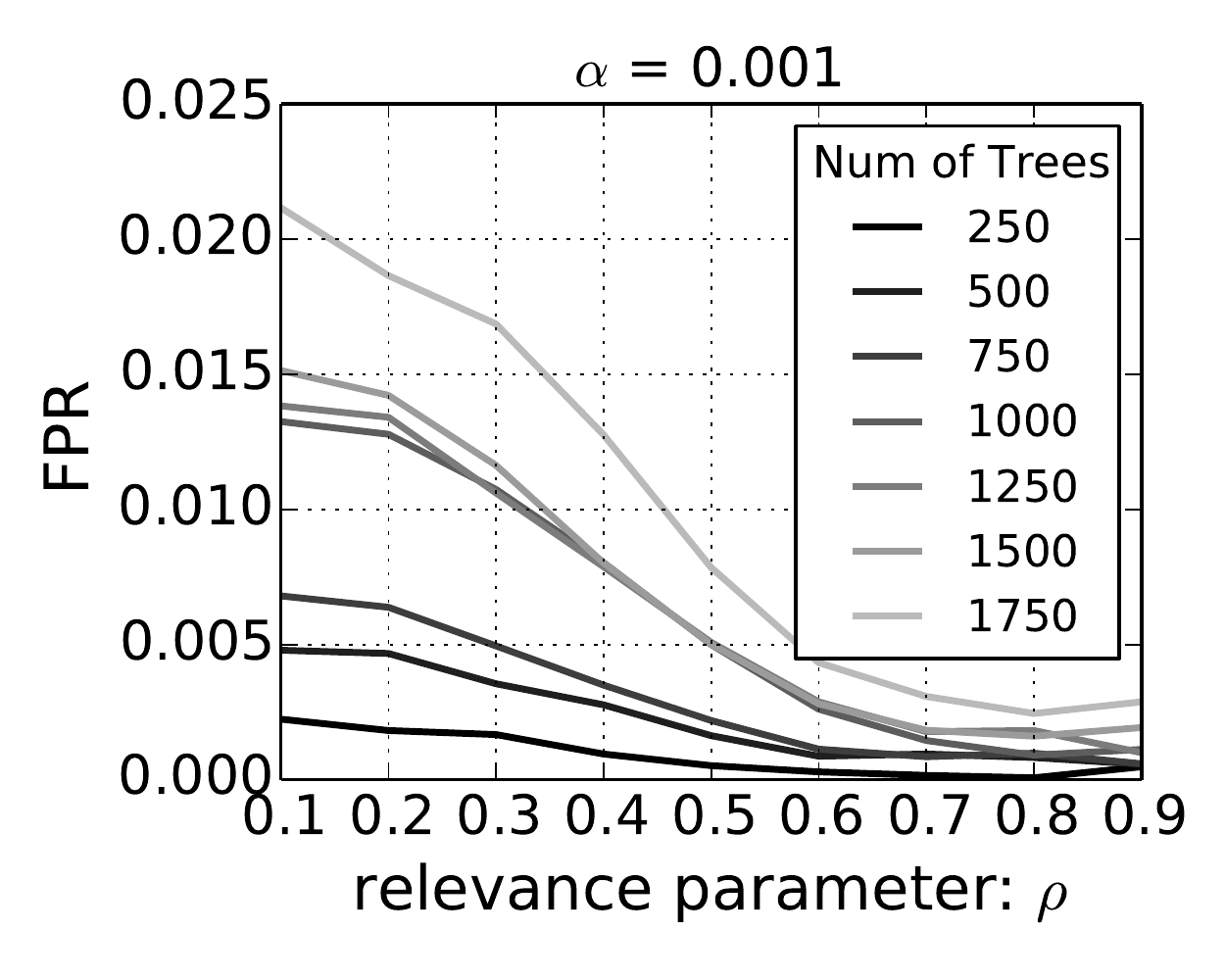} &
\includegraphics[width=0.21\linewidth]{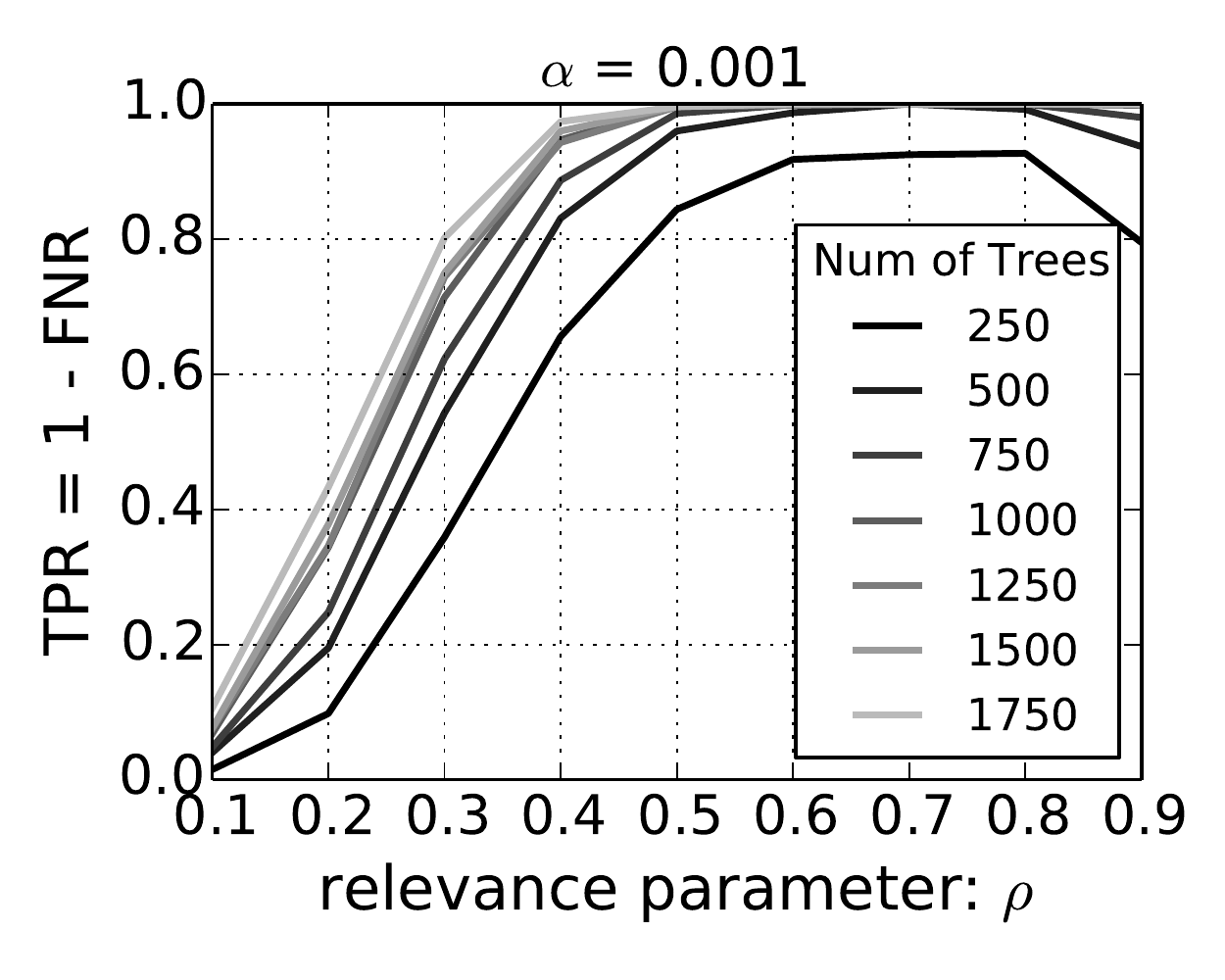} &
\includegraphics[width=0.21\linewidth]{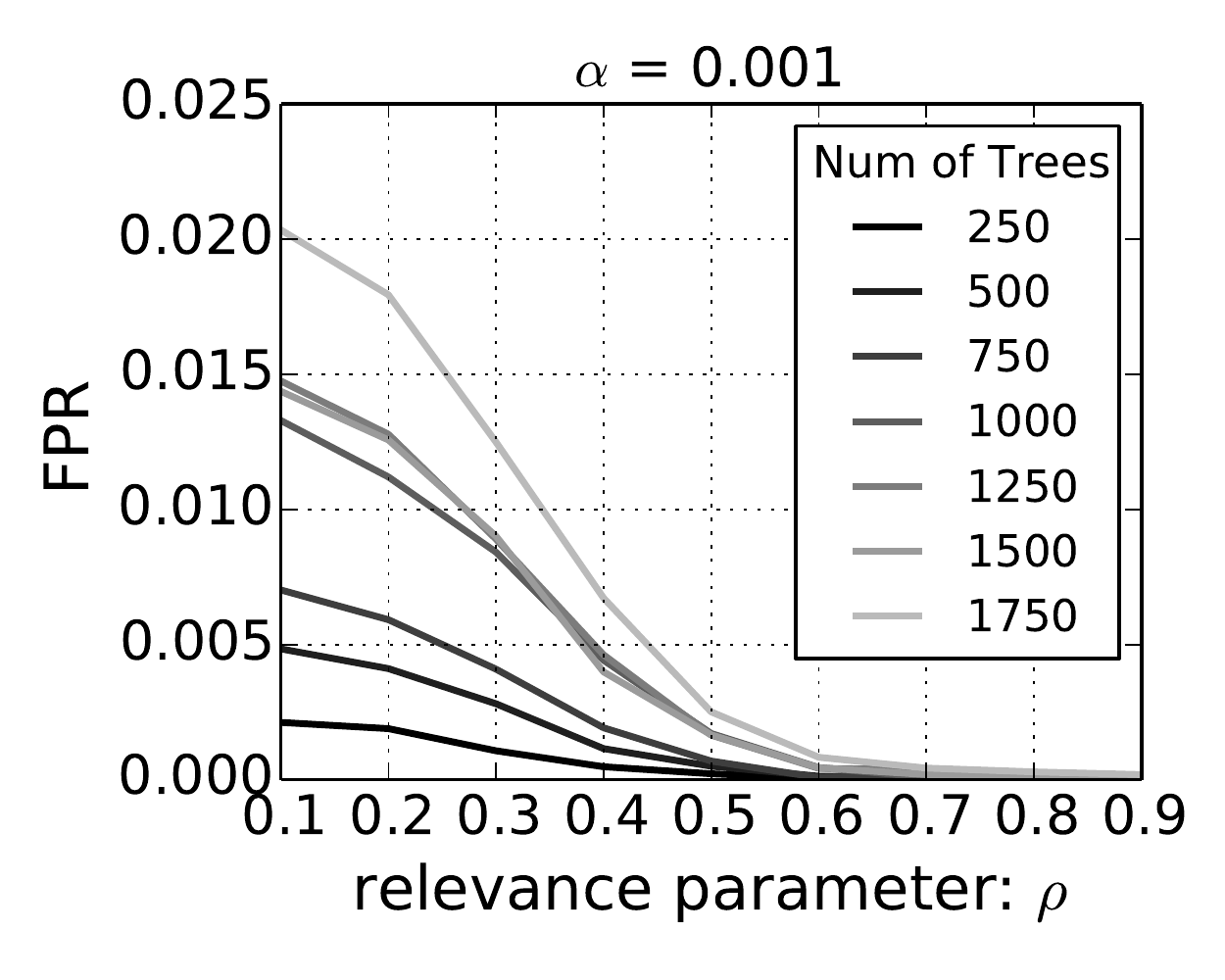} \\
\hline
\end{tabular}
\caption{\label{tab:power_analysis}\small Power Analysis for fixed sample size, feature dimension and size of the random feature subset: $(S,F,F_n) = (150,5000,250)$. Each quadrant shows graphs plotting True Positive Rates (1 - False Negative Rate) and False Positive Rates for varying relevance strength parameter $\rho$, number of trees and two different user set desired false positive rate limit $\alpha=(0.01,0.001)$. Different quadrants show the same plots for two different number of relevant features $N=(25,50)$ and different strategies.}
\end{table}

\noindent {\em Observations on true positive rates}: First of all, we observe that as expected an increase in relevant parameters results in an increase in true positive rates. Secondly, true positive rates increase with increasing number of trees. This behavior is not surprising since adding more trees will give the relevant parameters that have weaker empirical correlations with the label a higher chance to get selected. Moreover, the graphs show a convergence behavior where the contribution of $250$ additional trees is getting smaller as the number of trees increases. So, the power does not increase boundlessly. Third, increasing the number of relevant features causes true positive rates to decrease slightly due to the competition between relevant features.  Lastly, lower $\alpha$ yields a more conservative threshold which results in lower true positive rates. 

\noindent {\em Observations on the false positive rates}: First we observe that an increase in the relevance strength parameter results in a decrease in false positive rates. As the correlation strength between the relevant features and the label increases the chance of a nonrelevant feature getting selected at a node decreases. Furthermore, since for larger $N$ there will be more relevant features in random feature subsets, we expect that an increase in $\rho$ should cause a sharper decrease in the false positive rate for larger $N$. This is precisely what we observe in the graphs. Secondly, increasing number of trees result in increased false positive rates. This is not surprising because with an increasing number of trees a nonrelevant feature that has a weaker spurious statistical relation to the label will have a higher chance to get selected. 

The results in Table~\ref{tab:power_analysis} suggest that there is a trade-off between detecting more relevant features and higher false positive rates in choosing the number of trees in a forest. We believe this trade-off is an inherent characteristic of random forests. 
\paragraph{Comparison with Permutation Testing.} The last analysis we present in this section is a comparative study between the proposed method for determining a threshold and permutation testing as described in~\cite{Rodenburg2008} and~\cite{Altmann2010}. This permutation testing procedure aims to determine the null-distribution of a feature importance measure by permuting the labels of the samples and recomputing the feature importances. Then based on the importance measures computed in the permuted experiments one can compute a statistical significance of the original importance measure of a given feature as the portion of permutation experiments that yield a higher importance measure than the original one. Based on this significance one can split features as relevant and non-relevant by setting a desired significance rate. In this experiment, we applied this procedure to the selection frequency and compared the false positive and negative rates with those obtained using the proposed method. 

We experimented with different number of samples, feature dimensions and proportion of relevant features: $S = [100,250]$, $F = [500,2000,5000]$ and $N=[0, 2\%]$, and we fixed $F_n=F / 20$, $T = F / 10$ and the bagging ratio to $1/2$. For each parameter setting we first trained a forest and then for a desired false positive rate $\alpha$ we determined the model-based threshold and marked features with lower selection frequency as non-relevant. We also computed the permutation-based significance for each feature using $250$ permutations and marked features that had higher permutation-based p-value than $\alpha$ as non-relevant. For both methods, we computed false positive and negative rates for $\alpha=0.05$ and $0.01$, for which $250$ permutations are adequate. We repeated the same experiment 20 times and report the average values in Tables~\ref{tab:comparison_st1} and~\ref{tab:comparison_st2} for training strategies I and II respectively. 

\begin{table}[!htb]
\begin{centering}
\small
\begin{tabular}{|c|c|c|c|c|c|c|c|c|}
\hline
\multirow{2}{*}{$(S,F)$} & \multicolumn{4}{|c|}{Permutation Testing} & \multicolumn{4}{|c|}{Selection Frequency Model}\\
\cline{2-9}
& \multicolumn{2}{|c|}{$N=0$ - FPR} & \multicolumn{2}{|c|}{$N=2\%$ - FPR/FNR} & \multicolumn{2}{|c|}{$N=0$ - FPR} & \multicolumn{2}{|c|}{$N=\%$ - FPR/FNR}\\
\hline
$\alpha$ & 0.05 & 0.01 & 0.05 & 0.01 & 0.05 & 0.01 & 0.05 & 0.01 \\
\hline
(100, 500)  & 0.029 & 0.005 & 0.026/0.520 & 0.005/0.750 & 0.041 & 0.007 & 0.035/0.430 & 0.006/0.690\\
\hline
(100, 2000) & 0.033 & 0.007 & 0.015/0.133 & 0.002/0.308 & 0.047 & 0.013 & 0.022/0.088 & 0.004/0.205\\
\hline
(100, 5000) & 0.034 & 0.007 & 0.008/0.147 & 0.001/0.275 & 0.048 & 0.015 & 0.014/0.111 & 0.002/0.194\\
\hline
(250, 500)  & 0.030 & 0.008 & 0.020/0.340 & 0.005/0.500 & 0.028 & 0.009 & 0.018/0.350 & 0.005/0.430\\
\hline
(250, 2000) & 0.029 & 0.008 & 0.006/0.013 & 0.001/0.058 & 0.025 & 0.010 & 0.004/0.013 & 0.001/0.038\\
\hline
(250, 5000) & 0.032 & 0.008 & 0.002/0.012 & 0.000/0.029 & 0.027 & 0.010 & 0.007/0.006 & 0.001/0.012\\
\hline
\end{tabular}
\caption{\label{tab:comparison_st1}\small Comparison between the proposed selection frequency model and permutation testing for strategy I. The table provides false positive and negative rates for both the proposed model and permutation testing for different parameter settings and desired false positive limits (or permutation-based p-value in the case of permutation testing).}
\end{centering}
\end{table}

The permutation testing is able to limit the false positive rate at the desired level for all the experiments. The proposed model on the other hand, for $N=0$ and $\alpha=0.01$ results in slightly higher false positive rates than the desired level. On the other hand, for $N=2\%$ we observe that the proposed model limits the false positive rates at the desired level and yields lower false negative rates than the permutation testing. In other words it shows higher statistical power. 

In terms of computation times, the permutation testing is a computationally expensive approach. For the parameter setting ($(S,F)=(250,5000)$) permutations took 22 hours on an Intel Xeon 5472 3.0GHz CPU with 7GB of RAM and a forest code that is written in C++ (not optimized though). Although the permutation testing can be massively parallelized, the computational load prohibits its extensive use, especially in wrapper algorithms, such as backward or forward selection. The computation times for running the proposed selection frequency model is on the order of a second, which is almost negligible compared to the cost of training. Therefore, the proposed model can easily be integrated in wrapper algorithms. 

\begin{table}[!htb]
\begin{centering}
\small
\begin{tabular}{|c|c|c|c|c|c|c|c|c|}
\hline
\multirow{2}{*}{$(S,F)$} & \multicolumn{4}{|c|}{Permutation Testing} & \multicolumn{4}{|c|}{Selection Frequency Model}\\
\cline{2-9}
& \multicolumn{2}{|c|}{$N=0$ - FPR} & \multicolumn{2}{|c|}{$N=2\%$ - FPR/FNR} & \multicolumn{2}{|c|}{$N=0$ - FPR} & \multicolumn{2}{|c|}{$N=2\%$ - FPR/FNR}\\
\hline
$\alpha$ & 0.05 & 0.01 & 0.05 & 0.01 & 0.05 & 0.01 & 0.05 & 0.01 \\
\hline
(100, 500)  & 0.031 & 0.006 & 0.021/0.250 & 0.003/0.460 & 0.040 & 0.008 & 0.028/0.200 & 0.005/0.380\\
\hline
(100, 2000) & 0.033 & 0.008 & 0.011/0.098 & 0.001/0.238 & 0.048 & 0.014 & 0.016/0.073 & 0.002/0.145\\
\hline
(100, 5000) & 0.035 & 0.008 & 0.006/0.150 & 0.000/0.296 & 0.048 & 0.016 & 0.010/0.117 & 0.001/0.202\\
\hline
(250, 500)  & 0.028 & 0.008 & 0.011/0.000 & 0.002/0.020 & 0.019 & 0.006 & 0.008/0.000 & 0.002/0.010\\
\hline
(250, 2000) & 0.029 & 0.008 & 0.006/0.013 & 0.001/0.058 & 0.025 & 0.010 & 0.004/0.013 & 0.001/0.038\\
\hline
(250, 5000) & 0.032 & 0.009 & 0.000/0.017 & 0.000/0.053 & 0.027 & 0.013 & 0.002/0.008 & 0.000/0.019\\
\hline
\end{tabular}
\caption{\label{tab:comparison_st2}\small Same comparisons as in Table~\ref{tab:comparison_st1} but for training strategy II.}
\end{centering}
\end{table}
\subsection{Correlated Case: Cortical Thickness Maps}
In real world problems features are rarely independent from each other. Depending on the type of data there is often a complex dependency structure. Modeling the effects of a dependency structure on the false positive rates is not a trivial task. As often done in the literature, in our modeling approach we also have not explicitly tackled this task. In this section we present experiments that empirically analyze the use of the proposed method in the presence of a complex correlation structure between features. 

We experiment with {\em cortical thickness measurements} derived from brain magnetic resonance images (MRI), a very widely used type of data in neuroimaging and neuroscience. Thickness of human cortex can be measured from brain MRI using computational tools such as the Freesurfer software suite~\cite{dale1999cortical,fischl1999cortical,fischl2000measuring}. Such software tools produce discretized surface mesh representations of the human cortex with a cortical thickness measurement associated at each vertex, see Figure~\ref{fig:thickness}(a) for an example. Multiple studies in neuroscience and neuroimaging use these measurements to understand how the human cortex is structured as well as to detect the anatomical markers of neuropathological diseases and psychiatric disorders. 

The vector of cortical thickness measurements, also called thickness maps, have a very complex correlation structure. They provide a very challenging problem for false positive rate control in feature selection. We evaluate the proposed method using these maps. In particular, we build synthetic cortical thickness maps using a covariance matrix that we estimate from the publicly available dataset ``Open Access Series of Imaging Studies (OASIS)''~\cite{marcus2007open}. As such we construct datasets that have very similar correlation structures to real data and that also have ground truth, which allows for quantative analysis. In the next part we detail our data generation model then present results. 
\subsubsection{Data Generation Model}
\begin{figure}[!htb]
\centering
\small
\subfigure[Example Real Map]{\includegraphics[width=0.25\linewidth]{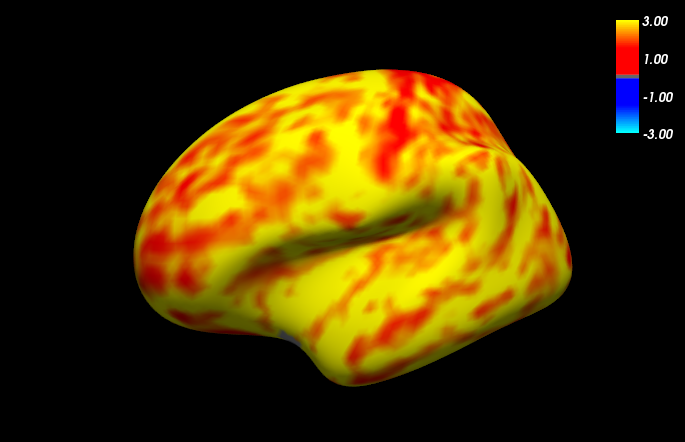}}
\subfigure[Example Synthetic Map I]{\includegraphics[width=0.25\linewidth]{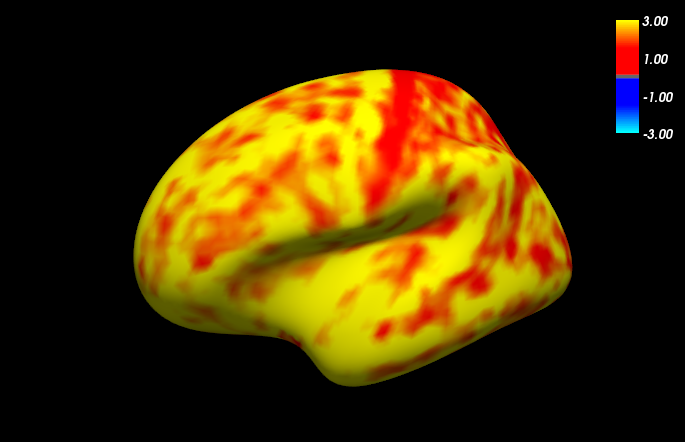}}
\subfigure[Example Synthetic Map II]{\includegraphics[width=0.25\linewidth]{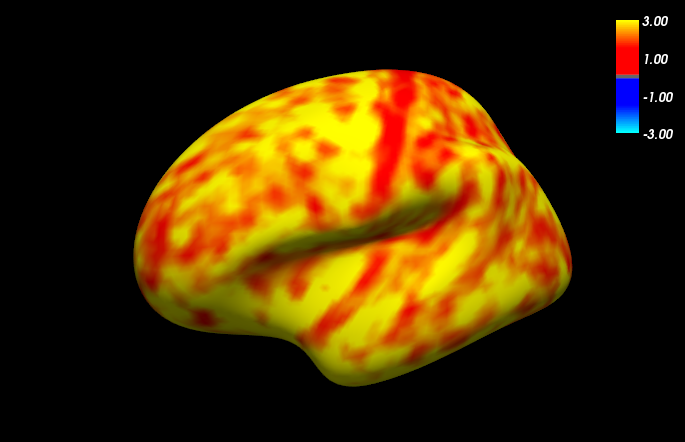}}
\subfigure[Mean Thickness]{\includegraphics[width=0.25\linewidth]{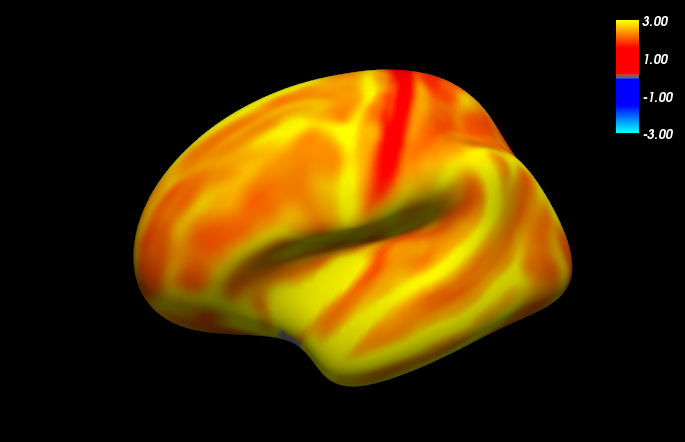}}
\subfigure[$N=0.1\%$]{\includegraphics[width=0.25\linewidth]{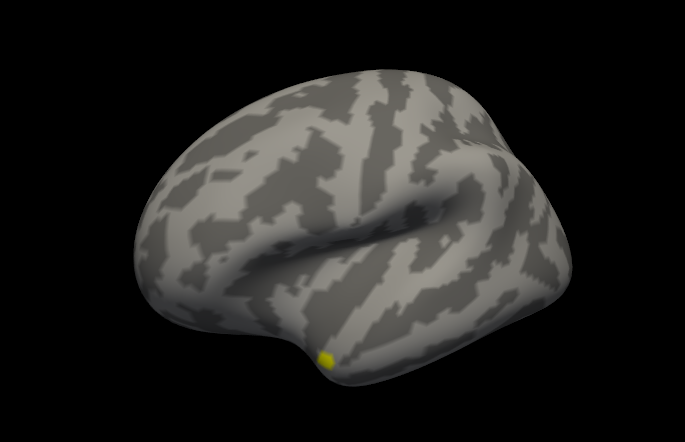}}
\subfigure[$N=0.2\%$]{\includegraphics[width=0.25\linewidth]{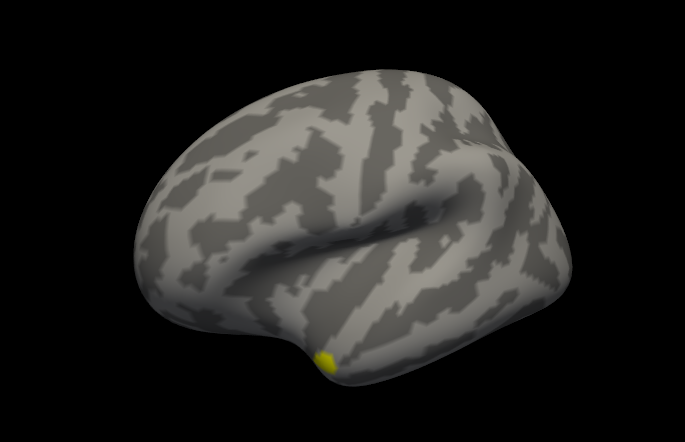}}
\subfigure[$N=0.5\%$]{\includegraphics[width=0.25\linewidth]{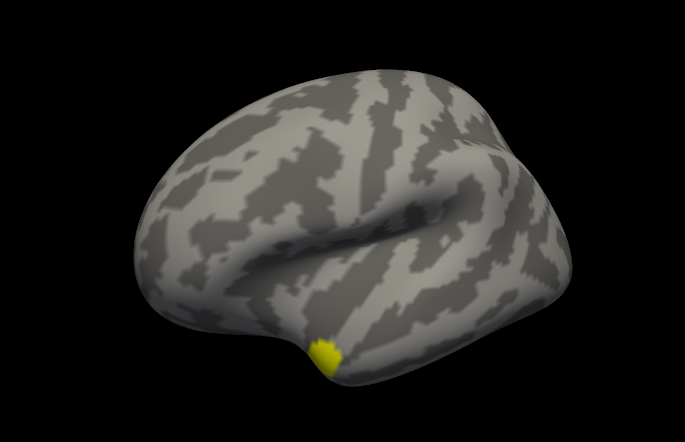}}
\subfigure[$N=1.0\%$]{\includegraphics[width=0.25\linewidth]{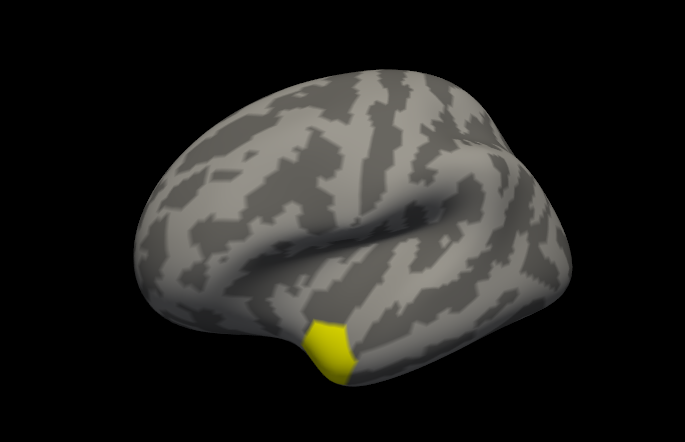}}
\subfigure[$N=2.0\%$]{\includegraphics[width=0.25\linewidth]{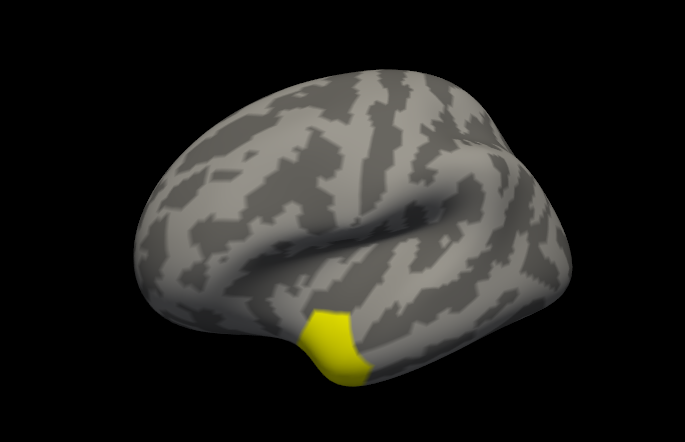}}
\caption{\label{fig:thickness} Data for cortical thickness experiments. (a-c) A real cortical thickness map extracted from a healthy individual and two synthetically generated maps. (d) Mean cortical thickness map of 315 individuals. (e-i) Relevant features in yellow (local atrophy) for varying sizes. }
\end{figure}
First, we define the random variable that corresponds to thickness maps. The general idea is to estimate the mean and the covariance matrix from an existing dataset and then define a random vector coming from a normal distribution with the estimated mean and covariance. We first extract cortical thickness maps discretized at 10242 vertices ($F=10242$) from the OASIS dataset\footnote{We refer the reader to Appendix~\ref{FREESURFER} and the references therein for further details on cortical thickness map extraction.}, which has images from 315 subjects in total aged between 18 and 96, and denote the entire dataset with $\tilde{X}\in\mathbb{R}^{10242\times 315}$. Then we estimate the mean and decompose the demeaned dataset with singular value decomposition (SVD):
\begin{eqnarray}
\nonumber \mu_X = \frac{1}{315}\sum_{i=1}^{315}\tilde{X}_{\cdot,i},\ X_{\cdot,i} = \tilde{X}_{\cdot,i} - \mu_X\ i\in [1,315],\ X = U\Sigma V^T,
\end{eqnarray}
where $X_{\cdot,i}$ denotes the i$^{th}$ column of $X$. In terms of the SVD components the empirical covariance matrix for the dataset can be computed as 
\begin{equation}
\nonumber \tilde{C} = \frac{1}{S}XX^T = \frac{1}{S}U\Sigma V^TV\Sigma^TU^T = U\frac{\Sigma\Sigma^T}{S}U^T.
\end{equation}
Then we define the random cortical thickness map generator as
\begin{eqnarray}
\nonumber \mathbf{x} = U\Sigma \mathbf{v} + \mu_X,\ \mathbf{v} \propto \mathcal{N}(\mathbf{0}_S,\mathbf{I}_S / S),
\end{eqnarray}
were $\mathbf{0}_S$ and $\mathbf{I}_S$ are the zero vector of size $S$ and the identity matrix of size $S\times S$.  It is easy to verify that the random vector $x$ has the same mean and covariance matrix as the empirically estimated ones. This method for creating synthetic datasets produces very realistic examples. For instance in Figures~\ref{fig:thickness}(a), (b) and (c) we display a real cortical thickness map and two synthetically generated maps. The similarity between the real map in (a) and the synthetic maps in (b) and (c) are striking.

The previous steps generate synthetic maps that are very similar to real ones. Now, we define a relationship between labels and maps to fully formulate the data generation model. Most neuro-degenerative diseases manifest themselves as thinning of the cortex, which is also refered to as {\em atrophy}. Motivated by this we formulate the relationship between the label and the thickness map as local atrophy in a region denoted as a binary vector $\mathbf{r}$. In Figures~\ref{fig:thickness}(e-i) we show in yellow the different regions we use in our experiments, where elements of $\mathbf{r}$ correponding to the yellow regions are ones and the remaining are zeros. The percentages represent the ratio of the number of vertices in the regions to the entire set. For $N=0$ the size of the region is zero. Based on the regions the full data generation model is given as 
\begin{equation}
\mathbf{x} \triangleq U\Sigma\mathbf{v} + \mu_X - y\mathbf{r}\frac{2\rho\sqrt{\textrm{diag}(\tilde{C})}}{\sqrt{1-\rho^2}},
\end{equation}
where $\textrm{diag}(\tilde{C})$ is the vector composed of diagonal elements of $\tilde{C}$. This construction ensures that the Pearson's correlation coefficient between the label $y$ and any thickness measurement within the region is $\rho$ assuming that $y=0$ and $y=1$ are equally likely. 
\subsubsection{Controlling False Positive Rates} 
In this section we evaluate how well the proposed method performs in terms of false positive and negative rates. We experiment with different samples sizes $S = [100,150,\dots,500]$ and different numbers of relevant features, in other words different sizes for the local atrophy region. For each parameter setting and training strategy we train a forest and determine the selection frequency threshold using Equation~\ref{eqn:thresholdII} for the desired rates $\alpha=0.01$ and $0.001$. We then compute the associated false positive and negative rates. We repeat the experiments 20 times for each parameter setting and report the average results in Table~\ref{tab:correlated}. 
\begin{table}[!htb]
\centering
\begin{tabular}{|c|c|c|}
\hline
\begin{sideways}Strategy I\end{sideways}&
\includegraphics[width=0.37\linewidth]{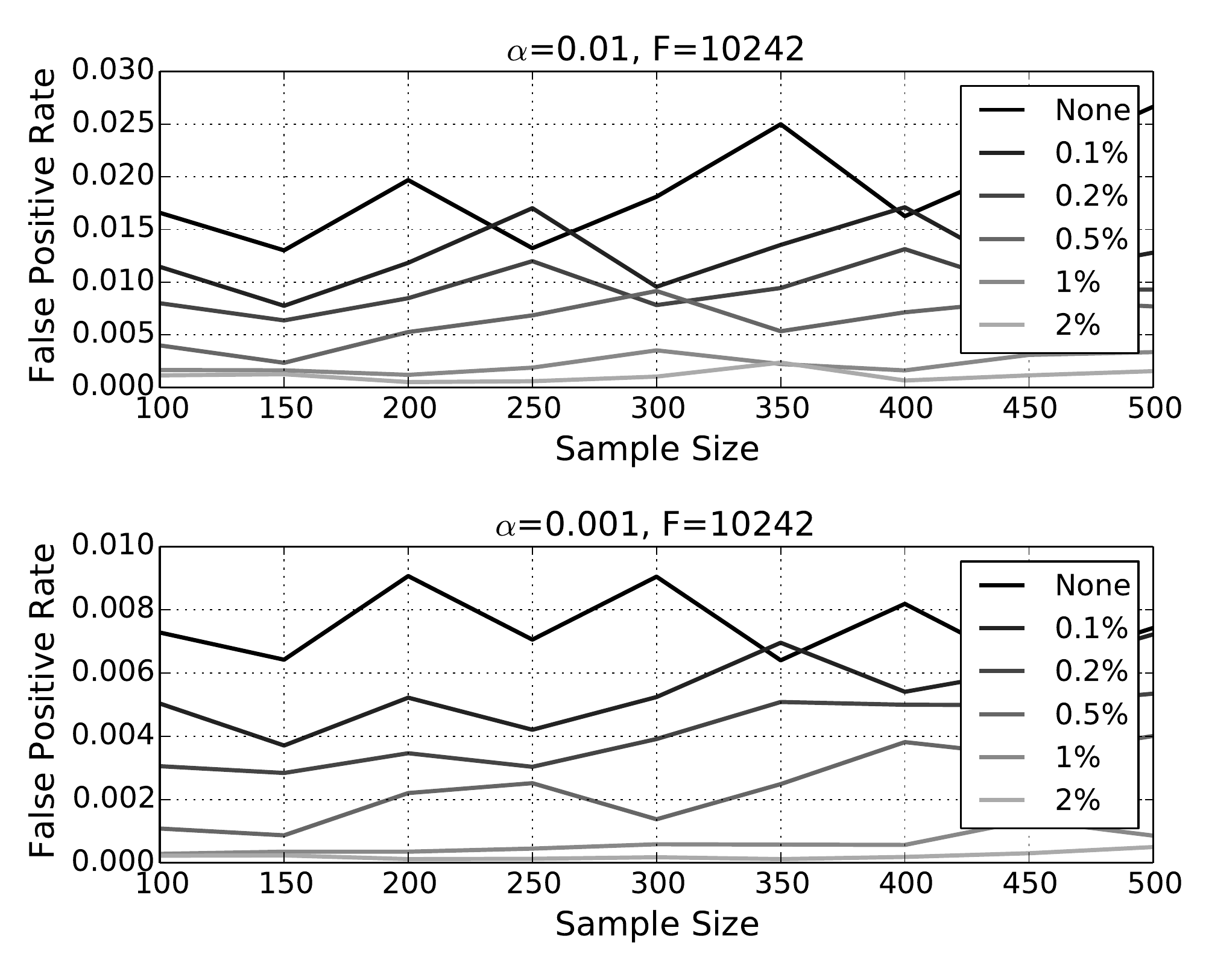}&
\includegraphics[width=0.37\linewidth]{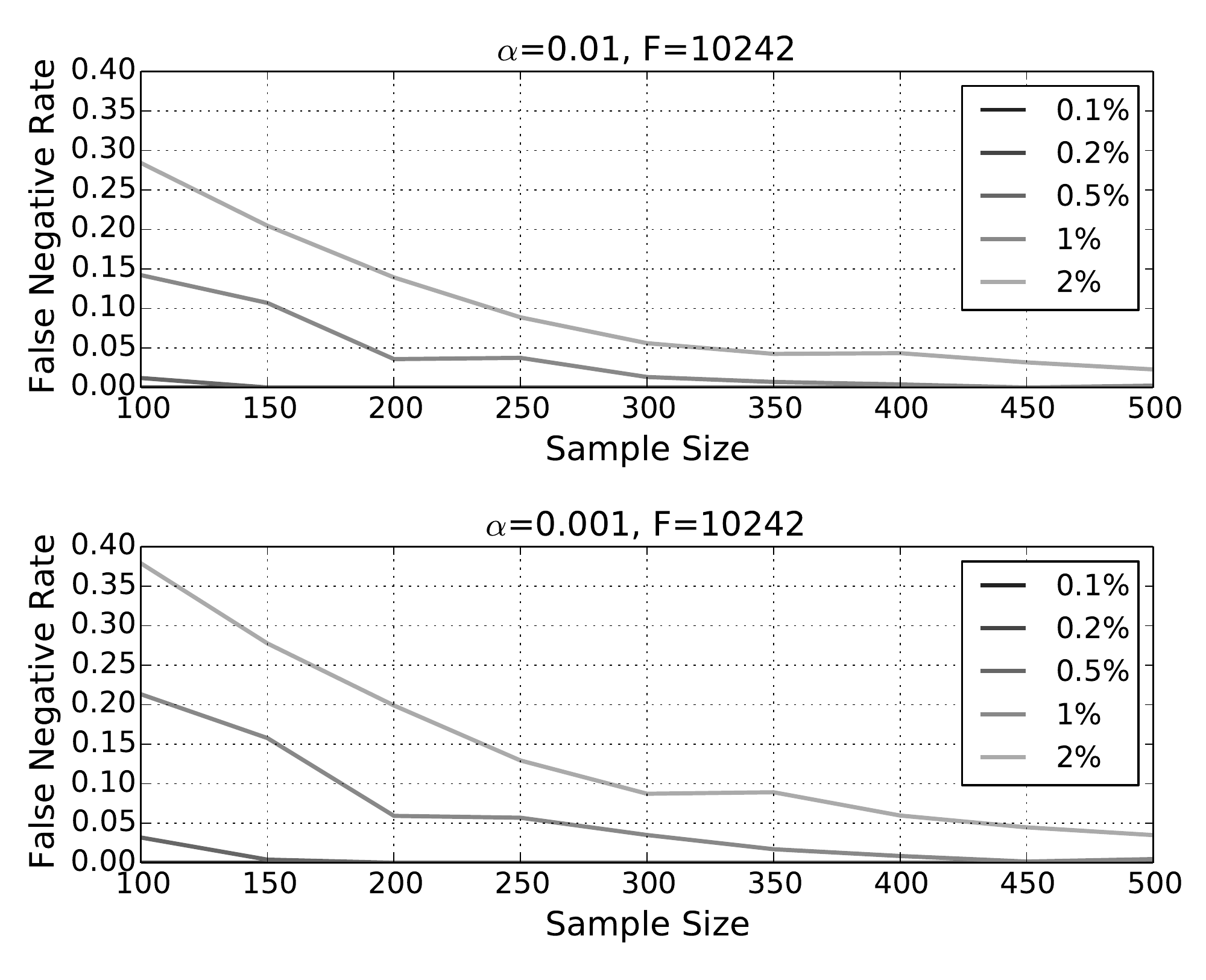}\\
\hline
\begin{sideways}Strategy II\end{sideways}&
\includegraphics[width=0.37\linewidth]{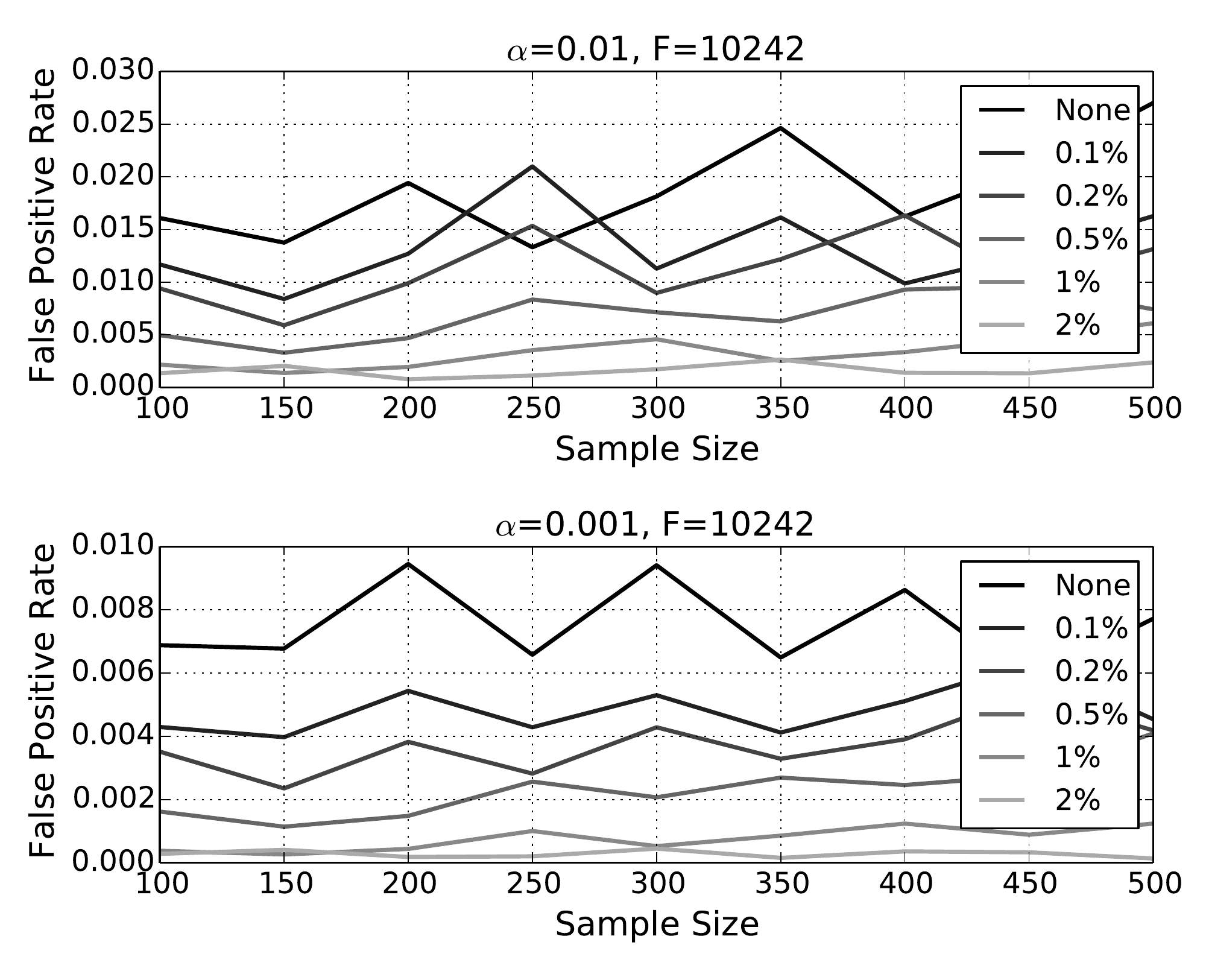}&
\includegraphics[width=0.37\linewidth]{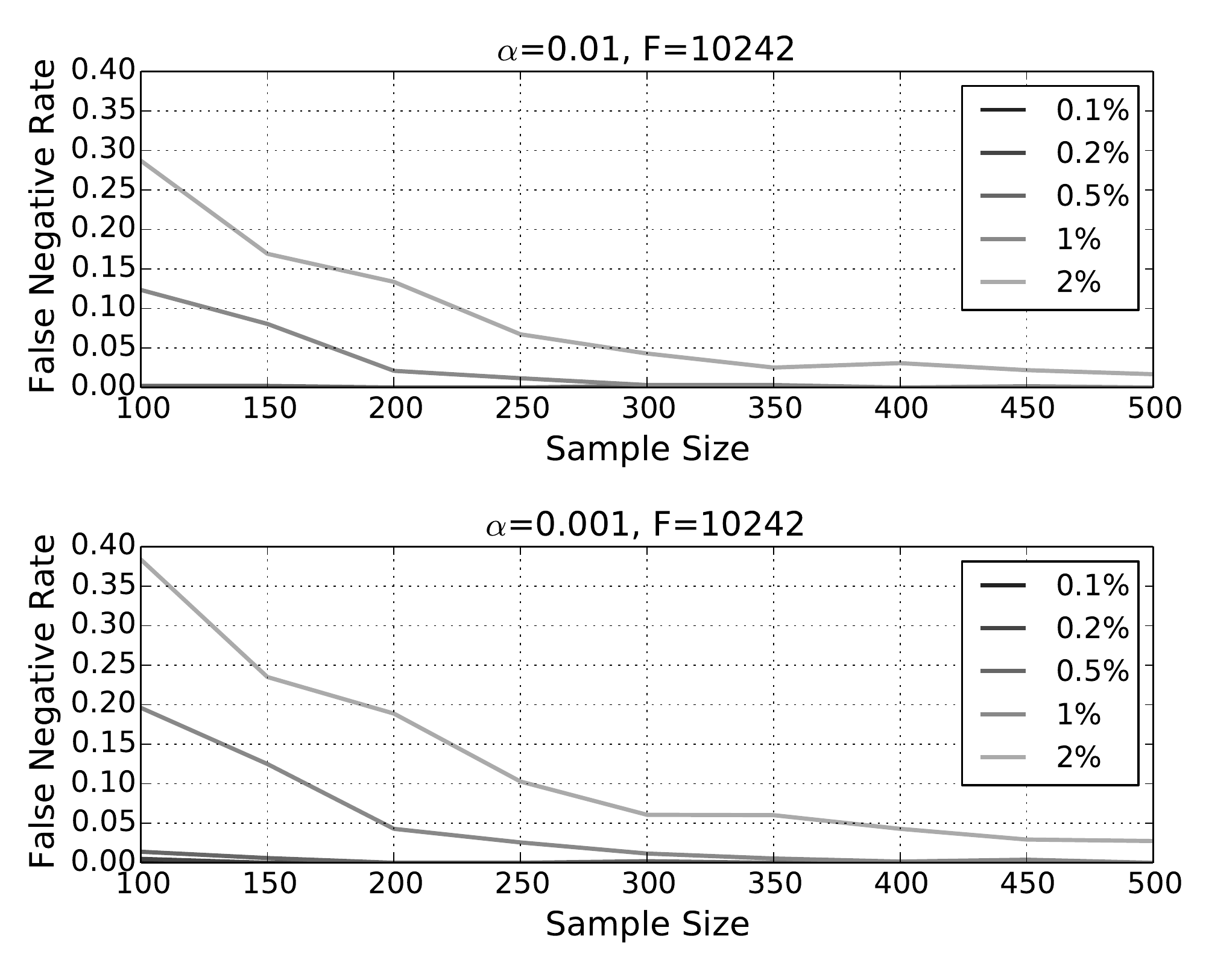}\\
\hline
\end{tabular}
\caption{\label{tab:correlated}\small False positive/negative rate analysis for the experiments with synthetic cortical thickness datasets. Each row shows graphs plotting false positive and negative rates for different sample sizes and different proportion of relevant features. }
\end{table}

The left column in Table~\ref{tab:correlated} shows the graphs plotting false positive rates. We observe that the behavior is very similar to the one we observed in the experiments with independent datasets. As the number of relevant features increase the false positive rates decrease and increasing sample size does not affect the false positive rates. For $\alpha=0.01$ we observe that except for the null case the false positive rates are around the desired rate. For $\alpha=0.001$ the rates are higher than the desired rate except for $N=1\%$ and $N=2\%$.

The right column shows the false negative rate plots for the experiments. We observe a sharp decrease with increasing sample size. We also see that the false negative rate increases as the number of relevant features increases, which is due to the competition between features as explained in the previous section. The false negative rates are particularly low when $N \leq 0.5\%$.

Overall, we see that the proposed method can be very useful in determining a selection frequency threshold that can split the relevant and non-relevant features even when there is a complex correlation structure between the features. Especially when there are actually relevant features, the threshold is able to limit the false positive rates on the order of the desired rate while keeping the false negative rates low.

\section{Discussions}
Experimental analysis demonstrated that the proposed method is able to provide thresholds that can limit the false positive rates on the order of the desired level, especially in the presence of relevant features, and keep false negative rates low for high-dimensional problems. How to get rid of the remaining false positives and improve upon the false negatives are the topics of this discussion section. 

The main source of higher false positive rates than the desired level is spurious statistics that violate the Assumption~\ref{assumption}. This is actually the case for any sort of feature selection method. As the dimension of the problem gets higher, in other words as the ratio $F/S$ increases, the risk of having stronger spurious statistical relationships between the label and some non-relevant features also increases. The remedy to tackle this issue is subsampling of the samples either with bagging or bootstrapping. In our experiments we have set the bagging ratio to $1/2$. Reducing this ratio even further improves on the problem of spurious statistics. However, in most problems there are not enough samples to have a bagging ratio lower than $1/2$. An alternative would be to use a wrapping algorithm based on bootstrapping such as the Stability Selection~\cite{meinshausen2010stability}.

The most prominent cause of false negative rates is the competition between relevant features. Each relevant feature will manifest a different empirical correlation with the label. Hence, their selection frequencies will be different. When a relevant feature ranks lower on the ranking of relevant features based on empirical correlations, the chances of it making it above the threshold will decrease. Wrapper algorithms, such as forward or backward selection methods, combined with bootstrapping schemes can remedy this issue. Alternatively, the recently proposed knock-out strategies~\cite{konukoglu2013feature}, which remove the most relevant features rather than removing the non-relevant ones, can also solve this issue. 

One of the biggest advantages of the proposed method is that it is straightforward to integrate it into a wrapping algorithm, such as the forward or backward feature selection, or a bootstrapping scheme such as Stability Selection~\cite{meinshausen2010stability}. The proposed method determines in a principled way a threshold on the selection frequency, so researchers using wrapper algorithms do not have to choose an arbitrary cut-off on the feature ordering. Moreover, unlike the permutation testing whose computation times make an integration difficult, the proposed method is very lightweight and its additional computational burden is almost negligible. 

An important application of Random Forests is localization and segmentation in computer vision and medical image analysis~\cite{Criminisi2012}. For such problems the trend is not to pre-compute features before training but rather sample the feature space at each node on the fly. This approach can be thought of as having a very large number of features and applying the training strategy I. Applying the proposed method on such problems directly is challenging due to the extremely high-dimensional nature of the data. However, in such cases one can use the spatial information to bin features coming from neighboring locations together. This would effectively reduce the feature dimension and the proposed approximate model can be applied with more confidence. 
\section{Conclusions}
This article proposed an approximate model for selection frequency in random forest. As shown, this model can be used to determine thresholds to split between relevant and non-relevant features in a principled way. The user sets a desired level of false positive rate and the model provides an optimal threshold on the selection frequency. 

The proposed model alleviates the need to use heuristics to set thresholds on feature rankings for random forest. It is computationally lightweight and therefore, can be easily integrated in any wrapper method. Furthermore, the models' good properties in terms of keeping the false negative rates low is a great advantage for such wrapper methods. 

Future research can focus on modeling the effects of bagging or bootstrapping of samples on the selection frequency model. These effects, even when approximately modeled, can provide a better approximation to the reality and therefore, provide a better way to determine thresholds and estimate the false positive rates. 
\section{Acknowledgments}
Authors would like to thank Ben Glocker, Darko Zikic and Tian Ge for providing valuable comments and proof-reading. 

This research was carried out in whole or in part at the Athinoula A. Martinos Center for Biomedical Imaging at the Massachusetts General Hospital, using resources provided by the Center for Functional Neuroimaging Technologies, P41EB015896, a P41 Biotechnology Resource Grant supported by the National Institute of Biomedical Imaging and Bioengineering (NIBIB), National Institutes of Health. This work also involved the use of instrumentation supported by the NIH Shared Instrumentation Grant Program and/or High-End Instrumentation Grant Program; specifically, grant number(s) S10RR023401, S10RR019307, S10RR019254 and S10RR023043. Melanie Ganz' research was also supported by the Alfred Benzon and the Lundbeck Foundation as well as the Carlsberg foundation. Since OASIS data was used we acknowledge the following grants: P50 AG05681, P01 AG03991, R01 AG021910, P50 MH071616, U24 RR021382, R01 MH56584. The research is also partially supported by R01EB006758, AG022381, 5R01AG008122-22, R01 AG016495-11, RC1 AT005728-01, R01 NS052585-01, 1R21NS072652-01, 1R01NS070963, R01NS083534 and by The Autism \& Dyslexia Project funded by the Ellison Medical Foundation, and by the NIH Blueprint for Neuroscience Research (5U01-MH093765), part of the multi-institutional Human Connectome Project.
\begin{appendices}
\section{FreeSurfer processing}\label{FREESURFER}
We processed the structural brain MRI (T1-weighted) scans with the FreeSurfer software suite (\url{https://freesurfer.nmr.mgh.harvard.edu}) \cite{fischl2012freesurfer} and computed subject-specific models of the cortical surface \cite{dale1999cortical,fischl1999cortical} as well as thickness measurements across the entire cortical mantle \cite{fischl2000measuring}. Subject-level thickness measurements were then transferred to a common coordinate system, via a surface-based nonlinear registration procedure \cite{fischl1999high}. For computational efficiency, we utilized the left-hemisphere of the {\em fsaverage5} representation, consisting of 10,242 vertices. We smoothed the cortical thickness maps with an approximate Gaussian kernel \cite{han2006reliability} with a full-width-half-maximum of $5 \; \mathrm{mm}$.

The extracted cortical thickness maps from the OASIS dataset as well as the software to generate synthetic data that is used in this article is available upon request. 
\end{appendices}
\bibliographystyle{plain}
\bibliography{article}

\begin{thebibliography}{10}

\bibitem{Altmann2010}
Andr\'{e} Altmann, Laura Toloşi, Oliver Sander, and Thomas Lengauer.
\newblock {Permutation importance: a corrected feature importance measure.}
\newblock {\em Bioinformatics (Oxford, England)}, 26(10):1340--7, May 2010.

\bibitem{Amit1997}
Yali Amit and Donald Geman.
\newblock Shape quantization and recognition with randomized trees.
\newblock {\em Neural computation}, 9(7):1545--1588, 1997.

\bibitem{Archer2008}
Kellie~J. Archer and Ryan~V. Kimes.
\newblock {Empirical characterization of random forest variable importance
  measures}.
\newblock {\em Computational Statistics \& Data Analysis}, 52(4):2249--2260,
  January 2008.

\bibitem{Breiman2001}
Leo Breiman.
\newblock Random forests.
\newblock {\em Machine learning}, 45(1):5--32, 2001.

\bibitem{Breiman2008}
Leo Breiman and Cutler Adele.
\newblock Random forests - classification manual.
\newblock
  \url{http://www.stat.berkeley.edu/~breiman/RandomForests/cc_manual.htm}.
\newblock Accessed: 2014-07-28.

\bibitem{Criminisi2012}
Antonio Criminisi, Jamie Shotton, and Ender Konukoglu.
\newblock Decision forests: A unified framework for classification, regression,
  density estimation, manifold learning and semi-supervised learning.
\newblock {\em Foundations and Trends{\textregistered} in Computer Graphics and
  Vision}, 7(2--3):81--227, 2012.

\bibitem{dale1999cortical}
Anders~M Dale, Bruce Fischl, and Martin~I Sereno.
\newblock Cortical surface-based analysis: I. segmentation and surface
  reconstruction.
\newblock {\em Neuroimage}, 9(2):179--194, 1999.

\bibitem{Deng2013}
H~Deng and George Runger.
\newblock {Gene Selection With Guided Regularized Random Forests}.
\newblock {\em Pattern Recognition}, 2013.

\bibitem{Deng2012}
Houtao Deng and George Runger.
\newblock Feature selection via regularized trees.
\newblock In {\em The 2012 International Joint Conference on Neural Networks
  (IJCNN)}, pages 1--8. IEEE, 2012.

\bibitem{Diaz-Uriarte2006}
Ram\'{o}n D\'{\i}az-Uriarte and Sara {Alvarez de Andr\'{e}s}.
\newblock {Gene selection and classification of microarray data using random
  forest.}
\newblock {\em BMC bioinformatics}, 7(1):3, January 2006.

\bibitem{fischl2012freesurfer}
Bruce Fischl.
\newblock Freesurfer.
\newblock {\em Neuroimage}, 62(2):774--781, 2012.

\bibitem{fischl2000measuring}
Bruce Fischl and Anders~M Dale.
\newblock Measuring the thickness of the human cerebral cortex from magnetic
  resonance images.
\newblock {\em Proceedings of the National Academy of Sciences},
  97(20):11050--11055, 2000.

\bibitem{fischl1999cortical}
Bruce Fischl, Martin~I Sereno, and Anders~M Dale.
\newblock Cortical surface-based analysis ii: Inflation, flattening, and a
  surface-based coordinate system.
\newblock {\em Neuroimage}, 9(2):195--207, 1999.

\bibitem{fischl1999high}
Bruce Fischl, Martin~I Sereno, Roger~BH Tootell, Anders~M Dale, et~al.
\newblock High-resolution intersubject averaging and a coordinate system for
  the cortical surface.
\newblock {\em Human brain mapping}, 8(4):272--284, 1999.

\bibitem{Genuer2011}
R~Genuer, I~Morlais, and W~Toussile.
\newblock {Gametocytes infectiousness to mosquitoes: variable selection using
  random forests, and zero inflated models}.
\newblock {\em arXiv preprint arXiv:1101.0344}, 2011.

\bibitem{Genuer2010}
Robin Genuer, Vincent Michel, Evelyn Eger, and B~Thirion.
\newblock {Random Forests based feature selection for decoding fMRI data}.
\newblock In {\em Proceedings of Compstat.}, 2010.

\bibitem{Gregorutti2013}
Baptiste Gregorutti, Bertrand Michel, and P~Saint-Pierre.
\newblock {Correlation and variable importance in random forests}.
\newblock {\em arXiv preprint arXiv:1310.5726}, pages 1--28, 2013.

\bibitem{han2006reliability}
Xiao Han, Jorge Jovicich, David Salat, Andre van~der Kouwe, Brian Quinn,
  Silvester Czanner, Evelina Busa, Jenni Pacheco, Marilyn Albert, Ronald
  Killiany, et~al.
\newblock Reliability of mri-derived measurements of human cerebral cortical
  thickness: the effects of field strength, scanner upgrade and manufacturer.
\newblock {\em Neuroimage}, 32(1):180--194, 2006.

\bibitem{Hapfelmeier2013}
A.~Hapfelmeier and K.~Ulm.
\newblock {A new variable selection approach using Random Forests}.
\newblock {\em Computational Statistics \& Data Analysis}, 60:50--69, April
  2013.

\bibitem{Hothorn2006}
Torsten Hothorn, Kurt Hornik, and Achim Zeileis.
\newblock Unbiased recursive partitioning: A conditional inference framework.
\newblock {\em Journal of Computational and Graphical statistics},
  15(3):651--674, 2006.

\bibitem{Jiang2004}
Hongying Jiang, Youping Deng, Huann-Sheng Chen, Lin Tao, Qiuying Sha, Jun Chen,
  Chung-Jui Tsai, and Shuanglin Zhang.
\newblock {Joint analysis of two microarray gene-expression data sets to select
  lung adenocarcinoma marker genes.}
\newblock {\em BMC bioinformatics}, 5(1):81, June 2004.

\bibitem{konukoglu2013feature}
Ender Konukoglu, Melanie Ganz, Koen~Van Leemput, and Mert~R. Sabuncu.
\newblock On feature relevance in image-based prediction models: An empirical
  study.
\newblock In {\em machine learning in medical imaging}, pages 171--178.
  springer, 2013.

\bibitem{Konukoglu2013}
Ender Konukoglu, Ben Glocker, Darko Zikic, and Antonio Criminisi.
\newblock Neighbourhood approximation using randomized forests.
\newblock {\em Medical image analysis}, 17(7):790--804, 2013.

\bibitem{Langs2011}
Georg Langs, Bjoern~H Menze, Danial Lashkari, and Polina Golland.
\newblock {Detecting stable distributed patterns of brain activation using Gini
  contrast.}
\newblock {\em NeuroImage}, 56(2):497--507, May 2011.

\bibitem{Lunetta2004}
Kathryn~L Lunetta, L~Brooke Hayward, Jonathan Segal, and Paul {Van Eerdewegh}.
\newblock {Screening large-scale association study data: exploiting
  interactions using random forests.}
\newblock {\em BMC genetics}, 5:32, January 2004.

\bibitem{marcus2007open}
D.S. Marcus, T.H. Wang, J.~Parker, J.G. Csernansky, J.C. Morris, and R.L.
  Buckner.
\newblock Open access series of imaging studies (oasis): cross-sectional mri
  data in young, middle aged, nondemented, and demented older adults.
\newblock {\em Journal of Cognitive Neuroscience}, 19(9):1498--1507, 2007.

\bibitem{meinshausen2010stability}
Nicolai Meinshausen and Peter B{\"u}hlmann.
\newblock Stability selection.
\newblock {\em Journal of the Royal Statistical Society: Series B (Statistical
  Methodology)}, 72(4):417--473, 2010.

\bibitem{Menze2009}
Bjoern~H Menze, B~Michael Kelm, Ralf Masuch, Uwe Himmelreich, Peter Bachert,
  Wolfgang Petrich, and Fred~A Hamprecht.
\newblock {A comparison of random forest and its Gini importance with standard
  chemometric methods for the feature selection and classification of spectral
  data.}
\newblock {\em BMC bioinformatics}, 10(1):213, January 2009.

\bibitem{Rodenburg2008}
Wendy Rodenburg, A~Geert Heidema, Jolanda M~A Boer, Ingeborg M~J
  Bovee-Oudenhoven, Edith J~M Feskens, Edwin C~M Mariman, and Jaap Keijer.
\newblock {A framework to identify physiological responses in microarray-based
  gene expression studies: selection and interpretation of biologically
  relevant genes.}
\newblock {\em Physiological genomics}, 33(1):78--90, March 2008.

\bibitem{Sandri2006}
Marco Sandri and Paola Zuccolotto.
\newblock Variable selection using random forests.
\newblock In {\em Data analysis, classification and the forward search}, pages
  263--270. Springer, 2006.

\bibitem{Strobl2008b}
Carolin Strobl, Anne-Laure Boulesteix, Thomas Kneib, Thomas Augustin, and Achim
  Zeileis.
\newblock {Conditional variable importance for random forests.}
\newblock {\em BMC bioinformatics}, 9:307, January 2008.

\bibitem{Strobl2007}
Carolin Strobl, Anne-Laure Boulesteix, Achim Zeileis, and Torsten Hothorn.
\newblock {Bias in random forest variable importance measures: illustrations,
  sources and a solution.}
\newblock {\em BMC bioinformatics}, 8:25, January 2007.

\bibitem{Strobl2008}
Carolin Strobl and Achim Zeileis.
\newblock {Danger: High power!–exploring the statistical properties of a test
  for random forest variable importance}.
\newblock Technical Report~17, Ludwig-Maximillians-Universitaet Muenchen, 2008.

\bibitem{Svetnik2004}
Vladimir Svetnik, Andy Liaw, Christopher Tong, and Ting Wang.
\newblock {Application of Breiman's random forest to modeling
  structure-activity relationships of pharmaceutical molecules}.
\newblock {\em Multiple Classifier Systems}, pages 334--343, 2004.

\bibitem{Tang2009}
Rui Tang, Jason~P Sinnwell, Jia Li, David~N Rider, Mariza de~Andrade, and
  Joanna~M Biernacka.
\newblock {Identification of genes and haplotypes that predict rheumatoid
  arthritis using random forests.}
\newblock {\em BMC proceedings}, 3 Suppl 7(Suppl 7):S68, January 2009.

\bibitem{Wang2010}
Minghui Wang, Xiang Chen, and Heping Zhang.
\newblock {Maximal conditional chi-square importance in random forests.}
\newblock {\em Bioinformatics}, 26(6):831--7, March 2010.

\bibitem{Wilf2000}
Herbert~S. Wilf.
\newblock Lectures on integer partitions.
\newblock Lecture Notes at University of Pennsylvania:
  \url{http://www.math.upenn.edu/~wilf/PIMS/PIMSLectures.pdf}, 2000.

\bibitem{Yang2009}
Wei Yang and C~Charles Gu.
\newblock {Selection of important variables by statistical learning in
  genome-wide association analysis}.
\newblock {\em BMC Proceedings}, 3(Suppl 7):S70, 2009.

\bibitem{Zhou2010}
Qifeng Zhou, Wencai Hong, Linkai Luo, and Fan Yang.
\newblock {Gene Selection Using Random Forest and Proximity Differences
  Criterion on DNA Microarray Data}.
\newblock {\em Journal of Convergence Information Technology}, 5(6):161--170,
  August 2010.

\end{thebibliography}
\end{document}